\documentclass[11pt, oneside]{article}
\usepackage[margin=1.0in]{geometry}
\pdfoutput=1 
\usepackage{graphicx}		
\usepackage{amssymb}
\usepackage{amsmath}
\usepackage{amsthm}
\usepackage{mathtools}
\usepackage{hyperref}
\usepackage{enumitem}
\usepackage{verbatim}
\usepackage{xcolor}
\usepackage[utf8]{inputenc}
\usepackage[T1]{fontenc}
\usepackage{microtype}
\usepackage[bottom]{footmisc}

\newtheorem{theorem}{Theorem}
\newtheorem{definition}[theorem]{Definition}

\newtheorem{lemma}[theorem]{Lemma}
\newtheorem{claim}[theorem]{Claim}
\newtheorem{corollary}[theorem]{Corollary}
\newtheorem{conjecture}[theorem]{Conjecture}
\numberwithin{theorem}{section}

\renewcommand{\iff}{\Leftrightarrow}

\newcommand{\ex}[2]{{\ifx&#1& \mathbb{E} \else \underset{#1}{\mathbb{E}} \fi \left[#2\right]}}
\newcommand{\pr}[2]{{\ifx&#1& \mathbb{P} \else \underset{#1}{\mathbb{P}} \fi \left[#2\right]}}
\newcommand{\var}[2]{{\ifx&#1& \mathsf{Var} \else \underset{#1}{\mathsf{Var}} \fi \left[#2\right]}}

\newcommand{\R}{\mathbb{R}}
\newcommand{\eps}{\varepsilon}

\DeclareMathOperator*{\argmin}{arg\,min}
\newcommand{\argmax}{\textup{argmax}}
\newcommand{\rank}{\mathrm{rank}}
\newcommand{\D}{\mathcal{D}}

\newcommand{\Q}{\mathcal{Q}}
\newcommand{\T}{\mathcal{T}}
\newcommand{\U}{\mathcal{U}}
\newcommand{\W}{\mathcal{W}}
\newcommand{\X}{\mathcal{X}}
\newcommand{\Y}{\mathcal{Y}}
\newcommand{\Z}{\mathcal{Z}}
\newcommand{\sZ}{\tilde{Z}}
\newcommand{\sz}{\tilde{z}}

\newcommand{\sx}{\tilde{x}}

\newcommand{\sy}{\tilde{y}}
\newcommand{\dkl}[2]{\mathrm{D}\left(#1\middle\|#2\right)}
\newcommand{\tvd}{\mathrm{d_{\mathit{TV}}}}

\newcommand{\AUC}{\mathsf{AUROC}}

\let\originalleft\left
\let\originalright\right
\renewcommand{\left}{\mathopen{}\mathclose\bgroup\originalleft}
\renewcommand{\right}{\aftergroup\egroup\originalright}

\newcommand{\mybignote}[2]{}

\usepackage{bbm}
\usepackage{float}
    
\usepackage[style=alphabetic,backend=bibtex,maxbibnames=20,maxcitenames=6,giveninits=true,doi=false,url=true]{biblatex}
\newcommand*{\citet}[1]{\AtNextCite{\AtEachCitekey{\defcounter{maxnames}{6}}}
\textcite{#1}}
\newcommand*{\citetall}[1]{\AtNextCite{\AtEachCitekey{\defcounter{maxnames}{999}}}
\textcite{#1}}
\newcommand*{\citep}[1]{\cite{#1}}

\title{Reasoning About Generalization via\\Conditional Mutual Information\footnotetext{Accepted for publication at the 33rd Annual Conference on Learning Theory (COLT 2020).}}
\author{Thomas Steinke\thanks{IBM Research -- Almaden. \dotfill \texttt{cmi@thomas-steinke.net}} \and Lydia Zakynthinou\thanks{Khoury College of Computer Sciences, Northeastern University. Part of this work was done during an internship at IBM Research -- Almaden. 
\dotfill \texttt{zakynthinou.l@northeastern.edu}}}

\begin{document}

\maketitle

\begin{abstract}
    We provide an information-theoretic framework for studying the generalization properties of machine learning algorithms. Our framework ties together existing approaches, including uniform convergence bounds and recent methods for adaptive data analysis.
    
    Specifically, we use Conditional Mutual Information (CMI) to quantify how well the input (i.e., the training data) can be recognized given the output (i.e., the trained model) of the learning algorithm. 
    We show that bounds on CMI can be obtained from VC dimension, compression schemes, differential privacy, and other methods. We then show that bounded CMI implies various forms of generalization.
    
\end{abstract}

\break

\tableofcontents

\break

\section{Introduction} 

How can we ensure that a machine learning system produces an output that generalizes to the underlying distribution, rather than overfitting its training data? That is, how can we ensure that the hypotheses or models that are produced are reflective of the underlying population the training data was drawn from, rather than patterns that occur only by chance in the training data? This is perhaps the fundamental question for the science of statistical machine learning. 

A vast array of methods have been proposed to answer this question. Most notably, the theory of uniform convergence shows that, if the output is sufficiently ``simple,'' then it cannot overfit too much. A more recent line of work has used distributional stability (in the form of differential privacy) to provide generalization guarantees that compose adaptively -- that is, statistical validity is preserved even when a dataset is reused multiple times with each successive analysis being influenced by the outcomes of prior analyses. Other methods for proving generalization include compression schemes and uniform stability.

Unfortunately, these different methods for providing generalization guarantees are largely disconnected from one another; it is, in general, not possible to compare or combine techniques. In this paper, we provide a framework to reason about many of these these differing approaches using the unifying language of information theory.

\subsection{Background: Generalization}

We consider the standard setting of statistical learning \cite{Valiant84,Haussler92,KearnsV94}. There is an unknown probability distribution $\mathcal{D}$ over some known set $\mathcal{Z}$. We have access to a sample $Z \in \mathcal{Z}^n$ consisting of $n$ independent draws from $\mathcal{D}$. Informally, our goal is to learn something about the underlying distribution $\mathcal{D}$ from the dataset $Z$. Formally, we have a function $\ell : \mathcal{W} \times \mathcal{Z} \to \mathbb{R}$ and our goal is to find some $w_* \in \mathcal{W}$ that approximately minimizes $\ell(w_*,\mathcal{D}) := \ex{Z' \leftarrow \mathcal{D}}{\ell(w_*,Z')}$.\footnote{For simplicity, in this introduction we only consider $\ell$ to be a linear function (that is, only taking in a single element of $\mathcal{Z}$). Our methods readily extend to the more general case where $\ell : \mathcal{W} \times \mathcal{Z}^m \to \mathbb{R}$.} We refer to $\ell$ as a ``loss function'' and $\ell(\cdot,\D)$ as the ``true loss'' or ``population loss''.

Intuitively, $w_*$ represents some hypothesis or model and $\ell(w_*,\mathcal{D})$ measures the veracity or quality of $w_*$.  In the supervised machine learning setting, $\mathcal{Z} = \mathcal{X} \times \mathcal{Y}$ represents pairs of feature vectors and labels and $w_*$ represents a function $f_{w_*} : \mathcal{X} \to \mathcal{Y}$ that predicts the label given the features.  For example, the 0-1 loss measures the error rate of the predictor: $\ell(w_*,\mathcal{D}) = \pr{(X,Y) \leftarrow \mathcal{D}}{f_{w_*}(X)\ne Y}$. Thus minimizing $\ell(w_*,\mathcal{D})$ corresponds to finding the most accurate predictor.

However, we cannot evaluate the true loss (a.k.a.~``risk'') $\ell(w_*,\mathcal{D})$ since the distribution $\mathcal{D}$ is unknown. Instead we can compute the empirical loss (a.k.a.~``empirical risk'') $\ell(w_*,Z) := \frac{1}{n} \sum_{i=1}^n \ell(w_*, Z_i)$ using the sample $Z$. A natural learning strategy is ``Empirical Risk Minimization (ERM)'' -- i.e., $w_* = \argmin_{w \in \mathcal{W}} \ell(w,Z)$. The question of generalization is thus: How can we ensure that $\ell(w_*,Z) \approx \ell(w_*,\mathcal{D})$?

The classical theory of uniform convergence \cite{VapnikC71} approaches this problem by studying the class of functions $\mathcal{F}:=\{\ell(w,\cdot) : w \in \mathcal{W}\}$. If we can show that  $\sup_{w\in\mathcal{W}}|\ell(w,Z)-\ell(w,\mathcal{D})|$ is small with high probability for a random $Z \leftarrow \D^n$, then the question of generalization is answered. Such bounds can be obtained from combinatorial properties of $\mathcal{F}$, such as its Vapnik-Chervonenkis (VC) dimension \cite{VapnikC71,Talagrand94,AlonBCH97} or its fat-shattering dimension \cite{KearnsS94,BartlettLW96}.

Uniform convergence makes no reference to the algorithm; it depends only on its range $\mathcal{F}$. An algorithm may generalize better than uniform convergence would suggest \citep{Shalev-ShwarzSSS09,Feldman16}. For example, it is common to add a regularizer to the ERM -- that is, $w_* = \argmin_{w \in \mathcal{W}} \ell(w,Z) + \lambda \|w\|$, where $\lambda>0$ is a parameter and $\|w\|$ is a measure of the complexity of $w$. Thus we explicitly consider generalization to be a property of the learning algorithm $A : \mathcal{Z}^n \to \mathcal{W}$, which may or may not be randomized.

There are several ways to show that a specific algorithm $A$ generalizes. Algorithms whose output essentially only depends on a few of the input data points, as formalized by compression schemes \cite{LittlestoneW86}, can be shown to generalize. Uniform stability \cite{BousquetE02} entails strong generalization bounds for algorithms where changing a single input datum does not change the loss of the algorithm too much. Similarly, differential privacy \cite{DworkMNS06} -- a distributional notion of stability -- entails generalization bounds \cite{DworkFHPRR15,BassilyNSSSU16,JungLNRSS19}.

In general, these various methods for proving generalization are incompatible and incomparable. This raises the question of whether it is possible to provide a unifying framework or language to study generalization.

\subsubsection{(Unconditional) Mutual Information}\label{sec:uncondmi}

A recent line of work has studied generalization using mutual information and related quantities 
\cite[etc.]{RussoZ16,RaginskyRTWX16,Alabdulmohsin16,FeldmanS18a,BassilyMNSY18,DworkFHPRR15b,RogersRST16,Smith17,XuR17,NachumY18,NachumSY18,EspositoGI19,BuZV19}. Specifically, for a (possibly randomized) algorithm $A : \mathcal{Z}^n \to \mathcal{W}$ and a dataset $Z \leftarrow \D^n$, we consider the quantity $I(A(Z);Z)$. This measures how much information the output of $A$ contains about its input.

Bounded mutual information implies generalization: If $\ell : \mathcal{W} \times \mathcal{Z} \to [0,1]$, $A : \mathcal{Z}^n \to \mathcal{W}$ and $Z \leftarrow \D^n$, then $|\ex{}{\ell(A(Z),Z)-\ell(A(Z),\D)}| \le \sqrt{\frac{2}{n}\cdot I(A(Z);Z)}$ \cite{RussoZ16,XuR17}.

Bounds on mutual information can be obtained from differential privacy or from bounds on the entropy of the output of $A$. Specifically, if $A$ is $\varepsilon$-differentially private, then $I(A(Z);Z) \le \frac12 \varepsilon^2 n$ \cite{McGregorMPRTV10,BunS16}. And we have the generic bound $I(A(Z);Z) \le H(A(Z)) \le \log|\mathcal{W}|$.

Unfortunately, mutual information can easily be infinite even in settings where generalization is easy to prove. Bassily, Moran, Nachum, Shafer, and Yehudayoff \cite{BassilyMNSY18,NachumSY18} showed that \emph{any} proper and consistent learner for threshold functions $A$ must have $I(A(Z);Z) \ge \Omega\left(\frac{\log\log|\mathcal{Z}|}{n^2}\right)$ when $Z \leftarrow \D^n$ for some worst-case distribution $\D$. The dependence on the size of the domain $\mathcal{Z} \subset \mathbb{R}$ is mild, but, if the domain is infinite, then the mutual information is unbounded. In contrast, the VC dimension of threshold functions is 1, which implies strong uniform convergence bounds even for infinite domains.

We remark that thresholds can be ``embedded'' into larger classes, such as higher-dimensional linear thresholds (halfspaces) or even neural networks. Thus these negative results for unconditional mutual  extend to those classes too. This strong negative result shows that \emph{any} proper empirical risk minimizer for thresholds must have unbounded mutual information; it is easier to show that many \emph{specific} natural algorithms and natural distributions have unbounded mutual information: Linear regression has unbounded mutual information (even in dimension 0 with Gaussian data, which is simply outputting the mean \cite{BuZV19}). The most natural algorithms for thresholds have infinite mutual information for \emph{any} continuous data distribution.

The fundamental issue with the mutual information approach is that even a single data point has infinite information content if the distribution is continuous. This makes it difficult to bound the mutual information -- an algorithm revealing a single data point is not an issue for generalization, but it is for mutual information.

We address the shortcomings of the mutual information approach by moving to conditional mutual information. Our conditioning approach can be viewed as ``normalizing'' the information content of each data point to one bit. That is, an algorithm that reveals one data point only has conditional mutual information of one bit, even if the unconditional mutual information would be infinite. 

\subsection{Our contributions: Conditional Mutual Information (CMI)}

We introduce the conditional mutual information (CMI) framework for reasoning about the generalization properties of machine learning algorithms. CMI is a quantitative property of an algorithm $A$ and, optionally, a distribution $\D$. (Note that it does \emph{not} depend on the loss function of interest.)

Intuitively, CMI measures how well we can ``recognize'' the input (i.e., training data) given the output (i.e., trained model) of the algorithm. Recognizing the input is formalized by considering a ``supersample'' consisting of $2n$ data points -- namely the $n$ input data points mixed with $n$ ``ghost'' data points -- and measuring how well it is possible to distinguish the true inputs from their ghosts.\footnote{The so-called "ghost samples" symmetrization technique has been used to prove generalization and Rademacher complexity bounds from VC bounds since its inception \cite{VapnikC71}. (The name is attributed to Luc Devroye \cite{Lugosi20,DevroyeGL96}.) This technique is an inspiration for our definition and terminology.}\footnote{This intuition for CMI should be contrasted with that for (unconditional) mutual information, which asks how much of the input we could reconstruct from the output without the prompt of a supersample.} (Note that the ghost samples are entirely hypothetical -- they only exist in the analysis.) We consider both the CMI with respect to a distribution -- i.e., the supersample consists of $2n$ independent draws from the distribution -- and the distribution-free CMI -- i.e., a worst-case supersample. The supersample is randomly partitioned into the input and the ghost samples. We then measure how much information the output reveals about this partition using mutual information, where we take the supersample to be known (i.e., we condition on the supersample and the unknown information is how it is partitioned).

We now state the formal definition of CMI:

\newcommand{\CMI}[2]{{\ifx&#2& \mathsf{CMI} \else \mathsf{CMI}_{#2} \fi \left(#1\right)}}
\begin{definition}[Conditional Mutual Information (CMI) of an Algorithm]\label{def:CMI}
Let $A:\Z^n\rightarrow \W$ be a randomized or deterministic algorithm.
Let $\D$ be a probability distribution on $\Z$ and let $\sZ\in\Z^{n \times 2}$ consist of $2n$ samples drawn independently from $\D$. Let $S\in \{0,1\}^n$ be uniformly random and independent from $\sZ$ and the randomness of $A$. Define $\sZ_S \in \Z^n$ by $(\sZ_S)_i = \sZ_{i,S_i+1}$ for all $i \in [n]$ -- that is, $\sZ_S$ is the subset of $\sZ$ indexed by $S$. 

The \emph{conditional mutual information (CMI) of $A$ with respect to $\D$} is $$\CMI{A}{\D}:=I(A(\sZ_S);S|\sZ).$$

The \emph{(distribution-free) conditional mutual information (CMI) of $A$} is $$\CMI{A}{}:=\sup_{\sz\in\Z^{n \times 2}} I(A(\sz_S);S).$$

\end{definition}

We  remark on some basic properties of CMI.\footnote{For background on information-theoretic quantities like mutual information, see, e.g., the textbook of \citet{CoverT06}.}

Firstly, $0 \le \CMI{A}{\D} \le \CMI{A}{} \le n \cdot \log 2$ for any $A$ and any $\D$.\footnote{We take $\log$ to denote the natural logarithm and, correspondingly, the units for information-theoretic quantitites are \emph{nats}, instead of \emph{bits}, where $1$ bit equals $\log2 \approx 0.7$ nats.} The case $\CMI{A}{\D}=0$ corresponds to the output of $A$ being independent from its input, such as when $A$ is a constant function. The other extreme, $\CMI{A}{\D} = n \cdot \log 2$, corresponds to an algorithm that reveals all of its input, allowing arbitrary overfitting. Note that the CMI is always finite, which is in stark contrast with unconditional mutual information. Essentially, the conditioning normalizes the information content of each datum to one bit -- an algorithm that reveals $k$ of its input points and reveals nothing about the other $n-k$ inputs has a CMI of $k$ bits. 

For further intuition about the scale or units of CMI, we briefly mention how it relates to generalization error and other notions: Our generalization bounds become non-vacuous as soon as the CMI drops below  $n \cdot \log 2$ nats (i.e., $n$ bits). In terms of asymptotics, we obtain meaningful generalization bounds whenever the CMI is $o(n)$. More precisely, $\CMI{A}{\D}=\varepsilon^2 n$ roughly corresponds to generalization error $\varepsilon$ and is roughly a consequence of $\varepsilon$-differential privacy. We also have, for any $A:\Z^n\rightarrow \W$ and any $\D$, that $\CMI{A}{\D} \le H(A(Z)) \le \log|\mathcal{W}|$, where $H(A(Z))$ is the Shannon entropy \cite{CoverT06} of the output of $A$ on an input $Z$ consisting of $n$ i.i.d.~draws from $\D$.

Finally, we note that CMI composes non-adaptively, i.e., if $A_1,A_2 : \mathcal{Z}^n \to \mathcal{W}$ are randomized algorithms (whose internal sources of randomness are independent) then $\CMI{(A_1,A_2)}{\D}\leq \CMI{A_1}{\D}+\CMI{A_2}{\D}$ for all distributions $\D$. Moreover, CMI has the postprocessing property (this is an immediate consequence of the data processing inequality for conditional mutual information). Namely, if $A : \mathcal{Z}^n \to \mathcal{W}$ and $B : \mathcal{W} \to \mathcal{W}'$ are randomized algorithms (whose internal sources of randomness are independent), then $\CMI{B(A(\cdot))}{\D} \le \CMI{A}{\D}$ for all distributions $\D$. This is an important robustness property (closely related to post-hoc generalization \cite{CummingsLNRW16,NissimSSSU18}). See Section~\ref{sec:composition} for more details. 

\nocite{LigettS19}

\subsubsection{Generalization from CMI}\label{sec:intro-cmigen}

The key property of CMI is, of course, that it implies generalization. Since there is no single definition of generalization, we prove several consequences of CMI bounds.

The following theorem gives several consequences for bounded linear loss functions.
\begin{theorem}\label{thm:intro-cmigen}
Let $A : \mathcal{Z}^n \to \mathcal{W}$ and $\ell : \mathcal{W} \times \mathcal{Z} \to [0,1]$. Let $\D$ be a distribution on $\mathcal{Z}$ and define $\ell(w,\D)=\ex{Z \leftarrow \D}{\ell(w,Z)}$ and $\ell(w,z)=\frac{1}{n} \sum_{i=1}^n \ell(w,z_i)$ for all $w\in\mathcal{W}$ and $z \in \mathcal{Z}^n$.
Then 
\begin{align}
\left|\ex{Z \leftarrow \D^n, A}{\ell(A(Z),Z)-\ell(A(Z),\D)}\right| &\le \sqrt{\frac{2}{n}\cdot\CMI{A}{\D}},\label{eq:cmigen-bdex}\\
\ex{Z \leftarrow \D^n, A}{\left(\ell(A(Z),Z)-\ell(A(Z),\D)\right)^2} &\le \frac{3\cdot\CMI{A}{\D}+2}{n},\label{eq:cmigen-bd2}\\
\intertext{and}
\ex{Z \leftarrow \D^n, A}{\ell(A(Z),\D)} &\le 2 \cdot \ex{Z \leftarrow \D^n, A}{\ell(A(Z),Z)} + \frac{3}{n} \cdot\CMI{A}{\D}.\label{eq:cmigen-realizable}
\end{align}
\end{theorem}

The first part of the theorem \eqref{eq:cmigen-bdex} is the simplest bound; it relates the expected empirical loss to the expected true loss. The second part \eqref{eq:cmigen-bd2} gives a bound on the expected squared difference between these quantities; this bound is qualitatively strictly stronger, but quantitatively weaker by (small) constants. The final part of the theorem deals with the realizeable (or overfitted) case where the empirical loss is zero or close to zero (sometimes this referred to as the interpolating setting); when $\ex{Z \leftarrow \D^n, A}{\ell(A(Z),Z)} \approx 0$ this yields a bound that is quadratically sharper than the other bounds.

We also have a result for unbounded loss functions:
\begin{theorem}
Let $A : \mathcal{Z}^n \to \mathcal{W}$ and $\ell : \mathcal{W} \times \mathcal{Z} \to \mathbb{R}$. Let $\D$ be a distribution on $\mathcal{Z}$.
Then \begin{equation}\left|\ex{Z \leftarrow \D^n, A}{\ell(A(Z),Z)-\ell(A(Z),\D)}\right| \le \sqrt{\frac{8}{n}\cdot\CMI{A}{\D}\cdot \ex{Z' \leftarrow \D}{\sup_{w \in \mathcal{W}} (\ell(w,Z'))^2}}.\label{eq:cmigen-unb}\end{equation}
\end{theorem}
The final term in the bound \eqref{eq:cmigen-unb} gives some scale for the loss function. It is necessary to make some kind of assumption on the losses, such as bounded moments. For example, this allows us to derive generalization bounds for squared loss (i.e., mean squared error) or hinge loss; see Sections \ref{sec:gen-squarederror} and \ref{sec:gen-hinge} for details. 

The basic generalization bounds for the agnostic setting \eqref{eq:cmigen-bdex},\eqref{eq:cmigen-bd2},\eqref{eq:cmigen-unb} and their proofs can be found in Section~\ref{sec:gen-basic}, while the proof for the realizable case \eqref{eq:cmigen-realizable} is in Section~\ref{sec:gen-realizable}.

Furthermore, we are able to extend our generalization results to non-linear loss functions; see Section \ref{sec:gen-nonlin}. As an example application, in Section \ref{sec:auroc} we derive the following generalization bound for the Area Under the ROC Curve (AUC/AUROC) statistic, which is a commonly-used non-linear statistic for measuring the performance of a classifier. Specifically, for a classifier $f : \Z \to \mathbb{R}$ that produces a numerical score or probability (rather than just a binary label), the AUROC is the probability that a random positive example has a higher score than a random negative example -- i.e., $\AUC(f,\D):=\pr{(Z_+,Z_-) \gets \D^2}{f(Z_+) > f(Z_-)|Z_+ \in \Z_+, Z_- \notin \Z_+}$, where $\Z_+ \subset \Z$ is the set of positive examples. The empirical AUROC $\AUC(f,z)$ is defined analogously. 
\begin{theorem}\label{thm:auc_intro}
Let $\mathcal{D}$ be a distribution on $\mathcal{Z}$.
Let $\mathcal{Z}_+ \subseteq \mathcal{Z}$ be the set of positive examples and assume $0< p:=\ex{Z \leftarrow \mathcal{D}}{Z \in \mathcal{Z}_+} <1$. Let $A :\mathcal{Z}^n \to \mathcal{W}$ be a randomized algorithm (whose randomness is independent from its input).
Then, for any $\varepsilon\in(0,1)$, 
$$\pr{Z \gets \D^n, A}{\left|\AUC(A(Z),Z) - \AUC(A(Z),\mathcal{D}) \right| \le \varepsilon} \ge 1 - O\left(\frac{\CMI{A}{\D}}{\varepsilon^2 p(1-p) n}\right).$$
\end{theorem}

The above bounds illustrate how we are able to derive a great variety of generalization bounds from a single CMI bound; these examples are by no means exhaustive. This versatility is a key strength of the CMI framework. 

We note that although most stated bounds are on the expectation of the generalization error (or its square), these can be converted into probability bounds via Markov's inequality. For example, Theorem \ref{thm:intro-cmigen}\eqref{eq:cmigen-bd2} implies $\pr{Z \leftarrow \D^n, A}{ \left|\ell(A(Z),Z)-\ell(A(Z),\D)\right|\ge \varepsilon} \le \frac{3\cdot \CMI{A}{\D} +2}{\varepsilon^2 n}$ for all $\varepsilon>0$. However, this does not yield ``high probability'' bounds -- that is, the failure probability decays polynomially with the desired error bound $\varepsilon$, rather than exponentially. 

\subsubsection{Obtaining CMI Bounds}

Next we turn our attention to methods for proving that an algorithm has low CMI. We show that a variety of known methods for proving generalization fit into our framework, by proving that they imply bounds on the CMI of the algorithm. Indeed, analysing these algorithms via CMI, versus a direct generalization analysis, yields essentially the same bound. These connections demonstrate the unifying nature of the CMI framework.

\paragraph{Compression Schemes}
First, we prove that, if an algorithm $A:\Z^n\rightarrow \W$ has a compression scheme of size $k$ \cite{LittlestoneW86}, then $\CMI{A}{}\leq O(k \cdot \log n)$. Intuitively, this is in agreement with the fact that an algorithm blatantly revealing $k$ of the input points and nothing about the rest would have a CMI of $k$ bits.
\begin{theorem}
Let $A_1 : \Z^n \to \Z^k$ have the property that $A_1(z) \subset z$ for all $z$. Let $A_2 : \Z^k \to \W$ be arbitrary and let $A : \Z^n \to \W$ satisfy $A(z)=A_2(A_1(z))$ for all $z$. Then $\CMI{A}{}\le k \log (2n)$.
\end{theorem}

\paragraph{Uniform Convergence \& VC Dimension}
Next, we show a connection between uniform convergence and CMI. We consider hypothesis classes $\W$ consisting of functions $h : \mathcal{X} \to \{0,1\}$ and we consider the standard 0-1 loss $\ell : \W \times (\X \times \{0,1\}) \to \{0,1\}$ defined by $\ell(h,(x,y)) = 0 \iff h(x)=y$. Bounded VC dimension is a necessary \cite{VapnikC71} and sufficient \cite{Talagrand94} condition for uniform convergence (for worst-case distributions) and is hence a sufficient condition for generalization. Note that CMI is a property which depends on the algorithm, whereas the VC dimension is a property of the output space; this appears to cause an incompatibility between the two methods. Nonetheless, we connect the two by proving that, for any hypothesis class of bounded VC dimension, there always \emph{exists} an empirical risk minimization algorithm $A:\Z^n\rightarrow\W$ with bounded CMI:

\begin{theorem}\label{thm:intro-vc-cmi}
Let $\Z=\X\times\{0,1\}$ and let $\mathcal{W}=\{h:\X\rightarrow\{0,1\}\}$ be a hypothesis class with VC dimension $d$.
Then, there exists an empirical risk minimizer $A:\Z^n\rightarrow \mathcal{W}$ for the 0-1 loss such that $\CMI{A}{} \leq d\log n + 2$.
\end{theorem}
We prove this theorem in Section~\ref{sec:cmi-vc} by showing that any algorithm satisfying a consistency property described next has bounded CMI and that there always exists an empirical risk minimizer with this consistency property. Intuitively, the consistency property we require says the following. Suppose the algorithm is run on some labelled dataset $(x,y) \in \mathcal{X}^n \times \{0,1\}^n$ to obtain an output hypothesis $h=A(x,y)$. If the dataset is relabelled to be perfectly consistent with $h$, then the algorithm should still output $h$ -- i.e., $A(x,h(x))=h$. This should also hold if further examples are added to the dataset (where the additional examples are also consistent with $h$) -- i.e., $A(x',h(x'))=h$ when $x \subset x'$. This is a very natural and reasonable consistency property. 

Note that it is not true that \emph{every} empirical risk minimizer for a class of bounded VC dimension has bounded CMI; if there are multiple minimizers to choose from, a pathological algorithm could encode superfluous information about the input in its output using this choice (thus violating our consistency property). We also remark that this bound is tight up to the $\log n$ term; combining Theorems \ref{thm:intro-vc-cmi} and \ref{thm:intro-cmigen} yields generalization bounds that are tight up to this term. It is natural to ask whether this logarithmic term can be removed. We conjecture that it can be removed by instead considering an \emph{approximate} empirical risk minimizer.

Obtaining CMI bounds in the case of compression schemes and VC dimension mainly reduces to observing that these two conditions effectively restrict the output space -- that is, conditioned on the supersample $\sZ$, there are few possible outputs $\W_{\sZ}:=\{A(\sZ_s):s\in\{0,1\}^n\}$ and we can use the worst-case entropy bound $\log|\W_{\sZ}|$. In both cases, this results in a multiplicative factor of $\log n$ in the CMI bound. This logarithmic factor could potentially be eliminated given more information about the structure of the problem. We demonstrate two specific cases where tighter bounds can be obtained by taking into account assumptions on the algorithm or the distribution $\D$. First, we prove that there exists an empirical risk minimizer which learns threshold functions in the realizable case and has \emph{constant} CMI, whereas the general result gives a bound of $O(\log n)$; see Section~\ref{sec:cmi-thresholds}. Second, we consider the problem of learning parity functions on $\{0,1\}^d$ when $\D$ is the uniform distribution. Intuitively, this uniformity assumption on $\D$ ensures that, as the number of samples $n$ increases, with high probability there will be only a single consistent hypothesis. This allows us to prove that there exists an empirical risk minimizer whose CMI decreases to zero as the number of samples increases, namely $\CMI{A}{\D}\le O(n\cdot 2^{d-n})$; see Section~\ref{sec:cmi-pseudo}.

\paragraph{Distributional Stability \& Differential Privacy}
Finally, we show that distributional stability implies CMI bounds. Differential privacy is the most well-known form of distributional stability and its generalization properties are well-established \cite{DworkFHPRR15,BassilyNSSSU16,JungLNRSS19}.
\begin{theorem}\label{thm:intro-stabcmi}
Let $A : \Z^n \to \W$ be a randomized algorithm. Any one of the following conditions imply that $\CMI{A}{} \le \varepsilon n$.
\begin{itemize}
    \item[(i)] $A$ is $\sqrt{2\varepsilon}$-differentially private \citep{DworkMNS06}.
    \item[(ii)] $A$ satisfies $\varepsilon$-concentrated differential privacy \citep{BunS16}.
    \item[(iii)] $A$ satisfies $\varepsilon$-average leave-one-out KL stability \citep{FeldmanS18a}.
    \item[(iv)] $A$ is $\varepsilon$-TV stable \citep{BassilyNSSSU16}.\footnote{$\varepsilon$-TV stability is equivalent to $(0,\varepsilon)$-differential privacy \cite{DworkKMMN06}.}
\end{itemize}
\end{theorem}
We remark that, with the exception of TV stability, all of the conditions in Theorem \ref{thm:intro-stabcmi} are known to imply bounds on (unconditional) mutual information. However, TV stability does not imply any bounds on mutual information, so this sets CMI apart. In particular, approximate differential privacy (a.k.a.~$(\varepsilon,\delta)$-differential privacy) implies TV stability and hence CMI bounds. See Section~\ref{sec:dpstab} for details.

\subsection{Related Work, Limitations, \& Further Work}

\paragraph{Information Theory \& Generalization}
Generalization is a very well-studied subject and several connections to information theory have been made. Some of these connections are orthogonal to our work; for example, the information bottleneck method \cite{TishbyPB00} considers the mutual information between the input/output of the classifier (rather than the training algorithm) and various intermediate representations internal to the classifier.

Various recent works have considered the mutual information between the input and output of the training algorithm and used this to derive generalization bounds; see the discussion in Section \ref{sec:uncondmi}. This line of work is the inspiration and starting point for our work. CMI extends this line of work. In particular, we are able to incorporate VC dimension into the CMI framework, whereas prior works \cite{BassilyMNSY18,NachumSY18} showed that this was impossible for (unconditional) mutual information.

Other extensions of the basic mutual information approach have been proposed. Inspired by generic chaining (a methodology from stochastic process theory closely related to uniform convergence), \citetall{AsadiAV18} consider the mutual information between the input of the algorithm and an \emph{approximation} of its output (or, rather, a sequence of closer and closer approximations of its output). This method provides tighter generalization bounds, but requires analysis of the geometry of the output space.

Another approach is to consider the mutual information between a single (but arbitrary) input datum and the output \cite{RaginskyRTWX16,WangLF16,BuZV19,HaghifamNKRD20}. If we consider the mutual information between a single datum and the output conditioned on the rest of the data (i.e., $I(A(Z);Z_i|Z_{-i})$, where $Z_{-i}=(Z_1,\cdots,Z_{i-1},Z_{i+1},\cdots,Z_n)$), then this implies bounds on the overall mutual information (i.e., $I(A(Z);Z) \le \sum_{i=1}^n I(A(Z);Z_i|Z_{-i})$) \cite[Lem.~3.7]{FeldmanS18a}. If we do not condition on the rest of the data, then the reverse inequality holds (i.e., $I(A(Z);Z) \ge \sum_{i=1}^n I(A(Z);Z_i)$) \cite[Eq.~17]{BuZV19} and it is possible to obtain sharper bounds than via the overall mutual information \cite{BuZV19,HaghifamNKRD20}. We believe that further exploration in this direction is warranted (in particular, by combining this single-datum approach with our conditioning approach). 

\citetall{NegreaHDKR19} study the mutual information between the output of an algorithm and a random subset of its input dataset. This is very similar to our CMI definition. This is used to provide generalization guarantees for Stochastic Gradient Langevin Dynamics (SGLD). Overall, their results are incomparable to ours, since they exploit the random subset method in a different manner -- they consider the ``disintegrated mutual information'' (in essence this is a random variable whose expectation is the conditional mutual information and each realization is the mutual information conditioned on a fixed value of the subset). However, their techniques can be combined with ours to yield even tighter bounds~\cite{HaghifamNKRD20}.

PAC-Bayesian bounds \cite{McAllester99} also relate information-theoretic quantities to generalization and are similar to the mutual information approach. These bounds are usually output-dependent -- that is, they give a generalization bound for a particular output hypothesis or hypothesis distribution, rather than uniformly bounding the expected error of the algorithm as we do. (Such output-dependent bounds may be stronger and output-independent results can be obtained by averaging over outputs.) PAC-Bayesian bounds can be used to analyze and interpret regularization.~\citet{HellstromD20} extend our generalization bounds for bounded loss functions to the PAC-Bayesian setting, as an application of their unifying approach to deriving information-theoretic generalization bounds.

\paragraph{High Probability Generalization} The generalization implied by CMI (Section \ref{sec:intro-cmigen}) does not yield ``high probability'' guarantees -- that is, to guarantee failure probability $\delta$, the error tolerance must grow polynomially in $1/\delta$, whereas polylogarithmic growth would be desireable. This is an inherent limitation of the CMI framework -- mutual information is an expectation and is thus not very sensitive to low-probability failures. In particular, an algorithm that does something ``good'' (e.g., output a fixed hypothesis) with probability $1-p$ and something ``bad'' (e.g., output a hypothesis entirely overfitted to the dataset) with probability $p$ has CMI $\approx p n$. Due to this sort of pathological example, CMI bounds cannot guarantee a failure probability lower than CMI$/n$.

An interesting direction for further work is to extend the CMI framework so that it yields high probability bounds. This requires moving from conditional mutual information to something like approximate max information \cite{RogersRST16} or R\'enyi mutual information \cite{EspositoGI20}. However, we note that many algorithms naturally produce high probability guarantees and it is almost always possible to obtain high probability guarantees by repetition to amplify the success probability. 

\paragraph{Loss Stability/Uniform Stability}
A long line of work
\cite[etc.]{RogersW78,DevroyeW79a,BousquetE02,FeldmanV19,DaganF19} has proven generalization bounds by showing that various algorithms have the property that their loss changes very little if a single input datum is replaced or removed. This is a very beautiful and well-developed theory that provides a unifying framework for generalization. In particular, differential privacy and VC dimension can be incorporated into the loss stability framework \cite{DaganF19}. 

Uniform stability (one of the strongest and most well-studied variants of loss stability) has the advantage that it readily yields high probability generalization bounds. 

However, one limitation of the loss stability approach is that the loss function is an integral part of the definition, whereas CMI and distributional stability notions do not depend on the loss function. Thus loss stability lacks the postprocessing robustness property and does not yield the same variety of generalization bounds as CMI does. It is common to consider multiple loss functions, such as when using a surrogate loss function.

Loss stability is typically defined for deterministic algorithms and randomized algorithms must be ``derandomized'' (such as by taking their expectation) to satisfy the definition. This is somewhat awkward and, arguably, the CMI framework is more elegant when handling randomized algorithms.

It is of course natural to ask whether loss stability and CMI can be unified in some way. In Section \ref{sec:ecmi} we discuss uniform stability further and propose a variant of CMI (Evaluated CMI or eCMI) that takes the loss function into account and allows us to translate between the notions. 

\paragraph{Adaptive Composition} When a single dataset is analyzed multiple times and each analysis is informed by the outcome of earlier analyses, generalization may fail even if each individual step generalizes well. This phenomenon led to the study of generalization in adaptive data analysis \cite[etc.]{HardtU14,DworkFHPRR15,SteinkeU15,DworkFHPRR15b,BassilyNSSSU16}. In particular, differential privacy provides a method for guaranteeing generalization that composes adaptively -- that is, running a sequence of algorithms, each of which is differentially private, on a single dataset results in a differentially private final output, even if each algorithm is given access to the output of previous algorithms.
Also, the recent work of \citet{LigettS19} introduced Local Statistical Stability, a notion based on the statistical distance between the prior distribution over the database elements and their posterior distribution conditioned on the output of the algorithm, which composes adaptively and yields high probability bounds.

Unfortunately, CMI does not compose adaptively (although we show that it does have the postprocessing property and satisfies non-adaptive composition in Section \ref{sec:composition}).  A challenge for further work is to fully integrate adaptive composition into some variant of the CMI framework. Towards this direction, we consider a variant of CMI (Universal CMI or uCMI) that does compose adaptively in Section \ref{sec:ucmi}. This notion extends CMI by considering a worst-case supersample $\sZ$ and worst-case distribution over $S$ but we can still show that we can obtain useful uCMI bounds from some of the notions we have tied into our framework.

\paragraph{More CMI bounds}
Our results give several methods for bounding CMI. An immediate direction for further work is to improve these results and to prove entirely new bounds for algorithms such as stochastic convex optimization or even non-convex gradient methods. The value of the CMI framework will be demonstrated if it yields new insights, such as entirely new generalization results or simplifications of known bounds.

Notably,~\citet{HaghifamNKRD20} recently showed that the generalization bounds based on CMI are tighter than those based on mutual information. Moreover, by combining CMI with the idea of ``disintegrated mutual information''~\citep{NegreaHDKR19} and the single-datum mutual information approach~\citep{BuZV19}, they give yet tighter bounds, which are later applied to yield improved generalization bounds for a Langevin dynamics algorithm. We believe that further exploration of the combination of these ideas with our conditioning approach is warranted.

~\citet{HellstromD20} extended our average generalization bounds to the PAC-Bayes and single-draw settings, giving both data-independent and data-dependent bounds. These results are established as an application of their proposed unifying approach to proving generalization bounds, which is based on an exponential inequality in terms of the information density between the output and the input data. Furthermore, the authors explore the effect of the conditioning approach to other measures, obtaining bounds in terms of the conditional versions of the $\alpha$-mutual information, the R\'enyi divergence, and the maximal leakage (the latter being tighter than the bound of its unconditional counterpart, in some cases).

\nocite{HaghifamNKRD20} \nocite{HellstromD20} 

\section{Preliminaries}\label{sec:prelim}
We adopt the convention of using upper-case letters for random variables and corresponding lower-case letters for their realizations. For a random variable $X$ and distribution $\mathcal{P}$, we write $X\gets \mathcal{P}$ to denote that $X$ is drawn from $\mathcal{P}$ and we write $X \gets \mathcal{P}^n$ to denote that $X$ consists of $n$ i.i.d.~draws from $\mathcal{P}$. We write $\mathcal{P}_X$ to denote the distribution of a random variable $X$. The uniform distribution over $\{0,1\}$ is denoted by $\U$. We denote $[n]=\{1,2,\cdots,n\}$.

In the statistical learning setting, we denote by $\Z=\X \times \Y$ the space of labeled examples and by $\W$ the set of hypotheses. For a loss function $\ell : \mathcal{W} \times \mathcal{Z} \to \mathbb{R}$, we define $\ell(w,Z) := \frac{1}{n} \sum_{i=1}^n \ell(w,Z_i)$ for linear loss functions and $\ell(w,\D) := \ex{Z \leftarrow \D}{\ell(w,Z)}$ for all $w \in \mathcal{W}$, dataset $Z \in \mathcal{Z}^n$ and distribution $\D$ on $\mathcal{Z}$.

\subsection{KL Divergence}
First, we define the KL divergence of two distributions.
\begin{definition}[KL Divergence]
 Let $\mathcal{P}, \Q$ be two distributions over the space $\Omega$ and suppose $\mathcal{P}$ is absolutely continuous with respect to $\Q$. The \emph{Kullback–Leibler (KL) divergence} (or \emph{relative entropy}) from $\Q$ to $\mathcal{P}$ is \[\dkl{\mathcal{P}}{\Q}=\ex{X\gets \mathcal{P}}{\log\frac{\mathcal{P}(X)}{\Q(X)}},\]
 where $\mathcal{P}(X)$ and $\Q(X)$ denote the probability mass/density functions of $\mathcal{P}$ and $\Q$ on $X$, respectively.\footnote{Formally, $\frac{\mathcal{P}(X)}{\Q(X)}$ is the Radon-Nikodym derivative of $\mathcal{P}$ with respect to $\mathcal{Q}$. If $P$ is not absolutely continuous with respect to $Q$ (i.e., $\frac{\mathcal{P}(X)}{\Q(X)}$ is undefined or infinite), then the KL divergence is defined to be infinite.}
\end{definition}

Note that we use $\log$ to denote the natural logarithm (i.e., base $e \approx 2.718$). Thus the units for all information-theoretic quantities are \emph{nats} rather than \emph{bits}.

A specific instance of $\mathcal{P}$ and $\Q$ which will prove useful is the case of Gaussian distributions.
\begin{lemma}[{\cite[Table~3]{GillAL13}}]\label{lem:klgaussian}
 Let $\mu,\nu, \sigma\in\R$. $$\dkl{\mathcal{N}(\mu,\sigma^2 I)}{\mathcal{N}(\nu,\sigma^2 I)} = \frac{1}{2\sigma^2} \|\mu-\nu\|_2^2.$$
\end{lemma}

Next, we note two useful properties of the KL divergence: the chain rule and convexity.
\begin{lemma}[Chain Rule for KL divergence,{~\cite[Thm.~2.5.3]{CoverT06}}]\label{lem:chainruleKL}
Let $\mathcal{P},\Q$ be two distributions over $\X\times\Y$. Then,
$$\dkl{\mathcal{P}(x,y)}{\Q(x,y)} = \dkl{\mathcal{P}(x)}{\Q(x)}+\ex{x'\gets \mathcal{P}(x)}{\dkl{\mathcal{P}(y|x=x')}{\Q(y|x=x'}},$$
where $\mathcal{P}(x)$ (resp. $\Q(x)$) denotes the marginal distribution of $\mathcal{P}$ (resp. $\Q$) on $\X$ and $\mathcal{P}(y|x=x')$ (resp. $\Q(y|x=x')$) denotes the marginal distribution of $\mathcal{P}$ (resp. $\Q$) on $\Y$ conditioned on $x=x'$.
\end{lemma}

\begin{lemma}[Convexity of KL Divergence,{~\cite[Thm.~11]{VEH14}}]\label{lem:convKL}
 Let $\mathcal{P}_0, \mathcal{P}_1, \Q_0,\Q_1$ be probability distributions over the space $\Omega$. For $t\in[0,1]$, define $\mathcal{P}_t=t\cdot\mathcal{P}_0 + (1-t)\cdot\mathcal{P}_1$ and $\Q_t=t\cdot\Q_0+(1-t)\cdot \Q_1$. Then, for all $t\in[0,1]$, \[\dkl{\mathcal{P}_t}{\Q_t}\leq t\cdot\dkl{\mathcal{P}_0}{\Q_0} + (1-t)\cdot\dkl{\mathcal{P}_1}{\Q_1}.\]
\end{lemma}
For the most part, we will use $\dkl{X}{Y}$ to denote $\dkl{\mathcal{P}}{\Q}$ for random variables $X,Y$ such that $X\gets \mathcal{P}$ and $Y\gets\Q$.

The next basic fact shows that the ``KL center'' of a collection of probability distributions is simply the mean of those distributions. (This is actually a more general fact that holds for all Bregman divergences \cite[Prop.~1]{BanerjeeMDG05}, \cite[Thm.~II.1]{FrigyikSG08}.)
\begin{lemma}[{\cite[Lem.~10.8.1]{CoverT06}}] \label{lem:klmin} 
Let $\{\mathcal{P}_y\}$ be a family of distributions indexed by $y \in \mathcal{Y}$ and let $\mathcal{Q}$ be a distribution on $\mathcal{Y}$. Let $\mathcal{P}_{\mathcal{Q}}=\ex{Y \gets \mathcal{Q}}{\mathcal{P}_Y}$ denote the convex combination of the distributions $\{\mathcal{P}_y\}$ weighted by $\mathcal{Q}$. Then $$\inf_{\mathcal{R}} \ex{Y \gets \mathcal{Q}}{\dkl{\mathcal{P}_Y}{\mathcal{R}}} = \ex{Y \gets \mathcal{Q}}{\dkl{\mathcal{P}_Y}{\mathcal{P}_{\mathcal{Q}}}}.$$
\end{lemma}

\subsection{Mutual Information}
The definition of mutual information is the following.
\begin{definition}[Mutual Information]
 Let $X,Y$ be two random variables jointly distributed according to $\mathcal{P}$ over the space $\X\times\Y$. The mutual information of $X$ and $Y$ is
 \[I(X;Y)=\dkl{\mathcal{P}(x,y)}{\mathcal{P}(x)\times\mathcal{P}(y)}= \ex{x'\gets\mathcal{P}(x)}{\dkl{\mathcal{P}(y|x=x')}{\mathcal{P}(y)}},\]
 where by $\mathcal{P}(x)\times\mathcal{P}(y)$ we denote the product of the marginal distributions of $\mathcal{P}$.
\end{definition}

The mutual information is symmetric and non-negative.
For distributions over a discrete domain, it can also be defined as
$$I(X;Y)=H(X)-H(X|Y)=H(Y)-H(Y|X),$$
where $H(\cdot)$ denotes the Shannon entropy of a random variable. The entropy satisfies $0\leq H(X)\leq \log(|\X|)$, where $|\X|$ is the size of the domain of the random variable $X$ and equality is achieved when $X$ is distributed uniformly over $\X$. The entropy also satisfies the following chain rule.
\begin{lemma}\label{lem:chainruleH}
For any random variables $X,Y$, $H(X,Y)=H(X)+H(Y|X)$.
\end{lemma}

Next, we define the conditional mutual information, a key notion for CMI.
\begin{definition}[Conditional Mutual Information]
For random variables $X,Y,Z$, the mutual information of $X$ and $Y$ conditioned on $Z$ is $$I(X;Y|Z)=\ex{z\gets\mathcal{P}_Z}{I(X|Z=z;Y|Z=z)}.$$
\end{definition}

Mutual information also satisfies the following chain rule.
\begin{lemma}[Chain Rule for Mutual Information,~\cite{CoverT06}]\label{lem:chainruleMI}
For any random variables $X,Y,Z$, $I(X,Y;Z) = I(X;Z)+I(Y;Z|X) = I(Y;Z)+I(X;Z|Y)$.
\end{lemma}

Finally, mutual information satisfies the data-processing inequality, which intuitively ensures that for a random variable $Y$ which depends on $X$, no post-processing of $Y$ could reveal more information about $X$ than $Y$.
\begin{lemma}[Data-Processing Inequality]\label{lem:data-proc}
If $X\rightarrow Y \rightarrow Z$ forms a Markov chain (i.e., $X$ and $Z$ are independent conditioned on $Y$), then $I(X;Y)\geq I(X;Z)$.
\end{lemma}

The reader is referred to the textbook of~\citet{CoverT06} for more background on information theory.

\subsection{Using Mutual Information}

We will make heavy use of the following lemma. The proof is straightforward, but the origin is unclear. It is known by at least two names: The "Donsker-Varadhan dual characterization of KL divergence" \cite[Lem.~2.1]{DonskerV75} (although a 1983 paper by the same authors is often erroneously cited instead of the 1975 reference) and the "Gibbs variational principle" which supposedly dates back to 1902 \cite{Gibbs02}.

\begin{lemma}[{\cite[Thm.~5.2.1]{Gray11},\cite[Lem.~4.10]{vanHandel14}}] \label{lem:sup}
Let $X$ and $Y$ be random variables on $\Omega$ (with $X$ absolutely continuous with respect to $Y$)  and $f : \Omega \to \mathbb{R}$ a (measurable) function. Then $$\ex{}{f(X)} \le \dkl{X}{Y} + \log \ex{}{e^{f(Y)}}.$$
Furthermore, for any $X$ and $Y$ there exists an $f$ that achieves equality, namely  $f(x) = \log(X(x)/Y(x))$ (where $X(\cdot)$ and $Y(\cdot)$ denote the probability mass/density functions of $X$ and $Y$ respectively). Similarly, for any $f$ and $Y$ there exists an $X$ that achieves equality, namely $X(x) = e^{f(x)} Y(x) / \ex{}{e^{f(Y)}}$.
\end{lemma}

We will mostly use this lemma in the following form.
\begin{corollary}
Let $S$, $S'$, and $Z$ be independent random variables where $S$ and $S'$ have identical distributions. Let $A$ be a random function whose randomness is independent from $S$, $S'$, and $Z$. Let $g$ be a fixed function. Then \begin{align*}
    \ex{A,S,Z}{g(A(S,Z),S,Z)} &\le \ex{Z}{\inf_{t>0} \frac{I(A(S,Z);S) + \log \ex{A,S,S',Z}{e^{t \cdot g(A(S,Z),S',Z)}}}{t}} \\&\le \inf_{t>0} \frac{I(A(S,Z);S|Z) + \ex{Z}{\log \ex{A,S,S',Z}{e^{t \cdot g(A(S,Z),S',Z)}}}}{t}.
\end{align*}
\end{corollary}
\begin{proof}
    This follows from Lemma \ref{lem:sup} and the fact that $$I(A(S,Z);S|Z) = \ex{Z}{I(A(S,Z);S)} = \ex{Z}{\dkl{(A(S,Z),S)}{(A(S,Z),S')}}.$$
    Specifically, we set $X=(A(S,Z),S,Z)$, $Y=(A(S,Z),S',Z)$, and $f((y,s,z))=t \cdot g(y,s,z)$. Then we optimize over the choice of the parameter $t>0$.
\end{proof}
This allows us to bound $\ex{}{g(A(S,Z),S,Z)}$ -- the object of interest -- in terms of the conditional mutual information $I(A(S,Z);S|Z)$ and the moment generating function $\ex{}{e^{t\cdot g(A(S,Z),S',Z)}}$.

We use Lemma~\ref{lem:sup} to obtain generalization bounds from CMI bounds. An alternative approach is to reason directly about probabilities instead of expectations:
\begin{lemma}[{\cite[Lem.~3.14]{FeldmanS18a}, cf.~\cite{IssaEG19}}]\label{lem:probkl}
Let $X$ and $Y$ be random variables and $E$ an event. Then $$\pr{}{X \in E} \le \frac{\dkl{X}{Y} + \log 2}{-\log \pr{}{Y \in E}}.$$
\end{lemma}

\subsection{Concentration Bounds}

To use Lemma \ref{lem:sup}, we must bound the moment generating function. We will make use of Hoeffding's lemma:

\begin{lemma}[\cite{Hoeffding63}]\label{lem:hoeffding}
Let $X \in [a,b]$ be a random variable with mean $\mu$. Then, for all $t \in \mathbb{R}$, $\ex{}{e^{tX}} \le e^{t\mu + t^2(b-a)^2/8}$.
\end{lemma}

An extension of Hoeffding's lemma is the following.
\begin{lemma}[\cite{McDiarmid89}]\label{lem:mcd-mgf}Let $f : \Omega^n \to \mathbb{R}$. Let $\Delta_1, \Delta_2, \cdots, \Delta_n \in \mathbb{R}$ be such that, for all $x_1, x_2, \cdots, x_n, y \in \Omega$ and all $i \in [n]$, $$\left| f(x_1, \cdots, x_{i-1}, y, x_{i+1}, \cdots, x_n) - f(x_1, \cdots, x_n) \right| \le \Delta_i.$$ Let $X_1, \cdots, X_n \in \Omega$ be independent random variables. Then, for any $t \in \mathbb{R}$, $$\ex{}{e^{t \left(f(X_1,\cdots,X_n) - \ex{}{f(X_1,\cdots,X_n)}\right)}} \le e^{\frac{t^2}{8} \sum_{i=1}^n \Delta_i^2}.$$\end{lemma}

Lemma \ref{lem:mcd-mgf} implies McDiarmid's inequality:
\begin{lemma}[\cite{McDiarmid89}]\label{lem:mcd} Let $f : \Omega^n \to \mathbb{R}$. Let $\Delta_1, \Delta_2, \cdots, \Delta_n \in \mathbb{R}$ be such that, for all $x_1, x_2, \cdots, x_n, y \in \Omega$ and all $i \in [n]$, $$\left| f(x_1, \cdots, x_{i-1}, y, x_{i+1}, \cdots, x_n) - f(x_1, \cdots, x_n) \right| \le \Delta_i.$$ Let $X_1, \cdots, X_n \in \Omega$ be independent random variables. Then, for any $\lambda\ge0$, $$\pr{}{f(X_1,\cdots,X_n) - \ex{}{f(X_1,\cdots,X_n)} \ge \lambda} \le  e^{\frac{-2\lambda^2}{\sum_{i \in [n]} \Delta_i^2}}.$$ \end{lemma}

In the spirit of McDiarmid's inequality, we have the following bound on variance.
\begin{lemma}[{\cite{Steele86}}]\label{lem:steele}
Let $f : \Omega^n \to \mathbb{R}$ and let $\Delta_1, \Delta_2, \cdots, \Delta_n \in \mathbb{R}$ be such that, for all $x_1, x_2, \cdots, x_n, y \in \Omega$ and all $i \in [n]$, $$\left| f(x_1, \cdots, x_{i-1}, y, x_{i+1}, \cdots, x_n) - f(x_1, \cdots, x_n) \right| \le \Delta_i.$$ Let $X_1, \cdots, X_n \in \Omega$ be independent random variables. Then $$\var{}{f(X_1, \cdots, X_n)} \le \frac12 \sum_{i=1}^n \ex{}{\left( f(X_1, \cdots, X_{i-1}, X_i', X_{i+1}, \cdots, X_n) - f(X_1, \cdots, X_n)\right)^2} \le \frac12 \sum_{i=1}^n \Delta_i^2,$$
where $X_i'$ denotes an independent copy of $X_i$.
\end{lemma}

\section{Composition \& Postprocessing}\label{sec:composition}
In this section, we establish two basic properties of CMI: non-adaptive composition and post-processing. These properties hold both for CMI and distribution-free CMI. The former, which we prove below, guarantees that revealing the outputs of two different processes on the data does not increase the CMI more than the sum of the CMI of each of the processes.

\begin{theorem}[Non-adaptive Composition for CMI]\label{th:nonadaptivecompCMI}
Let $A_1:\Z^n\rightarrow \W_1$ and $A_2:\Z^n\rightarrow \W_2$ be deterministic or randomized algorithms, whose internal sources of randomness are independent. The non-adaptive composition $A:\Z^n\rightarrow \W_1\times\W_2$ defined by $A(z)=(A_1(z), A_2(z))$ satisfies 
\begin{itemize}
\item[(i)] $\CMI{A}{\D}\leq\CMI{A_1}{\D}+\CMI{A_2}{\D}$ for any distribution $\D$ and 
\item[(ii)] $\CMI{A}{}\leq \CMI{A_1}{}+\CMI{A_2}{}$.
\end{itemize}
\end{theorem}
\begin{proof}
Fix $\sz\in\D^{n\times2}$ and let $S\gets \U^n$ uniformly random and independent from $\sZ$ and the randomness of $A_1$ and $A_2$. Then let $A_1(\sz_S)=F_1(S)$ and $A_2(\sz_S)=F_2(S)$. For any $\sz\in\D^{n\times2}$, it holds that $I(F_1(S); F_2(S) |S)=I(A_1(\sz_S); A_2(\sz_S) |S)=0$ since the randomness of $A_1$ is independent of the randomness of $A_2$.
\begin{claim}\label{cl:nonadaptivecomposition}
For any $S\gets\U^n$, and $F_1(S), F_2(S)$ such that $I(F_1(S); F_2(S) |S)=0$ it holds that
\[I(F_1(S), F_2(S);S)\le I(F_1(S);S)+I(F_2(S);S).\]
\end{claim}
\begin{proof}[Proof of Claim~\ref{cl:nonadaptivecomposition}]
By the chain rule of Lemma~\ref{lem:chainruleMI} it holds that
\begin{equation}\label{eq:nonadaptivecompchain1}
    I(F_1(S), F_2(S);S)= I(F_2(S);S)+I(F_1(S);S|F_2(S))
\end{equation}
To bound the second term, we invoke the chain rule on $I(F_2(S),S;F_1(S))$:
\begin{align*}
    & I(F_1(S);S)+I(F1(S);F_2(S)|S)=I(F_1(S);F_2(S))+I(F1(S);S|F_2(S))\\
    & \Rightarrow I(F1(S);S|F_2(S))=I(F_1(S);S)+I(F1(S);F_2(S)|S)-I(F_1(S);F_2(S))\\
    & \Rightarrow I(F1(S);S|F_2(S))=I(F_1(S);S)-I(F_1(S);F_2(S)) \tag{by assumption}\\
    & \Rightarrow I(F1(S);S|F_2(S))\le I(F_1(S);S) \tag{since $I(\cdot,\cdot)\ge 0$}
\end{align*}
Applying the latter bound to equation~\eqref{eq:nonadaptivecompchain1} proves the claim.
\end{proof}
Therefore, it holds that for any $\sz\in\Z^{n\times 2}$, 
\begin{equation}\label{eq:compositionforanyz}
I(A(\sz_S);S)\leq I(A_1(\sz_S);S) + I(A_2(\sz_S) ;S)
\end{equation}
Recall the definition of CMI:
$$ \CMI{A}{\D} = I(A(\sZ_S) ;S|\sZ) =\ex{\sz\gets\D^{n\times 2}}{I(A(\sz_S) ;S)}.$$
By inequality~\eqref{eq:compositionforanyz}, taking the expected values over $\sz\gets\D^{n\times 2}$ proves part (i) of the theorem.
Similarly for distribution-free CMI:
\begin{align*}
\CMI{A}{}& =\sup_{\sz\in\Z^{n\times 2}} I(A(\sz_S) ;S)\\
& = I(A(\sz^*_S) ;S) \tag{for $\sz^*=\argmax_{\sz\in\Z^{n\times 2}} I(A(\sz_S) ;S)$}\\
& \leq I(A_1(\sz^*_S) ;S) + I(A_2(\sz^*_S) ;S) \tag{by inequality~\eqref{eq:compositionforanyz}}\\
&\leq \sup_{\sz\in\Z^{n\times 2}} I(A_1(\sz_S) ;S) + \sup_{\sz\in\Z^{n\times 2}} I(A_2(\sz_S) ;S)\\
& = \CMI{A_1}{} + \CMI{A_2}{}.
\end{align*}
This proves part (ii) of the theorem and concludes the proof.
\end{proof}

The post-processing property is trivially inherited by the data-processing inequality.
\begin{theorem}[Post-processing for CMI]
    Let $A_1:\Z^n\rightarrow \W_1$ and $A_2:\W_1\rightarrow \W_2$ be deterministic or randomized algorithms, whose internal sources of randomness are independent. The post-processing algorithm $A:\Z^n\rightarrow \W_2$ defined by $A(z)=A_2(A_1(z))$ satisfies 
\begin{itemize}
\item [(i)] $\CMI{A}{\D}\leq \CMI{A_1}{\D}$ for any distribution $\D$ and 
\item[(ii)] $\CMI{A}{}\leq \CMI{A_1}{}$.
\end{itemize}
\end{theorem}
\begin{proof}
It holds that:
For any fixed $\sz\in\D^{n\times2}$, $S\rightarrow A_1(\sz_S)\rightarrow A_2(A_1(\sz_S))$ is a Markov chain. By the data-processing inequality in Lemma~\ref{lem:data-proc}, it holds that $I(S;A_1(\sz_S))\geq I(S;A_2(A_1(\sz_S)))$. 

Taking the expectation over $\sz\gets\D^{n\times 2}$, we get
\begin{align*}
    \CMI{A}{\D} &= I(A_2(A_1(\sZ_S)) ;S|\sZ) \\
   & =\ex{\sz\gets\D^{n\times 2}}{I(A_2(A_1(\sz_S));S)}\\
   &\leq \ex{\sz\gets\D^{n\times 2}}{I(A_1(\sz_S);S)} =\CMI{A_1}{\D},
\end{align*}
thus proving part (i) of the theorem.
Similarly for distribution-free CMI:
\begin{align*}
\CMI{A}{}& =\sup_{\sz\in\Z^{n\times 2}} I(A(\sz_S) ;S)\\
& = I(A(\sz^*_S) ;S) \tag{for $\sz^*=\argmax_{\sz\in\Z^{n\times 2}} I(A(\sz_S) ;S)$}\\
& \leq I(A_1(\sz^*_S) ;S) \\
&\leq \sup_{\sz\in\Z^{n\times 2}} I(A_1(\sz_S) ;S) = \CMI{A_1}{}.
\end{align*}
This proves part (ii) of the theorem and concludes the proof.
\end{proof}

These two properties demonstrate that CMI is is a robust notion of generalization. Another property that is desirable for learning algorithms is \emph{adaptive} composition. That is, if $A_1 : \Z^n \to \W_1$ and $A_2 : \Z^n \times \W_1 \to \W_2$, then consider $A : \Z^n \to \W_2$ defined by $A(z)=A_2(A_1(z),z)$. We would like to have $\CMI{A}{} \overset{?}{\le} \CMI{A_1}{} + \sup_{w \in \W_1} \CMI{A_2(\cdot, w)}{}$. However, this does \emph{not} hold. In Section~\ref{sec:ucmi}, we explore a variant of CMI which satisfies this stronger property.

\section{Methods for Bounding CMI}\label{sec:boundingCMI}
In this section, we examine several known methods for proving that an algorithm $A$ generalizes well, and we show that they all imply a bound on the CMI of the algorithm. We note that although we state and prove the following theorems for distribution-free CMI, they also imply the same bounds for CMI, since $\CMI{A}{\D}\leq \CMI{A}{}$ for any algorithm $A$ and distribution $\D$.

\subsection{Compression Schemes}\label{sec:cmi-compression}
\begin{definition}[Compression Scheme, \citet{LittlestoneW86}]\label{def:compressionscheme}
An algorithm $A:\Z^n\rightarrow \W$ has a {\em compression scheme of size $k$} if we can write $A(z) =A_2(A_1(z))$ for all $z$, where
\begin{enumerate}
    \item $A_1:\Z^n\rightarrow \Z^k$ is a ``compression algorithm'' which given a sample set $z\in\Z^n$ selects a subset $i_1, \ldots, i_k \in [n]$ and returns  $A_1(z)=(z_{i_1}, \ldots, z_{i_k})\in\Z^k$ and
    \item $A_2:\Z^k\rightarrow \W$ is an arbitrary ``encoding algorithm.''
\end{enumerate}
\end{definition}

\begin{theorem}[CMI of compression schemes]\label{th:compressionschemeCMI}\label{thm:comp-cmi}
If $A:\Z^n\rightarrow \W$ has a compression scheme of size $k$, then 
$\CMI{A}{}\leq k\cdot \log(2n)$.
\end{theorem}
\begin{proof}
Let $K=K(z)=\{i_1,\ldots, i_k\}\subset [n]$ denote the set of indices chosen by the compression algorithm $A_1$ on input $z \in \Z^n$. We will slightly abuse notation and denote by $z_K \in \Z^k$ the subset of $z\in\Z^n$ given by the indices $K \subset [n]$.
For any $\sz\in\Z^{n\times 2}$ consisting of $2n$ samples and $S\gets \U^n$:
\begin{align*}
    I(A(\sz_S); S) & = I(A_2(A_1(\sz_S));S) \\
    &\leq I((\sz_S)_{K(\sz_S)};S) \tag{Lemma~\ref{lem:data-proc}}\\
    &\leq I(K;S)\\
    &\leq H(K) \\
    &\leq k\log(2n),
\end{align*}
where the last inequality holds since the number of possible distinct values of $(\sz_S)_K $ is at most $\binom{2n}{k}\leq (2n)^k$. Therefore, $\CMI{A}{} = \sup_{\sz\in\Z^{n\times 2}} I(A(\sz_S); S) \leq k\log(2n)$.
\end{proof}

\subsection{Differential Privacy and Other Stability Notions}\label{sec:dpstab}
A recent line of work has studied the generalization guarantees of algorithms which satisfy some form of distributional stability, particularly differential privacy \cite[etc.]{DworkFHPRR15,BassilyNSSSU16}. CMI is closely related to all of these conditions. In fact, we show that (pure) differentially private algorithms, learners that satisfy KL stability, Average-Leave-One-Out KL stability, Mutual Information stability, TV stability, and $d$-bit learners  have bounded CMI.  

\nocite{HardtU14,SteinkeU15}

We start by defining some of these stability notions below.

\begin{definition}[Differential Privacy,~\cite{DworkMNS06,DworkKMMN06}]\label{def:DP}
An algorithm $A:\Z^n\rightarrow \W$ is $(\eps,\delta)$-differentially private (DP) if, for any two data sets $z,z'\in\Z^n$ that differ in a single element and for any measurable set $W\subset\W$,
\[\Pr[A(z)\in W] \leq e^{\eps}\cdot \Pr[A(z')\in W] +\delta.\]
If $\delta=0$, $A$ is $\eps$-differentially private.
\end{definition}

\begin{definition}[KL Stability,~\cite{BassilyNSSSU16}]\label{def:KLstability}
An algorithm $A:\Z^n\rightarrow \W$ is $\eps$-KL stable if, for any two data sets $z,z'\in\Z^n$ that differ in a single element,
\[\dkl{A(z)}{A(z')}\le \eps.\]
\end{definition}

\begin{definition}[Average Leave-One-Out KL Stability,~\cite{FeldmanS18a}]\label{def:ALKLstability}
An algorithm $A :(\Z^n\cup\Z^{n-1}) \to \mathcal{W}$ is $\eps$-ALKL stable if, for any data set $z\in\Z^n$,
\[\frac{1}{n}\sum_{i=1}^n \dkl{A(z)}{A(z_{-i})}\leq \eps,\]
where $z_{-i}$ denotes $z$ with the $i$-th element removed.\\
More generally, an algorithm $A :\Z^n \to \mathcal{W}$ is $\eps$-ALKL stable if there exist algorithms $A_1, A_2, \cdots, A_n : \Z^{n-1}\rightarrow \W$ such that, for all $z\in\Z^n$, we have $\frac{1}{n}\sum_{i=1}^n \dkl{A(z)}{A_i(z_{-i})}\leq \eps$.
\end{definition}
The second definition extends the notion to include algorithms that only admit inputs in $\Z^n$. We denote an $\eps$-ALKL stable algorithm as $A :\Z^n \to \mathcal{W}$ even if it satisfies the stronger first definition.

\begin{definition}[MI Stability,~\cite{RaginskyRTWX16, FeldmanS18a}]\label{def:MIstability}
An algorithm $A:\Z^n\rightarrow \W$ is $\eps$-MI stable if for any random sample set $Z \in \Z^n$ (possibly drawn from a non-product distribution)
\[\frac{1}{n}\sum_{i=1}^n I(A(Z);Z_i|Z_{-i})\leq \eps,\]
where $Z_i$ denotes the $i$-th element of $Z$ and $Z_{-i}$ denotes $Z$ with the $i$-th element removed.
\end{definition}

The following theorem states that all the above conditions imply a bound on the CMI of $A$.

\begin{theorem}[CMI of KL/ALKL/MI stable and DP algorithms]\label{th:CMIofstabilities}
If $A:\Z^n\rightarrow \W$ is an $\eps$-KL/$\eps$-ALKL/$\eps$-MI stable or $\sqrt{2\eps}$-DP algorithm, then $\CMI{A}{}\leq \eps n$.
\end{theorem}

We will base our proof on the following collection of results, which simplifies the relationship of these stability notions.

\begin{lemma}[\cite{FeldmanS18a}]\label{lem:stabilitiesandmi}
Let $A :\Z^n \to \mathcal{W}$ be a learning algorithm.
\begin{enumerate}[label=\roman*.]
    \item If $A$ is $\eps$-KL stable, then it is $\eps$-ALKL stable.
    \item If $A$ is $\sqrt{2\eps}$-DP, then it is $\eps$-ALKL stable.
    \item If $A$ is $\eps$-ALKL stable, then it is $\eps$-MI stable.
    \item If $A$ is $\eps$-MI stable and $Z$ is drawn from a product distribution over $\Z^n$, then it holds that $I(A(Z);Z)\leq \eps n.$
\end{enumerate}
\end{lemma}

We are now ready to prove the theorem.
\begin{proof}[Proof of Theorem~\ref{th:CMIofstabilities}]
Let $\mathcal{P}_{\sz_S}$ be the probability distribution of $\sz_S$ for a fixed $\sz\in\Z^{n\times2}$ and $S\gets \U^n$. Then, $\mathcal{P}_{\sz_S}$ is a product distribution for any $\sz$, including $\sz^*=\argmax_{\sz\in\Z^{n\times2}} I(A(\sz_S);\sz_S)$. By Lemma~\ref{lem:stabilitiesandmi}(iv) it holds that
\begin{equation*}\label{eq:supmi}
    \CMI{A}{}=\sup_{\sz\in\Z^{n\times2}} I(A(\sz_S);\sz_S)\leq \eps n,
\end{equation*}
which proves the theorem.
\end{proof}

We note that this theorem implies that if $A$ is $\eps$-differentially private, then $\CMI{A}{\D}\leq \eps^2n/2$ for any distribution $\D$.

The most closely related notion to ours are the \emph{$d$-bit information learners}, introduced by \citet{BassilyMNSY18}, which are algorithms that satisfy $I(A(Z);Z)\leq d$ whenever $Z$ is drawn from a product distribution. (Here we do not require that $Z$ consist of i.i.d.~samples as in the original definition \cite{BassilyMNSY18}.) Similar to the previous proof, if $A$ is a $d$-bit information learner, then $\CMI{A}{\D}\leq \CMI{A}{} \leq d$.

Lastly, we prove that if an algorithm is TV stable then it has bounded CMI. 
\begin{definition}[TV stability]
An algorithm $A:\Z^n\rightarrow \W$ is \emph{$\delta$-TV stable} if, for any two data sets $z,z'\in \Z^n$ that differ in a single element, \[\tvd(A(z),A(z')):=\sup_{W \subseteq \W}\left\vert \pr{}{A(z)\in W} - \pr{}{A(z') \in W} \right\vert\leq \delta .\]
\end{definition}
The quantity $\tvd$ defined above is known as the total variation (TV) distance or statistical distance.
Another definition of Total Variation (TV) distance is the following.
\begin{equation}\label{def:TVdist}
    \tvd(P,Q)=\frac{1}{2}\ex{x\gets Q}{\left|\frac{P(x)}{Q(x)}-1\right|}
\end{equation}
To use the latter definition, we must assume that $P(x)/Q(x)$ can be well-defined -- i.e., that $P$ is absolutely continuous with respect to $Q$. We choose to make this assumption to demonstrate the main ideas used in the proof. However, for completeness, we also provide an alternative way to prove the theorem in its full generality at the end of the proof.
\begin{theorem}[CMI of TV stable algorithms]\label{th:CMIofTV}
If $A:\Z^n\rightarrow \W$ is a $\delta$-TV stable algorithm then $\CMI{A}{}\leq \delta n$.
\end{theorem}
\begin{proof}
Let $\sz^*=\argmax_{\sz\in\Z^{n\times2}} I(A(\sz_S);S)$ and $S\gets \U^n$. Let us denote $A(\sz^*_s)$ by $F(s)$ for $s \in\{0,1\}^n$. Then
\[\CMI{A}{} = I(A(\sz^*_S);S)=I(F(S);S)\]
and it suffices to prove that $I(F(S);S)\leq \delta n$ for $S \leftarrow \U^n$.
Let us define $S_{<i}=(S_1, \ldots, S_{i-1})$, $S_{>i}=(S_{i+1}, \ldots, S_n)$, $S_{-i}=S_{<i}\circ S_{>i}$ and $S_{\leq i}=S_{<i}\circ S_i$, where $x\circ y$ denotes the concatenation of $x$ with $y$.
By the chain rule for mutual information (Lemma~\ref{lem:chainruleMI}) and by induction, 
\begin{equation}\label{eq:michain}
    I(F(S);S) =\sum_{i=1}^n I(F(S);S_i | S_{<i}).
\end{equation}
We will prove that $I(F(S);S_i | S_{<i})\leq I(F(S);S_i | S_{<i}, S_{>i})$. This holds because $S_{>i}$ is independent of $S_{<i}$ and $S_i$.
By applying the chain rule (Lemma~\ref{lem:chainruleMI}) on $I(F(S),S_{>i};S_i|S_{<i})$, for fixed $i\in[n]$,
\begin{align*}
    &I(F(S);S_i|S_{<i})+I(S_{>i};S_i|S_{<i}, F(S)) = I(S_{>i};S_i|S_{<i})+I(F(S);S_i|S_{<i}, S_{>i})\\
    &\Leftrightarrow I(F(S);S_i|S_{<i})=I(S_{>i};S_i|S_{<i})+I(F(S);S_i|S_{-i})-I(S_{>i};S_i|S_{<i}, F(S))\\
    &\Leftrightarrow I(F(S);S_i|S_{<i})=I(F(S);S_i|S_{-i})-I(S_{>i};S_i|S_{<i}, F(S)) \tag{$I(S_{>i};S_i|S_{<i})=0$}\\
    &\Rightarrow I(F(S);S_i|S_{<i})\leq I(F(S);S_i|S_{-i}). \tag{$I(\cdot,\cdot)\geq 0$}
\end{align*}

By inequality~\eqref{eq:michain}, it follows that $I(F(S);S)\leq \sum_{i=1}^n I(F(S);S_i|S_{-i})$ and it suffices to prove that \[\forall i\in[n] ~~~~~ I(F(S);S_i|S_{-i})\leq \delta.\]

For any $i\in[n]$, let $s^*_{-i}=\argmax_{x\in\{0,1\}^{n-1}}{I(F(S)|S_{-i}=x;S_i)}$. Now, let us denote the random variable $F(S)|S_{-i}=s^*_{-i}$ by $F_i(S_i)$. Then, for all $i\in [n]$, 
\[I(F(S);S_i|S_{-i})=\ex{s_{-i}\gets\U^{n-1}}{I(F(S)|S_{-i}=s_{-i};S_i)}\leq I(F(S)|S_{-i}=s^*_{-i};S_i)=I(F_i(S_i);S_i).\]

Therefore, it suffices to prove that, for uniformly random $S$,
$\forall i\in[n] ~~ I(F_i(S_i);S_i)\leq \delta.$
We fix an arbitrary $i \in [n]$. The relevant property of $F_i$ implied by TV stability is that $\tvd(F_i(0),F_i(1))\le\delta$.

Let us denote by $P_0$ and $P_1$ the probability distributions of $F_i(0)$ and $F_i(1)$, respectively. We denote by $\frac{P_0+P_1}{2}$ the convex combination of these two distributions. 
By the definition of mutual information, we have that
\begin{align*}
    I(F_i(S_i);S_i)& =\ex{r\gets\U}{\dkl{F_i(r)}{F_i(S_i)}}
     =\frac{1}{2}\dkl{F_i(0)}{F_i(S_i)} + \frac{1}{2}\dkl{F_i(1)}{F_i(S_i)}\\
    & =\frac{1}{2}\dkl{P_0}{\frac{P_0+P_1}{2}} + \frac{1}{2}\dkl{P_1}{\frac{P_0+P_1}{2}}\\
    & =\frac{1}{2}\ex{w\gets P_0}{\log\frac{2P_0(w)}{P_0(w)+P_1(w)}} + \frac{1}{2}\ex{w\gets P_1}{\log\frac{2P_1(w)}{P_0(w)+P_1(w)}} \\
    & \leq \frac{1}{2}\ex{w\gets P_0}{\left|\frac{2P_0(w)}{P_0(w)+P_1(w)}-1\right|} + \frac{1}{2}\ex{w\gets P_1}{\left|\frac{2P_1(w)}{P_0(w)+P_1(w)}-1\right|} \tag{$\log(x)\leq |x-1|$}\\
    & = \frac{1}{2}\ex{w\gets P_0}{\frac{|P_0(w)-P_1(w)|}{P_0(w)+P_1(w)}} + \frac{1}{2}\ex{w\gets P_1}{\frac{|P_1(w)-P_0(w)|}{P_0(w)+P_1(w)}}  \end{align*}\begin{align*}
    & = \frac{1}{2}\ex{w\gets P_1}{\frac{P_0(w)\cdot|P_0(w)-P_1(w)|}{P_1(w)\cdot(P_0(w)+P_1(w))}} + \frac{1}{2}\ex{w\gets P_1}{\frac{|P_0(w)-P_1(w)|}{P_0(w)+P_1(w)}}\\
    & = \frac{1}{2}\ex{w\gets P_1}{\frac{(P_0(w)+P_1(w))\cdot|P_0(w)-P_1(w)|}{P_1(w)\cdot(P_0(w)+P_1(w))}} \\
    & = \frac{1}{2}\ex{w\gets P_1}{\left|\frac{P_0(w)}{P_1(w)}-1\right|} 
    =\tvd(F_i(0),F_i(1)) \tag{by definition~\eqref{def:TVdist}}
\end{align*}

Now, recall that since $A:\Z^n\rightarrow \W$ is $\delta$-TV stable, $F:\{0,1\}^n\gets \W$ is also $\delta$-TV stable, and for any $i\in[n]$ $F_i:\{0,1\}\rightarrow \W$ is $\delta$-TV stable. 
Thus, for all $i\in[n]$, $I(F_i(S_i);S_i)\leq \tvd(F_i(0),F_i(1))\leq \delta$, which, as we argued, suffices to conclude that $\CMI{A}{}\leq \delta n$.
\end{proof}
We remark that the quantity $\mathrm{JSD}(P_0 \| P_1)=\frac{1}{2}\dkl{P_0}{\frac{P_0+P_1}{2}} + \frac{1}{2}\dkl{P_1}{\frac{P_0+P_1}{2}}$ is known as the Jensen-Shannon divergence and it is known that it can be bounded by TV distance \cite[Thm.~3]{Lin91}. Specifically, it holds that $\mathrm{JSD}(P_0\| P_1)\leq \tvd(F_i(0),F_i(1))$, which can replace the last part of the previous proof to establish the general result.

\subsection{VC dimension}\label{sec:cmi-vc}
The VC dimension is a property of a hypothesis class that implies uniform convergence. That is, bounded VC dimension implies that, for a sample of adequate size (depending linearly on the VC dimension), the true and empirical errors of \emph{all} hypotheses will be close with high probability \cite{VapnikC71,Talagrand94}. This is a sufficient condition for generalization.

For completeness, we briefly state the definition.

\begin{definition}[Vapnik-Chervonenkis dimension \cite{VapnikC71}]
Let $\mathcal{W}$ be a class of functions $h : \mathcal{X} \to \{0,1\}$. The VC dimension of $\mathcal{W}$ -- denoted $\mathsf{VC}(\mathcal{W})$ -- is the largest natural number $d$ such that there exist $x_1, \cdots, x_d \in \mathcal{X}$ and $h_1, \cdots, h_{2^d} \in \mathcal{W}$ such that, for each $j,k \in [2^d]$ with $j \ne k$, there exists some $i \in [d]$ such that $h_j(x_i) \ne h_k(x_i)$.
\end{definition}

Although the CMI of an algorithm does not only depend on its output hypothesis space, in this section we prove that there exists a connection between the two: we show that, for any class with bounded VC dimension, there exists an empirical risk minimization algorithm with corresponding CMI. (Note that it is \emph{not} the case that bounded VC dimension implies that \emph{every} empirical risk minimizer has bounded CMI.)

In contrast, for threshold functions on an unbounded domain (which have VC dimension 1), any proper empirical risk minimizer must have unbounded (unconditional) mutual information \cite{BassilyMNSY18,NachumSY18}. We remark that, if the distribution on features (i.e., the distribution of $X$ when $(X,Y) \gets \mathcal{D}$) is known (e.g., if there is a separate source of unlabelled data that does not ``count'' towards the mutual information), then it is possible to learn any class with bounded VC dimension with bounded (unconditional) mutual information \cite{XuR17}; in fact, differentially private learning is possible under this strong assumption \cite{AlonBM19}. 

In this section $\mathcal{W}$ is a class of functions $h : \mathcal{X} \to \{0,1\}$ and we work with the 0-1 loss $\ell : \mathcal{W} \times (\mathcal{X} \times \{0,1\}) \to \{0,1\}$ defined by $\ell(h,(x,y))=0 \iff h(x)=y$. We define the population loss $\ell(h,\D):=\ex{(X,Y) \gets \D}{\ell(h,(X,Y))}$, where $h : \mathcal{X} \to \{0,1\}$ and $\D$ is a distribution on $\mathcal{X} \times \{0,1\}$. For $z \in (\mathcal{X} \times \{0,1\})^n$ and $h : \mathcal{X} \to \{0,1\}$, we define $\ell(h,z) := \frac{1}{n} \sum_{i=1}^n \ell(h,z_i)$. In particular, $A : (\mathcal{X} \times \{0,1\})^n \to \mathcal{W}$ is called an empirical risk minimizer for $\mathcal{W}$ if $\ell(A(z),z) = \inf_{h \in \mathcal{W}} \ell(h,z)$ for all $z \in (\mathcal{X} \times \{0,1\})^n$.

\begin{theorem}\label{thm:vc-cmi}
Let $\Z=\X\times\{0,1\}$ and $\mathcal{H}=\{h:\X\rightarrow\{0,1\}\}$ a hypothesis class with VC dimension $d$.
Then, there exists an empirical risk minimizer $A:\Z^n\rightarrow \mathcal{H}$ such that $\CMI{A}{}\leq d\log n + 2$.
\end{theorem}

To prove Theorem \ref{thm:vc-cmi}, we introduce the following definition. Here we abuse notation by interchanging between $(\mathcal{X} \times \mathcal{Y})^n$ and $\mathcal{X}^n \times \mathcal{Y}^n$. That is, we refer to $(x,y)\in (\mathcal{X} \times \mathcal{Y})^n$ when we mean $x \in \mathcal{X}^n$ and $y \in \mathcal{Y}^n$. We also use (and abuse) the notation $(\mathcal{X} \times \mathcal{Y})^* := \bigcup_{n=0}^\infty (\mathcal{X} \times \mathcal{Y})^n$. Thus the notation $(x,y)\in (\mathcal{X} \times \mathcal{Y})^*$ means, for some $n$, we have $x \in \mathcal{X}^n$ and $y \in \mathcal{Y}^n$.

\begin{definition}[Global Consistency Property] \label{def:gc}
Let $\mathcal{W}$ be a class of functions $h : \mathcal{X} \to \mathcal{Y}$. A deterministic algorithm $A : (\mathcal{X} \times \mathcal{Y})^* \to \mathcal{W}$ is said to have the global consistency property if the following holds. Let $(x,y) \in (\mathcal{X} \times \mathcal{Y})^*$ and let $h=A(x,y)$. We require that, for any $x' \in \mathcal{X}^*$ such that $x'$ contains all the elements of $x$ (i.e., $\forall i ~ \exists j ~ x_i = x'_j$), we have $A(x',y')=h$ whenever $y'_i = h(x'_i)$ for all $i$.
\end{definition}

Intuitively, the global consistency property says the following. Suppose the algorithm is run on some labelled dataset $(x,y)$ to obtain an output hypothesis $h=A(x,y)$. If the dataset is relabelled to be perfectly consistent with $h$, then the algorithm should still output $h$. This should also hold if further examples are added to the dataset (where the additional examples are also consistent with $h$). This is a very natural and reasonable consistency property. Note that we restrict to deterministic algorithms for simplicity. 

The proof of Theorem \ref{thm:vc-cmi} now is split into the following two lemmata. Combining Lemma \ref{lem:gc+vc=cmi} with Lemma \ref{lem:gc+erm} implies the result.

\begin{lemma} \label{lem:gc+vc=cmi}
Let $A : (\mathcal{X} \times \{0,1\})^n \to \mathcal{W}$ be a deterministic algorithm, where $\mathcal{W}$ is a class of functions $h : \mathcal{X} \to \{0,1\}$ with VC dimension $d$.  Suppose $A$ (appropriately extended to inputs of arbitrary size) has the global consistency property. Then $\CMI{A}{} \le d \log n + 2$.
\end{lemma}
\begin{lemma} \label{lem:gc+erm}
Let $\mathcal{W}$ be a class of functions $h : \mathcal{X} \to \{0,1\}$. Then there exists a deterministic algorithm $A : (\mathcal{X} \times \{0,1\})^* \to \mathcal{W}$ that has the global consistency property and is an empirical risk minimizer -- that is, for all $(x,y) \in (\mathcal{X} \times \{0,1\})^*$, if $h_*=A(x,y)$, then $$\sum_i \mathbb{I}[h_*(x_i) \ne y_i] = \min_{h \in \mathcal{W}} \sum_i \mathbb{I}[h(x_i) \ne y_i].$$
\end{lemma}

To prove Lemma \ref{lem:gc+vc=cmi} we invoke the Sauer-Shelah lemma:\footnote{Vapnik and Chervonenkis proved a slightly weaker bound, namely $\left|\left\{(h(x_1),h(x_2),\cdots,h(x_m)) : h \in \mathcal{W}\right\}\right| \le m^{d+1}+1$ for $m>d$ \cite[Thm.~1]{VapnikC71}.}
\begin{lemma}[{\cite{Sauer72,Shelah72}}]\label{lem:sauer}
Let $\mathcal{W}$ be a class of functions $h : \mathcal{X} \to \{0,1\}$ with VC dimension $d$. For any $X=\{x_1, \cdots, x_m \} \subset \mathcal{X}$, the number of possible labellings of $X$ induced by $\mathcal{W}$ is \[\left|\left\{(h(x_1),h(x_2),\cdots,h(x_m)) : h \in \mathcal{W}\right\}\right| \le \sum_{k=0}^d {m \choose k} \le \left\{\begin{array}{cl}  (em/d)^d & ~~\text{ if } m \ge d \\  e^2 \cdot (m/2)^d  & ~~\text{ if } m \ge 2 \\ e \cdot m^d & ~~\text{ if } m \ge 1 \end{array} \right..\]
\end{lemma} 
Here we define ${m \choose k}=0$ if $k>m$. Thus $\sum_{k=0}^d {m \choose k}=2^m$ if $m \le d$. Note that we give three different forms of the final bound for convenience, all of which are derived from the bound 
$$\forall m \ge d ~~~ \forall x \ge 1 ~~~~~~ \sum_{k=0}^d {m \choose k} \le \sum_{k=0}^d {m \choose k} x^{d-k} \le \sum_{k=0}^m {m \choose k} x^{d-k} = \left(1+x^{-1}\right)^m \cdot x^{d} \le e^{m/x} \cdot x^d.$$

\begin{proof}[Proof of Lemma \ref{lem:gc+vc=cmi}]
Let $\sz^* = \argmax_{\sz \in \Z^{n \times 2}} I(A(\sz_S);S)$ so that $\CMI{A}{} = I(A(\sz^*_S);S)$. Since $A$ is assumed to be deterministic and $\sz^*$ is fixed, we have $I(A(\sz^*_S);S) = H(A(\sz^*_S)) \le \log|H|$, where $H:=\{A(\sz^*_s) : s \in \{0,1\}^n\} \subset \mathcal{W}$ is the set of all hypotheses that $A$ may return for a fixed supersample $\sz^*$. Our goal now is simply to bound the size of $H$. 

Intuitively, by the global consistency property, each element of $H$ can be generated by appropriately labelling the supersample $\sz^*$ according to that hypothesis and giving this re-labelled supersample as input to the algorithm. Thus the size of $H$ is at most the number of possible labellings of the supersample, which can be bounded by the Sauer-Shelah lemma (Lemma \ref{lem:sauer}).

Let $\sz^*=(\sx^*,\sy^*)$ where $\sx^* \in \mathcal{X}^{n\times 2}$ and $\sy^* \in \{0,1\}^{n \times 2}$. For $h \in \mathcal{W}$ and $s\in\{0,1\}^n$, denote $h(\sx^*) \in \{0,1\}^{n \times 2}$ and $h(\sx^*_s)\in\{0,1\}^n$ to be the labels obtained by applying the function $h$ coordinate-wise. By the global consistency property, if $h=A(\sz^*_s) = A(\sx^*_s,\sy^*_s)$ for some $s \in \{0,1\}^n$, then $h=A(\sx^*_s,h(\sx^*_s))=A(\sx^*,h(\sx^*))$. Thus $$H\subseteq H':=\{A(\sx^*,h(\sx^*)) : h \in \mathcal{W}\} = \{A(\sx^*,\sy) : \sy \in \{h(\sx^*) : h \in \mathcal{W}\}\}.$$
Hence $|H| \le |\{h(\sx^*) : h \in \mathcal{W}\}| \le \sum_{k=0}^d {2n \choose k} \le e^2 \cdot n^d$ by the Sauer-Shelah lemma, as $\sx^*$ consists of $m=2n$ points in $\mathcal{X}$. So $\CMI{A}{} \le \log|H| \le \log (e^2 \cdot n^d) = 2 + d \log n$
\end{proof}

To prove Lemma \ref{lem:gc+erm} we invoke the well-ordering theorem:
\begin{lemma}[{\cite{Zermelo04}}]\label{lem:wellorder}
Let $\mathcal{W}$ be a set. Then there exists a binary relation $\preceq$ with the following properties.
\begin{itemize}
    \item Transitivity: $~~~\forall f,g,h \in \mathcal{W} ~~~ f \preceq g \wedge g \preceq h \implies f \preceq h$
    \item Totality: $~~~\forall f,g \in \mathcal{W} ~~~ f \preceq g \vee g \preceq f$
    \item Antisymmetry: $~~~\forall f,g \in \mathcal{W} ~~~ f \preceq g \wedge g \preceq f \iff f=g$
    \item Well-order: $~~~\forall H \subset \mathcal{W} ~~ \left( ~ H \ne \emptyset ~~ \implies ~~ \exists h \in H ~~ \forall f \in H ~~ h \preceq f ~ \right)$
\end{itemize}
\end{lemma}

Intuitively, the well-ordering theorem says that there is an ordering of the elements of $\mathcal{W}$ such that each nonempty set has a least element. The standard ordering of the natural numbers has this property, but the standard ordering of the real numbers does not (e.g., an open interval such as $(0,1)$ has an infimum, but no minimum). The general statement of the well-ordering theorem is equivalent to the axiom of choice and there is no closed-form expression for a well-ordering of $\mathbb{R}$. However, on a finite computer, we could simply let $\preceq$ be the lexicographical ordering on the binary representations of elements of $\mathcal{W}$. We invoke this powerful result because we assume no structure (other than VC dimension) on $\mathcal{W}$. 

\begin{proof}[Proof of Lemma \ref{lem:gc+erm}]
An empirical risk minimizer $A : (\mathcal{X} \times \{0,1\})^n \to \mathcal{W}$ must have the property $$\forall (x,y) \in (\mathcal{X} \times \{0,1\})^n ~~~~~ A(x,y) \in \argmin_{h \in \mathcal{W}} \ell(h,(x,y)) := \left\{h \in \mathcal{W} : \ell(h,(x,y)) = \inf_{h' \in \mathcal{W}} \ell(h',(x,y))\right\}.$$
However, we must also ensure that $A$ satisfies the global consistency property. The only difficulty that arises here is when the argmin contains multiple hypotheses; we must break ties in a consistent manner. (Note that the argmin is never empty, as the 0-1 loss $\ell(h',(x,y)) = \frac{1}{n} \sum_{i=1}^n \mathbb{I}[h'(x_i) \ne y_i]$ always takes values in the finite set $\{0,1/n,2/n,3/n,\cdots,1\}$.)

Let $\preceq$ be a well-ordering of $\mathcal{W}$ (i.e., satisfying the properties in Lemma \ref{lem:wellorder}). Whenever there are multiple $h \in \mathcal{W}$ that minimize $\ell(h,(x,y))$, our algorithm $A(x,y)$ chooses the least element according to the well-ordering. In symbols, $A$ satisfies the following two properties, which also uniquely define it. $$\forall (x,y) \in (\mathcal{X} \times \{0,1\} )^*  ~~ \forall h \in \mathcal{W} ~~~ \left( \begin{array}{c} \ell(A(x,y),(x,y)) \le \ell(h,(x,y)) \\ \wedge \\ \ell(A(x,y),(x,y)) = \ell(h,(x,y)) \implies A(x,y) \preceq h \end{array} \right).$$

By construction, our algorithm $A$ is an empirical risk minimizer. It only remains to prove that it satisfies the global consistency property. To this end, let $(x,y) \in (\mathcal{X} \times \{0,1\})^n$ and let $x' \in \mathcal{X}^m$ where $x'$ contains all the elements of $x$ (i.e., $\forall i \in [n] ~ \exists j \in [m] ~ x_i=x'_j$). Let $h=A(x,y)$ and $h'=A(x',h(x'))$. We must prove that $h'=h$.

By construction, the empirical loss of $h$ on the dataset $(x',h(x'))$ is $0$. Since $h'$ is the output of an empirical risk minimizer on the dataset $(x',h(x'))$, it too has empirical loss $0$ on this dataset. In particular, $h(x'_j)=h'(x'_j)$ for all $j \in [m]$. Moreover, since $A$ breaks ties using the ordering, we have $h' \preceq h$. However, since $h$ and $h'$ agree on $x'$, they also agree on $x$ and, hence, have the same loss on the dataset $(x,y)$ -- that is, $\ell(h',(x,y))=\ell(h,(x,y)) = \inf_{h'' \in \mathcal{W}} \ell(h'',(x,y))$. This means that $A(x,y)$ outputting $h$ implies that $h \preceq h'$. Thus we conclude that $h=h'$, as required.
\end{proof}

We remark that bounded VC dimension does \emph{not} mean that \emph{every} empirical risk minimizer has bounded CMI. Our proof only shows that algorithms satisfying the global consistency property have bounded CMI. For an example of an empirical risk minimizer which lacks global consistency and has high CMI despite low VC dimension, consider the following. Threshold functions -- the class of functions $\{f_t : t \in \R\}$ where $f_t : \R \to \{0,1\}$ is defined by $f_t(x)=1 \iff x \ge t$ -- have VC dimension $d=1$. For a given dataset, there may be infinitely many $f_t$ that minimize 0-1 loss; any $t$ between the largest negative (label $0$) example and the smallest positive (label $1$) example gives $f_t$ with zero empirical loss. By exploiting this choice, an adversarial algorithm can achieve high CMI. If we think of $t$ being represented in binary, then the algorithm can pick $t$ in this interval such that the low-order bits encode the entire dataset and output $f_t$. This would have nearly-maximal CMI. In general, CMI can be increased by ``hiding'' information in the low-order bits of the output without materially affecting the algorithm; thus we need something like the global consistency property to guard against such pathologies.

Theorem \ref{thm:vc-cmi} has a $\log n$ term in the CMI bound. We believe that this is superfluous, but to remove it we may need to settle for an approximate ERM. Specifically, we conjecture that the result can be improved to the following.

\begin{conjecture}\label{conj:vc-agnostic}
There exists an absolute constant $c$ such that the following holds. For every class $\mathcal{W}$ of functions $f : \mathcal{X} \to \{0,1\}$ of VC dimension $d \ge 1$ and for every $n \ge 1$, there exists a randomized or deterministic algorithm $A : (\mathcal{X} \times \{0,1\})^n \to \mathcal{W}$ such that $\CMI{A}{} \le c \cdot d$ and, for every $(x,y) \in (\mathcal{X} \times \{0,1\})^n$, $$\ex{}{\ell(A(x,y),(x,y))} \le \inf_{h \in \mathcal{W}} \ell(h,(x,y)) + c \cdot \sqrt{\frac{d}{n}}.$$
\end{conjecture}

The error we permit in Conjecture \ref{conj:vc-agnostic} corresponds to the error of uniform convergence for a worst-case distribution \cite{Talagrand94}. In other words, the empirical error is of the same order as the generalization error.

Conjecture \ref{conj:vc-agnostic} covers the so-called agnostic setting, where we do not assume that there exists a hypothesis that achieves zero loss. It may be easier to prove a result for the realizable setting (where we do assume a consistent hypothesis exists):

\begin{conjecture}\label{conj:vc-realize}
There exist absolute constants $c$ and $c'$ such that the following holds. For every class $\mathcal{W}$ of functions $f : \mathcal{X} \to \{0,1\}$ of VC dimension $d \ge 1$ and for every $n \ge 1$, there exists a randomized or deterministic algorithm $A : (\mathcal{X} \times \{0,1\})^n \to \mathcal{W}$ such that $\CMI{A}{} \le c \cdot d$ and, for every $(x,y) \in (\mathcal{X} \times \{0,1\})^n$, if there exists $h \in \mathcal{W}$ such that $\ell(h,(x,y))=0$, then $$\ex{}{\ell(A(x,y),(x,y))} \le  c' \cdot \frac{d}{n}.$$
\end{conjecture}

\begin{conjecture}\label{conj:vc-realize2}
Conjecture \ref{conj:vc-realize} holds with $c'=0$.
\end{conjecture}

Next we prove Conjecture \ref{conj:vc-realize2} for the interesting special case of threshold functions.

\subsection{Threshold functions}\label{sec:cmi-thresholds}

In the previous section we showed that any class of small VC dimension has a learner with low (distribution-free) CMI. However, this general result yields a logarithmic factor, which we conjecture can be removed, at the expense of relaxing to an approximate empirical risk minimizer.

In this section, we show that, in at least one interesting special case, it is possible to remove the logarithmic factor for CMI.
We consider the class of threshold functions on the real line and prove that there exists a simple empirical risk minimization algorithm that has \emph{constant} CMI in the realizable setting. In contrast, Theorem \ref{thm:vc-cmi} yields a bound of $O(\log n)$, since the VC dimension of threshold functions is $d=1$.

Let us denote the hypothesis class of threshold functions by  $\T=\{f_t:\R\rightarrow \{0,1\} : t\in  \R \cup \{\infty\} \}$, where $f_t(x)=1 \iff x\geq t$ for any $x,t\in \R$ and $f_\infty$ is the constant $0$ function. 
\begin{theorem}\label{thm:thresholdsconstcmi}
There exists a deterministic algorithm $A:(\R\times\{0,1\})^* \rightarrow \T$ such that $\CMI{A}{}\leq 2$ and $A$ always outputs a threshold function realizing its dataset if one exists -- that is, for all $(x,y) \in (\R\times\{0,1\})^*$, if there exists $f_{t_*} \in \T$ such that $f_{t_*}(x_i)=y_i$ for all $i$, then $A(x,y)=f_t$ such that $f_t(x_i)=y_i$ for all $i$.
\end{theorem}

\begin{proof}
Define $A:(\R\times\{0,1\})^* \rightarrow \T$ by
$$A(z)=f_{\min \{x : (x,1)\in z\}\cup\{\infty\}}.$$
In other words, $A(z)$ outputs $f_t$ where $t$ is the smallest positive example in $z$. (If there are no positive examples, it outputs $f_\infty$.)

Firstly, this algorithm outputs a threshold realizing the dataset if one exists: Let $(x,y) \in (\R \times \{0,1\})^n$. Suppose $f_{t_*} \in \T$ satisfies $f_{t_*}(x_i)=y_i$ for all $i$. Consider $i \in [n]$ such that $y_i=1$. Then $x_i \ge t_*$ and $x_i \ge t=\min (\{x_i : i \in [n], ~ y_i=1\} \cup \{\infty\})$, as required. Now consider $i \in [n]$ such that $y_i=0$. Then $x_i < t_*$. We also have $t_* < t$. Hence $x_i < t$, as required.

It only remains to analyze the CMI.
Fix an arbitrary $\sz\in\Z^{n\times 2}$. We must show that $I(A(\sz_S);S)\le 2$ when $S$ is uniform on $\{0,1\}^n$.
We have $I(A(\sz_S);S)\le H(A(\sz_S))$.

Let $t_1 < t_2 < \cdots < t_m < t_{m+1}=\infty$ be an enumeration of $\{x : (x,1)\in \sz\}\cup\{\infty\}$.\footnote{For notational simplicity, we assume that there are no repeated values. Repeated values only decrease the entropy.} Then, for every $s \in \{0,1\}^n$, there is some $j \in [m+1]$ such that $A(\sz_s)=t_j$.

Now we consider the distribution of $A(\sz_S)$. Firstly, if $(t_1,1) \in \sz_s$, then $A(\sz_s)=f_{t_1}$. (Obviously, if $(t_j,1) \notin \sz_s$, then $A(\sz_s)\ne f_{t_j}$ for all $j \in [m]$.) Thus $\pr{S}{A(\sz_S)=f_{t_1}}=1/2$, since there is a $1/2$ probability that $(t_1,1)$ is selected. Next, if $(t_1,1) \notin \sz_s$ and $(t_2,1) \in \sz_s$, then $A(\sz_s)=f_{t_2}$ and we have $\pr{S}{A(\sz_S)=f_{t_2} \mid A(\sz_S) \ne f_{t_1}}\in \{1/2,1\}$ and hence $\pr{S}{A(\sz_S)=f_{t_2}}\in \{1/4,1/2\}$. If $(t_1,1)$ and $(t_2,1)$ are ``coupled'' in the sense that their inclusion in $\sz_s$ is determined by the same bit $s_i$ (i.e, $\sz_{i,1}=(t_1,1)$ and $\sz_{i,2}=(t_2,1)$ or vice versa), then the conditional probability is $1$, as one of the two must be included in $\sz_s$. However, if their inclusion is determined by different bits, then the probability is $1/2$. Note that, if $\pr{S}{A(\sz_S)=f_{t_2} \mid A(\sz_S) \ne f_{t_1}}=1$, then $\pr{S}{A(\sz_S)=f_{t_j}}=0$ for all $j>2$.

By continuing this reasoning, we can characterize the distribution of $A(\sz_S)$ as a truncated geometric distribution: There exists $k \in [m+1]$ such that, for all $j \in [m+1]$ $$\pr{S}{A(\sz_S)=f_{t_j}}= \left\{\begin{array}{cl} 2^{-j} & ~\text{ if } j < k \\ 2^{-j+1} & ~\text{ if } j=k \\ 0 & ~\text{ if } j>k \end{array}\right..$$ Specifically, $k$ is either $m+1$ or the smallest value such that $(t_{k},1)$ and $(t_j,1)$ are coupled for some $j<k$.

This allows us to bound the entropy as required:
\begin{align*}
H(A(\sz_S))= \sum_{j=1}^{k-1} 2^{-j}\log(2^j)+ 2^{1-k} \log(2^{k-1}) = (2-2^{2-k})\cdot \log 2 < 2.
\end{align*}

\end{proof}

We remark that, while the CMI of the algorithm in the proof of Theorem \ref{thm:thresholdsconstcmi} is constant, the (unconditional) mutual information is infinite whenever the distribution is continuous. Thus this gives a stark separation between CMI and MI.

\subsection{Pseudodeterministic Learning}\label{sec:cmi-pseudo}
In this section, we give an example that demonstrates how assumptions on the distribution $\D$ can provide improved bounds for the CMI. Indeed, the CMI can \emph{decrease} towards $0$ as the sample size $n$ increases.

More generally, these techniques apply to ``pseudodeterministic'' learning problems. That is, whenever the algorithm will with high probability output a single correct hypothesis. This is to be contrasted with the situation where there are many correct or nearly-correct hypotheses and running the algorithm again on fresh data will likely result in a different one being chosen.

The following lemma shows how a pseudodeterministic learner attains low CMI as $p \to 0$.

\begin{lemma}[CMI of Pseudodeterministic Learners]\label{lem:pseudodeterministic}
Let $A : \Z^n \to \W$ be a randomized or deterministic algorithm. Let $\D$ be a distribution on $\Z$. Let $p \in [0,1]$ and suppose there exists $w_* \in \W$ such that $\pr{Z \gets \D^n, A}{A(Z)=w_*}=1-p$. Then $$\CMI{A}{\D} \le p \cdot (\log(1/p)+1+ \log|\W|).$$
\end{lemma}
\begin{proof}
By definition, $\CMI{A}{\D}=I(A(\sZ_S);S | \sZ) \leq H(A(\sZ_S) | \sZ) \leq H(A(\sZ_S))$, where $\sZ \gets \D^{n \times 2}$ and $S \in \{0,1\}^n$ is uniform and independent. Since $\sZ_S$ simply has the distribution $\D^n$, we have $\CMI{A}{\D} \leq H(A(Z))$ where $Z \gets \D^n$.

Let $Z \gets \D^n$. Let $R \in \{0,1\}$ be a random variable indicating the event $A(Z)=w_*$. Then $p=\pr{Z,A}{R=1}$. Since $R$ is a deterministic function of $A(Z)$, we have
\begin{align*}
   H(A(Z)) & = H(A(Z),R) \\
   & =H(R) + H(A(Z)|R) \tag{by Lemma~\ref{lem:chainruleH}}\\
   & = (-p\log(p)-(1-p)\log(1-p)) + (p\cdot H(A(Z)|R=1) + (1-p)\cdot H(A(Z)|R=0))\\
   & \leq -p\log(p)-(1-p)\log(1-p)+p\cdot \log|\W|+0\\
   & \leq -p\log(p)+p +p\cdot  \log |\W|,
\end{align*}
since $-(1-p)\log(1-p)\leq p$ for any $p\in[0,1]$.
\end{proof}

Next we apply this lemma to an example.

\paragraph{Learning Parities} Consider the Boolean hypercube $\X=\{0,1\}^d$. For $w\in\{0,1\}^d$, the parity function $f_w:\X\rightarrow \{0,1\}$ is defined as the parity of the number of $1$'s appearing in the coordinates selected by the indicator vector $w$: \[f_w(x)=\langle w, x \rangle \bmod{2} \text{ for } x\in\X.\]

In the problem of learning parities (without noise), we are given $n$ samples from $\X$ labeled by a parity function $f_{w^*}$ and we aim to learn $w^*$. More formally, given samples $z=((x_1,y_1), \ldots, (x_n, y_n))$ drawn from a distribution $\D$ such that $\forall i\in[n]~y_i=f_{w^*}(x_i)$, the learner returns $A(z)=w$ such that $\forall i\in[n]~y_i=f_{w}(x_i)$. This can be done efficiently using Gaussian elimination (in the noiseless case).

Since the class of all parity functions is a finite class of size $2^d$, we can already prove that for any distribution $\D$ and any proper learner $A$, $\CMI{A}{D}\leq I(A(\sZ_S);S | \sZ) \leq H(A(\sZ_S)) \leq \log(2^d) < d$.

The following theorem shows that making an assumption on $\D$ allows us to prove a better bound, which goes to zero as the number of samples $n$ increases.

\begin{theorem}\label{th:pseudodeterministic}
Let $f_{w^*}$ be a parity function and let $\D$ be the uniform distribution over the set $\{(x,f_{w^*}(x)) \mid x\in\{0,1\}^d\}$. Then any consistent learner $A:\Z^n\rightarrow \{0,1\}^d$ satisfies $\CMI{A}{\D}\leq O(n\cdot 2^{d-n})$.
\end{theorem}
The assumption that features are uniform can be relaxed; it suffices if the distribution on features has small bias \cite{NaorN90}. Likewise, we can relax the requirement that the distribution be realizable and allow a small amount of adversarial noise in the labels. These assumptions are only for the simplicity of the statement and proof.
\begin{proof}


By Lemma \ref{lem:pseudodeterministic} it suffices to bound $p=\pr{}{A(Z) \ne w_*}$. Since $A$ is assumed to be a consistent learner, $p$ is at most the probability that there exists some parity function $f_w$ where $w \ne w^*$ that is consistent with the data. We apply a union bound over all $2^d-1$ other parities.

Fix a parity function $f_w$ where $w \ne w^*$. For a uniformly random $X \in \{0,1\}^d$, we have $\pr{X}{f_w(X)=f_{w^*}(X)}=\frac12$. For $n$ samples, we have $\pr{Z \leftarrow \D^n}{\forall i \in [n] ~~ f_w(X_i)=Y_i} = 2^{-n}$, where $Z_i=(X_i,Y_i) \in \{0,1\}^d \times \{0,1\}$ for all $i \in [n]$.
Thus 
\begin{align*}
p &\le \pr{Z \leftarrow \D^n}{\exists w \in \{0,1\}^d \setminus \{ w^* \} ~~~ \forall i \in [n] ~~~ f_w(X_i) = Y_i} \\
&\le \sum_{ w \in \{0,1\}^d \setminus \{ w^* \}} \pr{Z \leftarrow \D^n}{\forall i \in [n] ~~ f_w(X_i) = Y_i} \\
&\le (2^d-1) \cdot 2^{-n} \le 2^{d-n}.
\end{align*}
Hence $\CMI{A}{\D}\leq 2^{d-n} \cdot (\log(2^{n-d})+1 + d \log 2) = 2^{d-n} \cdot (n \log 2 + 1)$.
\end{proof}

\section{Bounded CMI Implies Generalization}

In this section we translate CMI into generalization results. We provide a variety of results to demonstrate the versatility of CMI.

\subsection{Linear Loss Bounds for the Agnostic Setting}\label{sec:gen-basic}

We begin with the simplest form of a generalization bound for a linear loss.

\begin{theorem}\label{thm:lossagnostic}
Let $\mathcal{D}$ be a distribution on $\mathcal{Z}$.  Let $A :\mathcal{Z}^n \to \mathcal{W}$ be a randomized algorithm. Let $\ell : \mathcal{W} \times \mathcal{Z} \to \mathbb{R}$ be an arbitrary (deterministic, measurable) function.

Suppose there exists $\Delta : \mathcal{Z}^2 \to \mathbb{R}$ such that $|\ell(w,z_1)-\ell(w,z_2)| \le \Delta(z_1,z_2)$ for all $z_1,z_2\in\mathcal{Z}$ and $w \in \mathcal{W}$.

Then $$\left|\ex{Z \gets \D^n ,A}{\ell(A(Z),Z)-\ell(A(Z),\mathcal{D})}\right| \le  \sqrt{\frac{2}{n} \cdot \CMI{A}{\D} \cdot \ex{(Z_1,Z_2) \leftarrow \mathcal{D}^2}{\Delta(Z_1,Z_2)^2}}.$$
\end{theorem}
In particular, if the range of $\ell$ is in the interval $[a,b]$, then we can set $\Delta$ to be the constant $b-a$:
\begin{corollary}
Let $\mathcal{D}$ be a distribution on $\mathcal{Z}$. Let $A :\mathcal{Z}^n \to \mathcal{W}$ be a randomized algorithm. Let $\ell : \mathcal{W} \times \mathcal{Z} \to [0,1]$ be an arbitrary function.

Then $$\left|\ex{Z \gets \D^n, A}{\ell(A(Z),Z)-\ell(A(Z),\mathcal{D})}\right| \le  \sqrt{\frac{2}{n} \cdot \CMI{A}{\D}}.$$
\end{corollary}

By the triangle inequality, we can also set $\Delta(z_1,z_2) = \sup_{w_1} |\ell(w_1,z_1)| + \sup_{w_2} |\ell(w_2,z_2)|$, Then $\ex{(Z_1,Z_2)\leftarrow\D^2}{\Delta(Z_1,Z_2)^2} \le 4 \ex{Z\leftarrow\D}{\sup_w (\ell(w,Z))^2}$.
\begin{corollary}
Let $\mathcal{D}$ be a distribution on $\mathcal{Z}$.  Let $A :\mathcal{Z}^n \to \mathcal{W}$ be a randomized algorithm. Let $\ell : \mathcal{W} \times \mathcal{Z} \to \mathbb{R}$ be an arbitrary function.

Then $$\left|\ex{Z \gets \D^n, A}{\ell(A(Z),Z)-\ell(A(Z),\mathcal{D})}\right| \le  \sqrt{\frac{8}{n} \cdot \CMI{A}{\D} \cdot \ex{Z \leftarrow \D}{\sup_{w \in \mathcal{W}} \left( \ell(w,Z) \right)^2}}.$$
\end{corollary}

The proof of Theorem~\ref{thm:lossagnostic} serves as a prototype for most proofs of this section, as they follow the same structure with some necessary key modifications.
\begin{proof}[Proof of Theorem \ref{thm:lossagnostic}]
Let $\sZ \in \mathcal{Z}^{n \times 2}$ consist of $2n$ independent samples from $\mathcal{D}$. Let $S,S' \in \{0,1\}^n$ be uniformly random. Assume $\sZ$, $S$, $S'$, and the randomness of $A$ are all independent.
Let $W=A(\sZ_S)$. 
Let $f_{\sz}(w,s) = \ell(w,\sz_s)-\ell(w,\sz_{\overline{s}})$. 
Then
\begin{align*}
    &\ex{Z \gets \D^n, A}{\ell(A(Z),Z)-\ell(A(Z),\mathcal{D})}\\
    &= \ex{\sZ,S,A}{\ell(A(\sZ_S),\sZ_S)-\ell(A(\sZ_S),\sZ_{\overline{S}})}\\
    &= \ex{\sZ,S,A}{f_{\sZ}(A(\sZ_S),S)}\\
    &\le \ex{\sZ}{\inf_{t>0} \frac{I(A(\sZ_S);S) + \log \ex{S,S',A}{e^{t f_{\sZ}(A(\sZ_S),S')}}}{t}} \tag{By Lemma \ref{lem:sup}}  \end{align*}\begin{align*}
    &\le \inf_{t>0} \frac{I(A(\sZ_S);S|Z) + \ex{\sZ}{\log \ex{W,S'}{e^{t f_{\sZ}(W,S')}}}}{t}\\
    &= \inf_{t>0} \frac{\CMI{A}{\D} + \ex{\sZ}{\log \ex{W}{\prod_{i=1}^n \ex{S'_i}{e^{\frac{t}{n} (\ell(W,(\sZ_{S'})_i)-\ell(W,(\sZ_{\overline{S'}})_i))}}}}}{t}\\
    &= \inf_{t>0} \frac{\CMI{A}{\D} + \ex{\sZ}{\log \ex{W}{\prod_{i=1}^n \ex{S'_i}{e^{\frac{t}{n} (1-2S'_i) (\ell(W,\sZ_{i,1})-\ell(W,\sZ_{i,2}))}}}}}{t}\\
    &\le \inf_{t>0} \frac{\CMI{A}{\D} + \ex{\sZ}{\log \ex{W}{\prod_{i=1}^n e^{\frac{t^2}{2n^2} (\ell(W,\sZ_{i,1})-\ell(W,\sZ_{i,2}))^2}}}}{t} \tag{By Lemma \ref{lem:hoeffding}}\\
    &\le \inf_{t>0} \frac{\CMI{A}{\D} + \ex{\sZ}{\sup_{w \in \mathcal{W}}\sum_{i=1}^n \frac{t^2}{2n^2} (\ell(w,\sZ_{i,1})-\ell(w,\sZ_{i,2}))^2}}{t}\\
    &= \inf_{t>0} \frac{\CMI{A}{\D}}{t} + \frac{t}{2n} \ex{\sZ}{\sup_{w \in \mathcal{W}} \frac{1}{n}\sum_{i=1}^n (\ell(w,\sZ_{i,1})-\ell(w,\sZ_{i,2}))^2}\\
    &= \sqrt{\frac{2}{n} \cdot \CMI{A}{\D} \cdot \ex{\sZ}{\sup_{w \in \mathcal{W}} \frac{1}{n}\sum_{i=1}^n (\ell(w,\sZ_{i,1})-\ell(w,\sZ_{i,2}))^2}}\\
    &\le \sqrt{\frac{2}{n} \cdot \CMI{A}{\D} \cdot \ex{\sZ}{\sup_{w \in \mathcal{W}} \frac{1}{n}\sum_{i=1}^n (\Delta(\sZ_{i,1},\sZ_{i,2}))^2}}\\
    &= \sqrt{\frac{2}{n} \cdot \CMI{A}{\D} \cdot \ex{(Z_1,Z_2)\leftarrow\mathcal{D}^2}{(\Delta(Z_1,Z_2))^2}}.
\end{align*}
This concludes the proof of the theorem.
\end{proof}

\subsubsection{Expected Absolute Error}
We now prove a slightly stronger statement, albeit losing a small constant factor in our bound.
\begin{theorem}\label{thm:cmigen-absloss}
Let $\mathcal{D}$ be a distribution on $\mathcal{Z}$.  Let $A :\mathcal{Z}^n \to \mathcal{W}$ be a randomized algorithm. Let $\ell : \mathcal{W} \times \mathcal{Z} \to \mathbb{R}$ be an arbitrary function.

Suppose there exists $\Delta : \mathcal{Z}^2 \to \mathbb{R}$ such that $|\ell(w,z_1)-\ell(w,z_2)| \le \Delta(z_1,z_2)$ for all $z_1,z_2\in\mathcal{Z}$ and $w \in \mathcal{W}$.

Then $$\ex{Z \gets \D^n, A}{\left|\ell(A(Z),Z)-\ell(A(Z),\mathcal{D})\right|} \le  \sqrt{\frac{2}{n} \cdot \left( \CMI{A}{\D} + \log 2 \right) \cdot \ex{(Z_1,Z_2) \leftarrow \mathcal{D}^2}{\Delta(Z_1,Z_2)^2}}.$$
\end{theorem}
\begin{proof}
This proof closely follows that of Theorem \ref{thm:lossagnostic}, with a few key modifications.
Let $\sZ \in \mathcal{Z}^{n \times 2}$ consist of $2n$ independent samples from $\mathcal{D}$. Let $S,S' \in \{0,1\}^n$ be uniformly random. Assume $\sZ$, $S$, $S'$, and the randomness of $A$ are independent.
Let $f_{\sz}(w,s) = \ell(w,\sz_s)-\ell(w,\sz_{\overline{s}})$. 
Then
\begin{align*}
    &\ex{Z \gets \D^n, A}{\left|\ell(A(Z),Z)-\ell(A(Z),\mathcal{D})\right|}\\
    &\le \ex{\sZ,S,A}{\left|\ell(A(\sZ_S),\sZ_S)-\ell(A(\sZ_S),\sZ_{\overline{S}})\right|}\\
    &= \ex{\sZ,S,A}{\left|f_{\sZ}(A(\sZ_S),S)\right|}\\
    &\le \inf_{t>0} \frac{I(A(\sZ_S);S|Z) + \ex{\sZ}{\log \ex{W,S'}{e^{t \left| f_{\sZ}(W,S') \right|}}}}{t}\\
    &\le \inf_{t>0} \frac{\CMI{A}{\D} + \ex{\sZ}{\log \ex{W,S'}{e^{t f_{\sZ}(W,S')}+e^{-t f_{\sZ}(W,S')}}}}{t}\\
    &\le \inf_{t>0} \frac{\CMI{A}{\D}+ \log 2}{t} + \frac{t}{2n} \ex{\sZ}{\sup_{w \in \mathcal{W}} \frac{1}{n}\sum_{i=1}^n (\ell(w,\sZ_{i,1})-\ell(w,\sZ_{i,2}))^2}\\
    &\le \sqrt{\frac{2}{n} \left(\CMI{A}{\D}+\log 2\right) \ex{(Z_1,Z_2)\leftarrow\mathcal{D}^2}{(\Delta(Z_1,Z_2)^2}}.
\end{align*}
\end{proof}

\subsubsection{Expected Squared Error}
Our last theorem for linear loss functions gives a stronger result, bounding the expected squared difference between the empirical and true loss, again only losing in small constant factors.
\begin{theorem}\label{thm:cmigen-linsquared}
Let $\mathcal{D}$ be a distribution on $\mathcal{Z}$.  Let $A :\mathcal{Z}^n \to \mathcal{W}$ be a randomized algorithm. Let $\ell : \mathcal{W} \times \mathcal{Z} \to \mathbb{R}$ be an arbitrary (deterministic, measurable) function.

Suppose there exists $\Delta : \mathcal{Z}^2 \to \mathbb{R}$ such that $|\ell(w,z_1)-\ell(w,z_2)| \le \Delta(z_1,z_2)$ for all $z_1,z_2\in\mathcal{Z}$ and $w \in \mathcal{W}$.

Then \begin{align*}
    \ex{Z \gets \D^n, A}{\left(\ell(A(Z),Z)-\ell(A(Z),\mathcal{D})\right)^2}  &\le \inf_{u\in\left(0,1\right)} \frac{2 \cdot \CMI{A}{} - \log (1-u)}{u \cdot n} \cdot \ex{(Z_1,Z_2) \gets \D^2}{ \Delta(Z_1,Z_2)^2}\\
    &\le \frac{3 \cdot \CMI{A}{} + \log 3}{n} \cdot  \ex{(Z_1,Z_2) \gets \D^2}{ \Delta(Z_1,Z_2)^2}.
\end{align*}

Furthermore, if $\Delta(z_1,z_2) \le \Delta$ for all $z_1,z_2 \in \Z$, then
\begin{align*}
\ex{Z \gets \D^n, A}{\left(\ell(A(Z),Z)-\ell(A(Z),\mathcal{D})\right)^2} &\le \inf_{u\in\left(0,1\right)} \frac{2 \cdot \CMI{A}{\D} - \log (1-u)}{u \cdot n} \cdot  \Delta^2\\
&\le \frac{3 \cdot \CMI{A}{\D}+\log 3}{n} \cdot \Delta^2.
\end{align*}
\end{theorem}
As before, if the range of $\ell$ is bounded to $[a,b]$, then we can set $\Delta$ to be the constant $b-a$:
\begin{corollary}
Let $\mathcal{D}$ be a distribution on $\mathcal{Z}$.  Let $A :\mathcal{Z}^n \to \mathcal{W}$ be a randomized algorithm. Let $\ell : \mathcal{W} \times \mathcal{Z} \to [0,1]$ be an arbitrary bounded function.

Then 
\begin{align*}
    \ex{Z \gets \D^n, A}{\left(\ell(A(Z),Z)-\ell(A(Z),\mathcal{D})\right)^2} 
    &\le \inf_{u\in\left(0,1\right)} \frac{2\cdot \CMI{A}{\D} - \log (1-u)}{u \cdot n}
    \le \frac{3\cdot \CMI{A}{\D}+\log 3}{n}.
\end{align*}
and, for all $\varepsilon>0$,
\begin{align*}
    \pr{Z \gets \D^n,A}{\left|\ell(A(Z),Z)-\ell(A(Z),\mathcal{D})\right| \ge \varepsilon} 
    &\le \inf_{u\in\left(0,1\right)} \frac{2\cdot \CMI{A}{\D} - \log (1-u)}{u \cdot n \cdot \varepsilon^2}
    \le \frac{3\cdot  \CMI{A}{\D}+\log 3}{n \cdot \varepsilon^2}.
\end{align*}
\end{corollary}
\begin{proof}[Proof of Theorem \ref{thm:cmigen-linsquared}] Again, this proof closely follows that of Theorem \ref{thm:lossagnostic}, with a few key modifications. Let $\sZ \in \mathcal{Z}^{n \times 2}$ consist of $2n$ independent samples from $\mathcal{D}$. Let $S,S' \in \{0,1\}^n$ be uniformly random. Let $G \leftarrow \mathcal{N}(0,1)$ be a standard Gaussian. Assume $\sZ$, $S$, $S'$, $G$, and the randomness of $A$ are independent. Note that for all $\lambda \in \mathbb{R}$ we have $\ex{}{e^{\lambda G}} = e^{\lambda^2/2}$ and, if $\lambda < 1/2$, then $\ex{}{e^{\lambda G^2}} = \frac{1}{\sqrt{1-2\lambda}}$. Let $f_{\sz}(w,s) = \ell(w,\sz_s)-\ell(w,\sz_{\overline{s}})$. 
We have
\begin{align*}
    &\ex{Z \gets \D^n, A}{\left(\ell(A(Z),Z)-\ell(A(Z),\mathcal{D})\right)^2}\\
    &\le \ex{\sZ,S, A}{\left(\ell(A(\sZ_S),\sZ_S)-\ell(A(\sZ_S),\sZ_{\overline{S}})\right)^2} \tag{Jensen's inequality}\\
    &= \ex{\sZ,S,A}{f_{\sZ}(A(\sZ_S),S)^2}\\
    &\le \ex{\sZ}{\inf_{t>0} \frac{I(A(\sZ_S);S) + \log \ex{W,S'}{e^{t f_{\sZ}(W,S')^2}}}{t}}\tag{Lemma \ref{lem:sup}}\\
    &= \ex{\sZ}{\inf_{t>0} \frac{I(A(\sZ_S);S) + \log \ex{W,S'}{\ex{G}{e^{G \cdot \sqrt{2t} f_{\sZ}(W,S')}}}}{t}}\\
    &= \ex{\sZ}{\inf_{t>0} \frac{I(A(\sZ_S);S) + \log \ex{G,W}{\prod_{i=1}^n \ex{S'_i}{e^{\frac{G \sqrt{2t}}{n} (\ell(W,(\sZ_{S'})_i)-\ell(W,(\sZ_{\overline{S'}})_i))}}}}{t}}\\
    &= \ex{\sZ}{\inf_{t>0} \frac{I(A(\sZ_S);S) + \log \ex{G,W}{\prod_{i=1}^n \ex{S'_i}{e^{\frac{G\sqrt{2t}}{n} (1-2S'_i) (\ell(W,\sZ_{i,1})-\ell(W,\sZ_{i,2}))}}}}{t}}\\
    &\le \ex{\sZ}{\inf_{t>0} \frac{I(A(\sZ_S);S) + \log \ex{G,W}{\prod_{i=1}^n e^{\frac{2tG^2}{2n^2} (\ell(W,\sZ_{i,1})-\ell(W,\sZ_{i,2}))^2}}}{t}}\\ \tag{Lemma \ref{lem:hoeffding}}  
    &\le \ex{\sZ}{\inf_{t>0} \frac{I(A(\sZ_S);S) + \log \ex{G,W}{e^{\frac{tG^2}{n^2}\sum_{i=1}^n  \Delta(\sZ_{i,1},\sZ_{i,2})^2}}}{t}} \end{align*}\begin{align*}
    &= \ex{\sZ}{\inf_{t\in\left(0,\frac{n^2}{2\sum_{i=1}^n  \Delta(\sZ_{i,1},\sZ_{i,2})^2)}\right)} \frac{I(A(\sZ_S);S) + \log \left(\frac{1}{\sqrt{1-2\frac{t}{n^2}\sum_{i=1}^n  \Delta(\sZ_{i,1},\sZ_{i,2})^2}}\right)}{t}}\\
    &= \ex{\sZ}{\inf_{u\in\left(0,1\right)} \frac{I(A(\sZ_S);S) -\frac12 \log (1-u)}{u \cdot n^2} \cdot 2\sum_{i=1}^n \Delta(\sZ_{i,1},\sZ_{i,2})^2}.
\end{align*}
To complete the proof, we now have two bounds: First
\begin{align*}
&\ex{\sZ}{\inf_{u\in\left(0,1\right)} \frac{I(A(\sZ_S);S) -\frac12 \log (1-u)}{u \cdot n^2} \cdot 2\sum_{i=1}^n \Delta(\sZ_{i,1},\sZ_{i,2})^2} \\
&\le \inf_{u\in\left(0,1\right)} \frac{2 \cdot \sup_{\sz \in \Z^{n \times 2}} I(A(\sz_S);S) - \log (1-u)}{u \cdot n^2} \cdot  \ex{\sZ}{\sum_{i=1}^n \Delta(\sZ_{i,1},\sZ_{i,2})^2}\\
&= \inf_{u\in\left(0,1\right)} \frac{2 \cdot \CMI{A}{} - \log (1-u)}{u \cdot n} \cdot \ex{(Z_1,Z_2) \gets \D^2}{ \Delta(Z_1,Z_2)^2}\\
&\le \frac{3 \cdot \CMI{A}{} + \log 3}{n} \cdot  \ex{(Z_1,Z_2) \gets \D^2}{ \Delta(Z_1,Z_2)^2},
\end{align*}
where the final inequality follows from setting $u=2/3$.

Second, if $\Delta(z_1,z_2) \le \Delta$ for all $z_1,z_2 \in \Z$, then
\begin{align*}
&\ex{\sZ}{\inf_{u\in\left(0,1\right)} \frac{I(A(\sZ_S);S) -\frac12 \log (1-u)}{u \cdot n^2} \cdot 2\sum_{i=1}^n \Delta(\sZ_{i,1},\sZ_{i,2})^2} \\
&\le \ex{\sZ}{\inf_{u\in\left(0,1\right)} \frac{I(A(\sZ_S);S) -\frac12 \log (1-u)}{u \cdot n^2} \cdot 2\sum_{i=1}^n \Delta^2}\\
&\le \inf_{u\in\left(0,1\right)} \frac{\ex{\sZ}{I(A(\sZ_S);S)} -\frac12 \log (1-u)}{u \cdot n^2} \cdot 2\sum_{i=1}^n \Delta^2\\
&= \inf_{u\in\left(0,1\right)} \frac{2 \cdot \CMI{A}{\D} - \log (1-u)}{u \cdot n} \cdot  \Delta^2\\
&\le \frac{3 \cdot \CMI{A}{\D}+\log 3}{n} \cdot \Delta^2.
\end{align*}
\end{proof}

\subsection{Linear Loss Bound for the Realizable Setting}\label{sec:gen-realizable}

Next we give a tighter bound for the setting where the expected empirical loss of the algorithm is zero or close to zero. This arises in the realizable setting (i.e., there exists a hypothesis in the class attaining zero population loss) and also the overfitted setting (i.e., the algorithm outputs a hypothesis that is contorted to fit the dataset). This setting is also referred to as the interpolating setting \cite{NegreaDR19}.

First we give a statement for the simple case where the empirical error is exactly zero.

\begin{theorem}\label{thm:cmigen-realize1}
Let $\mathcal{D}$ be a distribution on $\mathcal{Z}$. Let $A :\mathcal{Z}^n \to \mathcal{W}$ be a randomized algorithm. Let $\ell : \mathcal{W} \times \mathcal{Z} \to [0,1]$ be an arbitrary bounded function.

Suppose $\ex{Z \gets \D^n,A}{\ell(A(Z),\sZ_S)}=0$.

Then $$\ex{Z \gets \D^n, A}{\ell(A(Z),\mathcal{D})} \le  \frac{\CMI{A}{\D}}{n \cdot \log 2} \le 1.443 \cdot \frac{\CMI{A}{\D}}{n}.$$
\end{theorem}
We remark that the constants of this theorem are tight. It shows that we can obtain non-vacuous generalization bounds as soon as $\CMI{A}{\D}<n \cdot \log 2$ and $\CMI{A}{\D} \le n \cdot \log 2$ always holds.
\begin{proof}
Let $\sZ \in \mathcal{Z}^{n \times 2}$ consist of $2n$ independent samples from $\mathcal{D}$. Let $S,S' \in \{0,1\}^n$ be uniformly random. Assume $\sZ$, $S$, $S'$, and the randomness of $A$ are independent.
Let $u,t>0$ be determined later.
Let $f_{\sz}(w,s) = \ell(w,\sz_{\overline{s}})-u \cdot \ell(w,\sz_s)$.  Then
\begin{align*}
    &\ex{Z \gets \D^n,A}{\ell(A(Z),\mathcal{D})}\\
    &= \ex{\sZ,S,A}{\ell(A(\sZ_S),\sZ_{\overline{S}})-u\cdot\ell(A(\sZ_S),\sZ_S)}\\
    &= \ex{\sZ,S,A}{f_{\sZ}(A(\sZ_S),S)}\\
    &\le \ex{\sZ}{\frac{I(A(\sZ_S);S) + \log \ex{S,S',A}{e^{t f_{\sZ}(A(\sZ_S),S')}}}{t}}\tag{Lemma \ref{lem:sup}}\\
    &= \frac{ \CMI{A}{\D} + \ex{\sZ}{\log \ex{S,A}{\prod_{i=1}^n \ex{S'_i}{e^{\frac{t}{n} (\ell(A(\sZ_S),(\sZ_{\overline{S'}})_i)-u\cdot\ell(A(\sZ_S),(\sZ_{S'})_i))}}}}}{t}\\
    &= \frac{ \CMI{A}{\D} + \ex{\sZ}{\log \ex{S,A}{\prod_{i=1}^n \left(\frac12\cdot e^{\frac{t}{n} (\ell(A(\sZ_S),(\sZ_{\overline{S}})_i)-u\cdot\ell(A(\sZ_S),(\sZ_{S})_i))}+\frac12\cdot e^{\frac{t}{n} (\ell(A(\sZ_S),(\sZ_{{S}})_i)-u\cdot\ell(A(\sZ_S),(\sZ_{\overline{S}})_i))}\right)}}}{t}\\
    &= \frac{ \CMI{A}{\D} + \ex{\sZ}{\log \ex{S,A}{\prod_{i=1}^n \left(\frac12\cdot e^{\frac{t}{n} (\ell(A(\sZ_S),(\sZ_{\overline{S}})_i))}+\frac12\cdot e^{\frac{t}{n} (-u\cdot\ell(A(\sZ_S),(\sZ_{\overline{S}})_i))}\right)}}}{t} \tag{since $\ell(A(\sZ_S),\sZ_S)=0$}\\
    &= \frac{ \CMI{A}{\D} + \ex{}{\log \ex{}{\prod_{i=1}^n \left(\frac12\cdot e^{X_i}+\frac12\cdot e^{-u \cdot X_i}\right)}}}{t},
\end{align*}
where we define $X_i = \frac{t}{n} \ell(A(\sZ_S),(\sZ_{\overline{S}})_i) \in [0,t/n]$. Now we set the parameters $u,t>0$. First we set $t=n \cdot \log 2$ so that $X_i \le \log 2$. We will take $u \to \infty$. Note that the above bound holds for all values of $u,t>0$. So we can take the infimum over all values of $u>0$. If $X_i > 0$, then $\inf_{u>0} \frac12 \cdot e^{X_i} + \frac12 \cdot e^{-u \cdot X_i} = \frac12 \cdot e^{X_i} \le 1$. If $X_i =0$, then $\frac12 \cdot e^{X_i} + \frac12 \cdot e^{-u \cdot X_i} = 1$ for all $u$. Since the expression is continuous and increasing, we can substitute in the bound $1$ to obtain
\begin{align*}
    \ex{Z \gets \D^n,A}{\ell(A(Z),\mathcal{D})}
    &\le \frac{ \CMI{A}{\D} + \ex{}{\log \ex{}{\prod_{i=1}^n 1}}}{n \cdot \log 2} = \frac{\CMI{A}{\D}}{n \cdot \log 2} \le 1.443 \cdot \frac{\CMI{A}{\D}}{n}.
\end{align*}
\end{proof}

Next we give a statement where we do not assume the empirical loss is exactly zero, but the statement is most useful when it it nearly zero. This is a more general statement, but we attain slightly worse constants.

\begin{theorem}\label{thm:cmigen-realize2}
Let $\mathcal{D}$ be a distribution on $\mathcal{Z}$. Let $A :\mathcal{Z}^n \to \mathcal{W}$ be a randomized algorithm. Let $\ell : \mathcal{W} \times \mathcal{Z} \to [0,1]$ be an arbitrary bounded function.
Then $$\ex{Z \gets \D^n, A}{\ell(A(Z),\mathcal{D})} \le  2 \cdot \ex{Z \gets \D^n, A}{\ell(A(Z),Z)} + \frac{3 \cdot \CMI{A}{\D}}{n}.$$
\end{theorem}
\begin{proof}
Let $\sZ \in \mathcal{Z}^{n \times 2}$ consist of $2n$ independent samples from $\mathcal{D}$. Let $S,S' \in \{0,1\}^n$ be uniformly random. Assume $\sZ$, $S$, $S'$, and the randomness of $A$ are independent.
Let $u,t>0$ be determined later.
Let $f_{\sz}(w,s) = \ell(w,\sz_{\overline{s}})-u \cdot \ell(w,\sz_s)$.  Let $c=\frac{e^{t/n}-1-t/n}{(t/n)^2}$. Then
\begin{align*}
    &\ex{Z \gets \D^n,A}{\ell(A(Z),\mathcal{D})-u\cdot \ell(A(Z),Z)}\\
    &= \ex{\sZ,S,A}{\ell(A(\sZ_S),\sZ_{\overline{S}})-u\cdot\ell(A(\sZ_S),\sZ_S)}\\
    &= \ex{\sZ,S,A}{f_{\sZ}(A(\sZ_S),S)}\\
    &\le \ex{\sZ}{\frac{I(A(\sZ_S);S) + \log \ex{S,S',A}{e^{t f_{\sZ}(A(\sZ_S),S')}}}{t}}\tag{Lemma \ref{lem:sup}}\\
    &= \frac{ \CMI{A}{\D} + \ex{\sZ}{\log \ex{S,A}{\prod_{i=1}^n \ex{S'_i}{e^{\frac{t}{n} (\ell(A(\sZ_S),(Z_{\overline{S'}})_i)-u\cdot\ell(A(\sZ_S),(\sZ_{S'})_i))}}}}}{t}\\
    &\le \frac{ \CMI{A}{\D} + \ex{\sZ}{\log \ex{S,A}{\prod_{i=1}^n \ex{S'_i}{\begin{array}{c} 1+\frac{t}{n} (\ell(A(\sZ_S),(\sZ_{\overline{S'}})_i)-u\cdot\ell(A(\sZ_S),(\sZ_{S'})_i))\\+c\frac{t^2}{n^2} (\ell(A(\sZ_S),(\sZ_{\overline{S'}})_i)-u\cdot\ell(A(\sZ_S),(\sZ_{S'})_i))^2 \end{array}}}}}{t}
    \tag{$e^x\le 1+x+cx^2$ for all $x\le t/n$.}\\
    &= \frac{ \CMI{A}{\D} + \ex{\sZ}{\log \ex{S,A}{\prod_{i=1}^n \left(\begin{array}{c} 1+\frac{t}{n} \frac{1-u}{2} \left(\ell(A(\sZ_S),(\sZ_{{S}})_i)+\ell(A(\sZ_S),(\sZ_{\overline{S}})_i)\right)\\+c\frac{t^2}{n^2}\frac{1}{2} \left(\ell(A(\sZ_S),(\sZ_{{S}})_i)-u\cdot \ell(A(\sZ_S),(\sZ_{\overline{S}})_i)\right)^2\\+c\frac{t^2}{n^2}\frac{1}{2} \left(\ell(A(\sZ_S),(\sZ_{\overline{S}})_i)-u\cdot \ell(A(\sZ_S),(\sZ_{{S}})_i)\right)^2 \end{array}\right)}}}{t}\tag{$\ex{S'}{\ell(w,(\sZ_{S'})_i))}=\ex{S'}{\ell(w,(\sZ_{\overline{S'}})_i))}=\frac12 \ell(w,(\sZ_S)_i) + \frac12 \ell(w,(\sZ_{\overline{S}})_i)$.}\\
    &= \frac{ \CMI{A}{\D} + \ex{\sZ}{\log \ex{S,A}{\prod_{i=1}^n \left(\begin{array}{c} 1+\frac{t}{n} \frac{1-u}{2} \left(X_i+\overline{X}_i\right)\\+c\frac{t^2}{n^2}\frac{1}{2} \left(X_i-u\cdot \overline{X}_i\right)^2\\+c\frac{t^2}{n^2}\frac{1}{2} \left(\overline{X}_i-u\cdot X_i\right)^2 \end{array}\right)}}}{t}\tag{Denoting $X_i := \ell(A(\sZ_S),(\sZ_{{S}})_i)$ and $\overline{X}_i := \ell(A(\sZ_S),(\sZ_{\overline{S}})_i)$.} \\
    &= \frac{ \CMI{A}{\D} + \ex{\sZ}{\log \ex{S,A}{\prod_{i=1}^n \left(\begin{array}{c} 1+\frac{t}{n} \frac{1-u}{2} \left(X_i+\overline{X}_i\right)\\+c\frac{t^2}{n^2}\frac{1}{2} \left(X_i^2-2u\cdot X_i\overline{X}_i+u^2 \cdot \overline{X}_i^2+\overline{X}_i^2 - 2u\cdot \overline{X}_i X_i + u^2 \cdot X_i^2\right) \end{array}\right)}}}{t}  \end{align*}\begin{align*}
    &= \frac{ \CMI{A}{\D} + \ex{\sZ}{\log \ex{S,A}{\prod_{i=1}^n \left(\begin{array}{c} 1+\frac{t}{n} \frac{1-u}{2} \left(X_i+\overline{X}_i\right)\\+c\frac{t^2}{n^2} \left(\frac{1+u^2}{2} \cdot (X_i^2 + \overline{X}_i^2) - 2u\cdot X_i \overline{X}_i\right) \end{array}\right)}}}{t}\\
    &\le \frac{ \CMI{A}{\D} + \ex{\sZ}{\log \ex{S,A}{\prod_{i=1}^n \left(\begin{array}{c} 1+\frac{t}{n} \frac{1-u}{2} \left(X_i+\overline{X}_i\right)\\+c\frac{t^2}{n^2} \cdot \frac{1+u^2}{2} \cdot (X_i + \overline{X}_i) \end{array}\right)}}}{t}\tag{$X_i,\overline{X}_i \in [0,1]$ and $u>0$}\\
    &= \frac{ \CMI{A}{\D} + \ex{\sZ}{\log \ex{S,A}{\prod_{i=1}^n \left(\begin{array}{c} 1+\frac{X_i+\overline{X}_i}{2}\left( \frac{t}{n} (1-u) +c\frac{t^2}{n^2} \cdot (1+u^2) \right) \end{array}\right)}}}{t}.
\end{align*}
Now we set the parameters $u,t>0$ such that $$\frac{t}{n} (1-u) +c\frac{t^2}{n^2} \cdot (1+u^2) = \frac{t}{n} (1-u) +(e^{t/n}-1-t/n) \cdot (1+u^2)\le 0.$$ Then we have $$\ex{Z \gets \D^n,A}{\ell(A(Z),\mathcal{D})-u\cdot \ell(A(Z),Z)} \le \frac{ \CMI{A}{\D} + \ex{\sZ}{\log \ex{S,A}{\prod_{i=1}^n \left( 1 +0\right)}}}{t} = \frac{\CMI{A}{\D}}{t}.$$
In order to attain the sharpest bound we must minimize $u$ and maximize $t$.

We set $u=2$ and $t=n/3$. This satisfies the constraint and proves the theorem.
\end{proof}

\subsection{Bounds for Non-Linear Loss Functions}\label{sec:gen-nonlin}

We now give bounds for nonlinear loss functions. These are our most general results. Up to constants, they subsume the bounds in the previous subsection for linear loss functions.

We begin with a probability bound.
\begin{theorem}\label{thm:cmigen-nonlin-prob}
Let $\ell : \mathcal{W} \times \mathcal{Z}^n \to \mathbb{R}$ be an arbitrary (measurable) function. For $i \in [n]$, define $\Delta_i: \mathcal{W} \times \mathcal{Z}^{n \times 2} \to \mathbb{R}$ as $$\Delta_i(w, \sz) = \sup_{s \in \{0,1\}^n} |\ell(w,\sz_s)-\ell(w,\sz_{(s_1,\cdots,s_{i-1},1-s_i,s_{i+1},\cdots,s_n)})|.$$ Define $\Delta: \mathcal{Z}^{n \times 2} \to \mathbb{R}$ by $\Delta(\sz)^2 = \sup_{w\in\mathcal{W}} \sum_{i \in [n]} \Delta_i(w,\sz)^2$.

Let $\mathcal{D}$ be a distribution on $\mathcal{Z}$. Let $\sZ \in \mathcal{Z}^{n \times 2}$ consist of $2n$ independent samples from $\mathcal{D}$. Let $S \in \{0,1\}^n$ be uniformly random. Let $A :\mathcal{Z}^n \to \mathcal{W}$ be a randomized algorithm. Assume $\sZ$, $S$, and the randomness of $A$ are independent. 

Then, for all $\lambda>0$ and all $u \ge 0$, $$\pr{\sZ,S,A}{\left|\ell(A(\sZ_S),\sZ_S)-\ell(A(\sZ_S),\sZ_{\overline{S}})\right| \ge \lambda} \le \frac{2u}{\lambda^2}\cdot \left( I(A(\sZ_S);S|\sZ) + 2\right)+  \pr{\sZ}{\Delta(\sZ)^2>u}.$$
\end{theorem}
Before delving into the proof, we remark on the interpretation of the conclusion of the above theorem. Firstly $I(A(\sZ_S);S|\sZ)=\CMI{A}{\D}$. The quantity $\Delta(\sZ)$ is a measure of the sensitivity of $\ell$. If we have a uniform upper bound on the sensitivity -- i.e., $\Delta(\sz)\le u$ for all $\sz$ -- then we can substitute this $u$ into the expression above and the second term becomes $0$. In essence, the theorem allows us to use a high probability bound on the sensitivity in place of a uniform bound.

Note that the conclusion of the theorem is that $\ell(A(\sZ_S),\sZ_S)\approx\ell(A(\sZ_S),\sZ_{\overline{S}})$ rather than $\ell(A(\sZ_S),\sZ_S)\approx\ell(A(\sZ_S),\mathcal{D}^n)$. This distinction is minor, as it is usually easy to prove that $\ell(A(\sZ_S),\sZ_{\overline{S}})\approx\ell(A(\sZ_S),\mathcal{D}^n)$ (since $A(\sZ_S)$ and $\sZ_{\overline{S}}$ are independent) such as using McDiarmid's inequality \ref{lem:mcd} and then apply the triangle inequality.

\begin{corollary}
Let $\ell : \mathcal{W} \times \mathcal{Z}^n \to \mathbb{R}$ be an arbitrary (measurable) function. For $i \in [n]$, define  $$\Delta_i = \sup_{z \in \Z^n, z' \in \Z,w\in\W} |\ell(w,z)-\ell(w,(z_1, z_2, \cdots, z_{i-1},z',z_{i+1},\cdots,z_n)|.$$ Define $\Delta = \sqrt{ \sum_{i \in [n]} \Delta_i^2}$.

Let $\mathcal{D}$ be a distribution on $\mathcal{Z}$. Let $A :\mathcal{Z}^n \to \mathcal{W}$ be a randomized algorithm. 

Then, for all $\lambda>0$, $$\pr{Z \gets \D^n,A}{\left|\ell(A(Z),Z)-\ell(A(Z),\D^n)\right| \ge 2\lambda} \le \frac{2\Delta^2}{\lambda^2}\cdot \left( \CMI{A}{\D} + 2\right) + 2 \cdot e^{\frac{-2\lambda^2}{\Delta^2}}.$$
\end{corollary}

\begin{proof}[Proof of Theorem~\ref{thm:cmigen-nonlin-prob}]
Define $f_{\sz}(w,s)=\ell(w,\sz_s)-\ell(w,\sz_{\overline{s}})$. 
Let $W=A(\sZ_S)$ and let $S'$ be an independent copy of $S$. 

Note that, since we are bounding the probability (rather than expectation as in our other proofs), we consider the indicator function: Define $\tilde{f}_{\sz}(w,s)=\mathbb{I}\left[|f_{\sz}(w,s)|\ge\lambda\right]$.

By Lemma \ref{lem:sup},
\begin{align*}
    &\pr{\sZ,S,A}{\left|\ell(A(\sZ_S),\sZ_S)-\ell(A(\sZ_S),\mathcal{D}^n)\right| \ge \lambda}\\
    &=\ex{\sZ,S,A}{\tilde{f}_{\sZ}(A(\sZ_S),S)}\\
    &\le \inf_{t>0} \frac{I(A(\sZ_S);S|\sZ) + \ex{\sZ}{\log \ex{W,S'}{e^{t \tilde{f}_{\sZ}(W,S')}}}}{t}\\
    &\le \inf_{t>0} \frac{I(A(\sZ_S);S|\sZ) + \ex{\sZ}{\sup_{w\in\mathcal{W}}\log \ex{S'}{e^{t \tilde{f}_{\sZ}(w,S')}}}}{t}.\\
\end{align*}
We now focus on upper bounding ${\log \ex{S'}{e^{t \tilde{f}_{\sz}(w,S')}}}$. 
Since $\tilde{f}_{\sz}(w,S')\in\{0,1\}$, we have $$\log \ex{S'}{e^{t \tilde{f}_{\sz}(w,S')}}=\log\left(1+\ex{S'}{\tilde{f}_{\sz}(w,S')}(e^t-1)\right) \le \min\left\{\ex{S'}{\tilde{f}_{\sz}(w,S')}(e^t-1), t\right\}.$$

Since $f_{\sz}$ has the symmetry property $f_{\sz}(w,\overline{s}) = - f_{\sz}(w,s)$, we have $\ex{}{f_{\sz}(w,S')}=0$ for any $\sz$ and any $w$. Also, for all $i \in [n]$, all $\sz\in\mathcal{Z}^{n\times 2}$, all $w\in\mathcal{W}$, and all $s\in\{0,1\}^n$, $$|f_{\sz}(w,s)-f_{\sz}(w,(s_1,\cdots,s_{i-1},1-s_i,s_{i+1},\cdots,s_n))| \le 2\Delta_i(w,\sz)$$
and $\Delta(\sz)=\sup_{w\in\mathcal{W}}\sum_{i\in[n]}\Delta_i(w,\sz)$. Thus, by Lemma \ref{lem:mcd}, for all $\sz\in\mathcal{Z}^{n\times 2}$ and $w\in\mathcal{W}$,
\begin{align*}
    \ex{S'}{\tilde{f}_{\sz}(w,S')} = \pr{S'}{|f_{\sz}(w,S')|\ge\lambda}\le 2 \cdot e^{\frac{-\lambda^2}{2\cdot \Delta(\sz)^2}}.
\end{align*}
Consequently, for $u>0$, $\sz\in\mathcal{Z}^{n\times 2}$, and $w\in\mathcal{W}$, we have
$$\log \ex{S'}{e^{t \tilde{f}_{\sz}(w,S')}} \le 2 \cdot e^{\frac{-\lambda^2}{2u}} \cdot (e^t-1)+ t \cdot \mathbb{I}[\Delta(\sz)^2>u].$$
Thus
\begin{align*}
    &\pr{\sZ,S,A}{\left|\ell(A(\sZ_S),\sZ_S)-\ell(A(\sZ_S),\mathcal{D}^n)\right| \ge \lambda}\\
    &\le \inf_{t>0 \atop u>0} \frac{I(A(\sZ_S);S|\sZ) + \ex{\sZ}{2 \cdot e^{\frac{-\lambda^2}{2u}} \cdot (e^t-1)+ t \cdot \mathbb{I}[\Delta(\sZ)^2>u]}}{t}\\
    &= \inf_{t>0 \atop u>0} \frac{I(A(\sZ_S);S|\sZ) + 2 \cdot e^{\frac{-\lambda^2}{2u}} \cdot (e^t-1)}{t}+  \pr{\sZ}{\Delta(\sZ)^2>u}\\
    &< \inf_{u>0} \frac{2u}{\lambda^2}\left(I(A(\sZ_S);S|\sZ)+ 2\right)+  \pr{\sZ}{\Delta(\sZ)^2>u},
\end{align*}
where the final inequality follows by setting $t=\lambda^2/2u$.
\end{proof}

We also have a nearly identical expectation bound for the same setting.

\begin{theorem}\label{thm:cmigen-nonlin}
Let $\ell : \mathcal{W} \times \mathcal{Z}^n \to \mathbb{R}$ be an arbitrary (measurable) function.
For $i \in [n]$, define $\Delta_i: \mathcal{W} \times \mathcal{Z}^{n \times 2} \to \mathbb{R}$ as $$\Delta_i(w, \sz) = \sup_{s \in \{0,1\}^n} |\ell(w,\sz_s)-\ell(w,\sz_{(s_1,\cdots,s_{i-1},1-s_i,s_{i+1},\cdots,s_n)})|.$$ Define $\Delta: \mathcal{Z}^{n \times 2} \to \mathbb{R}$ by $\Delta(\sz)^2 = \sup_{w\in\mathcal{W}} \sum_{i \in [n]} \Delta_i(w,\sz)^2$.

Let $\mathcal{D}$ be a distribution on $\mathcal{Z}$. Let $\sZ \in \mathcal{Z}^{n \times 2}$ consist of $2n$ independent samples from $\mathcal{D}$. Let $S \in \{0,1\}^n$ be uniformly random. Let $A :\mathcal{Z}^n \to \mathcal{W}$ be a randomized algorithm. Assume $\sZ$, $S$, and the randomness of $A$ are independent. 

Then 
$$\ex{\sZ,S,A}{\left|\ell(A(\sZ_S),\sZ_S)-\ell(A(\sZ_S),\mathcal{D}^n)\right|} \le \sqrt{2 \left(\CMI{A}{\D} + \log 2\right) \cdot \ex{\sZ}{\Delta(\sZ)^2}}$$
and
\begin{align*}
    &\ex{\sZ,S,A}{\left(\ell(A(\sZ_S),\sZ_S)-\ell(A(\sZ_S),\mathcal{D}^n)\right)^2} \\&~~~~~\le \min \left\{ \begin{array}{l} \frac{8}{3} \cdot \left(\CMI{A}{} + \log 2\right) \cdot \ex{\sZ}{\Delta(\sZ)^2}\\ \inf_{u>0 \atop p>1} \frac83 \cdot u \cdot \left(\CMI{A}{\D} + \log 2 \right) + n \cdot \pr{\sZ}{\Delta(\sZ)^2 > u}^{1-1/p} \cdot \ex{\sZ}{\Delta(\sZ)^{2p}}^{1/p} \end{array} \right..
\end{align*}
\end{theorem}
\begin{proof}
As in our other proofs, define $f_{\sz}(w,s)=\ell(w,\sz_s)-\ell(w,\sz_{\overline{s}})$. Let $W=A(\sZ_S)$ and let $S'$ be an independent copy of $S$. Since $f_{\sz}$ has the symmetry property $f_{\sz}(w,\overline{s}) = - f_{\sz}(w,s)$, we have $\ex{}{f_{\sz}(w,S')}=0$ for any $\sz$ and any $w$. We have, for all $i \in [n]$, all $\sz\in\mathcal{Z}^{n\times 2}$, all $w\in\mathcal{W}$, and all $s\in\{0,1\}^n$, $$|f_{\sz}(w,s)-f_{\sz}(w,(s_1,\cdots,s_{i-1},1-s_i,s_{i+1},\cdots,s_n))| \le 2\Delta_i(w,\sz).$$
Then, by Lemmas \ref{lem:sup} and \ref{lem:mcd-mgf},
\begin{align*}
    &\ex{\sZ,S,A}{\left|\ell(A(\sZ_S),\sZ_S)-\ell(A(\sZ_S),\mathcal{D}^n)\right|}\\
    &\le \ex{\sZ,S,A}{|f_{\sZ}(A(\sZ_S),S)|}\\
    &\le \inf_{t>0} \frac{I(A(\sZ_S);S|\sZ) + \ex{\sZ}{\log \ex{W,S'}{e^{t |f_{\sZ}(W,S')|}}}}{t}\\
    &\le \inf_{t>0} \frac{I(A(\sZ_S);S|\sZ)+ \ex{\sZ}{\log \ex{W,S'}{e^{t f_{\sZ}(W,S')}+e^{-t f_{\sZ}(W,S')}}}}{t}\\
    &\le \inf_{t>0} \frac{I(A(\sZ_S);S|\sZ) + \ex{\sZ}{\log \ex{W}{2 \cdot e^{t^2 (2\Delta(\sZ))^2/8}}}}{t}\\
    &= \inf_{t>0} \frac{I(A(\sZ_S);S|\sZ) + \log 2}{t} + \frac{t}{2} \ex{\sZ}{\Delta(\sZ)^2}\\
    &= \sqrt{2 \left(I(A(\sZ_S);S|\sZ) + \log 2\right) \cdot \ex{\sZ}{\Delta(\sZ)^2}},
\end{align*}
as required for the first part of the theorem.

Now we turn to the second part of the theorem.
By Lemma \ref{lem:sup},
\begin{align*}
    &\ex{\sZ,S,A}{\left(\ell(A(\sZ_S),\sZ_S)-\ell(A(\sZ_S),\mathcal{D}^n)\right)^2}\\
    &\le \ex{\sZ,S,A}{f_{\sZ}(A(\sZ_S),S)^2}\\
    &\le \ex{\sZ}{\inf_{t>0} \frac{I(A(\sZ_S);S) + \log \ex{W,S'}{e^{t f_{\sZ}(W,S')^2}}}{t}}.
\end{align*}
We now focus on upper bounding ${\log \ex{W,S'}{e^{t f_{\sZ}(W,S')^2}}}$. Let $G$ be a standard Gaussian, independent from everything else. 
By Lemma \ref{lem:mcd-mgf},
\begin{align*}
    {\log \ex{W,S'}{e^{t f_{\sZ}(W,S')^2}}} &= {\log \ex{W,S',G}{e^{\sqrt{2t} f_{\sZ}(W,S') \cdot G}}}\\
    &\le {\log \ex{W,G}{e^{2t \cdot G^2 \cdot 4\Delta(\sZ)^2/8}}}\\
    &= {\log \left(\frac{1}{\sqrt{1-2t\Delta(\sZ)^2}}\right)}\\
    &= -\frac12 \log\left(1-2t \Delta(\sZ)^2\right).
\end{align*}
Unfortunately, this bound is infinite when $t \Delta(\sZ)^2 \ge \frac12$. In this case, we have the trivial bound $${\log \ex{W,S'}{e^{t f_{\sZ}(W,S')^2}}} \le t \cdot \sup_{w,s} f_{\sZ}(w,s)^2 \le t \cdot \left(\sup_{w\in\mathcal{W}} \sum_{i\in[n]} \Delta_i(w,\sZ) \right)^2 \le t \cdot n \cdot \Delta(\sZ)^2.$$
Combining these yields, for any $\lambda\in(0,1)$, the bound $${\log \ex{W,S'}{e^{t f_{\sZ}(W,S')^2}}} \le -\frac12\log(1-\lambda) + \mathbb{I}[2t\Delta(\sZ)^2>\lambda]\cdot t \cdot n \cdot \Delta(\sZ)^2.$$
We have
\begin{align*}
    &\ex{\sZ,S,A}{\left(\ell(A(\sZ_S),\sZ_S)-\ell(A(\sZ_S),\mathcal{D}^n)\right)^2}\\
    &\le \ex{\sZ}{\inf_{t>0 \atop \lambda \in (0,1)} \frac{I(A(\sZ_S);S) + -\frac12\log(1-\lambda) + \mathbb{I}[2t\Delta(\sZ)^2>\lambda]\cdot t \cdot n \cdot \Delta(\sZ)^2}{t}}\\
    &\le \ex{\sZ}{\inf_{t>0} \frac{ I(A(\sZ_S);S) + \log 2}{t} + \mathbb{I}\left[t \cdot \Delta(\sZ)^2> \frac38 \right]\cdot n \cdot \Delta(\sZ)^2}\tag{set $\lambda=3/4$}\\
    &\le \inf_{t>0} \frac{I(A(\sZ_S);S|\sZ) + \log 2}{t} + \inf_{p>1} \ex{\sZ}{\mathbb{I}\left[t \cdot \Delta(\sZ)^2> \frac38 \right]^{\frac{1}{1-1/p}}}^{1-1/p}\cdot n \cdot \ex{\sZ}{|\Delta(\sZ)^2|^p}^{1/p}
    \tag{H\"older's inequality}\\
    &= \inf_{t>0} \frac{ \CMI{A}{\D} + \log 2}{t} + \inf_{p>1} \pr{\sZ}{t \cdot \Delta(\sZ)^2> \frac38}^{1-1/p}\cdot n \cdot \ex{\sZ}{\Delta(\sZ)^{2p}}^{1/p}\\
    &= \inf_{u>0 \atop p>1} \frac83 \cdot u \cdot \left(  \CMI{A}{\D} + \log 2 \right) + n \cdot \pr{\sZ}{\Delta(\sZ)^2 > u}^{1-1/p} \cdot \ex{\sZ}{\Delta(\sZ)^{2p}}^{1/p}.\tag{$u=\frac{3}{8t}$}
\end{align*}
We also have
\begin{align*}
    &\ex{\sZ,S,A}{\left(\ell(A(\sZ_S),\sZ_S)-\ell(A(\sZ_S),\mathcal{D}^n)\right)^2}\\
    &\le \ex{\sZ}{\inf_{t>0 \atop \lambda \in (0,1)} \frac{I(A(\sZ_S);S) + -\frac12\log(1-\lambda) + \mathbb{I}[2t\Delta(\sZ)^2>\lambda]\cdot t \cdot n \cdot \Delta(\sZ)^2}{t}}\\
    &\le \ex{\sZ}{\inf_{t>0} \frac{ I(A(\sZ_S);S) + \log 2}{t} + \mathbb{I}\left[t \cdot \Delta(\sZ)^2> \frac38 \right]\cdot n \cdot \Delta(\sZ)^2}\tag{set $\lambda=3/4$}\\
    &\le \ex{\sZ}{\inf_{t>0} \frac{ \CMI{A}{} + \log 2}{t} + \mathbb{I}\left[t \cdot \Delta(\sZ)^2> \frac38 \right]\cdot n \cdot \Delta(\sZ)^2}\\
    &\le \frac{8}{3} \cdot \left(\CMI{A}{} + \log 2\right) \cdot \ex{\sZ}{\Delta(\sZ)^2}.\tag{set $t=\frac{3}{8\Delta(\sZ)^2}$}
\end{align*}
\end{proof}

\subsection{Application: Squared Error}\label{sec:gen-squarederror}
A common example of an unbounded loss is squared loss: $\ell(w,(x,y)) = (f_w(x)-y)^2$. For example, this arises in linear regression, where $f_w(x)=\langle w , x \rangle$. 

More precisely, we have data domain $\mathcal{Z}=\mathcal{X} \times \mathbb{R}$ and, for all $w \in \mathcal{W}$, there is a function $f_w : \mathcal{X} \to \mathbb{R}$ such that the loss function  $\ell : \mathcal{W} \times \mathcal{Z} \to \mathbb{R}$ is given by $\ell(w,(x,y)) = (f_w(x)-y)^2$.

We assume $f_w$ is $c$-Lipschitz -- i.e., $|f_w(x)-f_w(x')| \le c \cdot \|x-x'\|_p$ for all $x,x' \in \mathcal{X}$ -- and $f_w(0)=0$. For example, if the function is linear $f_w(x)=\langle w, x \rangle$ and $\mathcal{W}=\{w \in \mathbb{R}^d : \|w\|_q \le c \}$, where $\|\cdot\|_p$ and $\|\cdot\|_q$ are dual norms (i.e., $1/p+1/q=1$), then $f_w : \mathbb{R}^d \to \mathbb{R}$ is $c$-Lipschitz and $f_w(0)=0$. 

Then $$0 \le \ell(w,(x,y)) = (f_w(x)-y)^2 \le 2f_w(x)^2 + 2y^2 \le 2 c^2 \|x\|_p^2 + 2y^2.$$
Thus 
\begin{align*}
    \left(\ell(w,(x_1,y_1))-\ell(w,(x_2,y_2))\right)^2 &\le \left(2 c^2 \|x_1\|_p^2 + 2y_1^2 + 2 c^2 \|x_2\|_p^2 + 2y_2^2 \right)^2\\
    &\le 16c^4 (\|x_1\|_p^4+\|x_2\|_p^4) + 16(y_1^4 + y_2^4)\\
    &=: \Delta((x_1,y_1),(x_2,y_2))^2
\end{align*}
and
\begin{align*}
    \ex{((X_1,Y_1),(X_2,Y_2)) \leftarrow \mathcal{D}^2}{\Delta((X_1,Y_1),(X_2,Y_2))^2} &= 32\ex{(X,Y)\leftarrow\D}{c^4 \|X\|_p^4 +  Y^4}.
\end{align*}
Combining this with Theorem \ref{thm:cmigen-absloss} yields (denoting $Z_i=(X_i,Y_i)$) \begin{align*}
   & \ex{Z \leftarrow \D^n \atop W \leftarrow A(Z)}{\left|\frac{1}{n}\sum_{i=1}^n \left( f_W(X_i)-Y_i\right)^2 - \ex{(X',Y') \leftarrow \D}{\left(f_W(X')-Y'\right)^2}\right|}\\
   &\le O\left(\sqrt{\frac{\CMI{A}{\D}}{n} \cdot \ex{(X,Y)\leftarrow\D}{c^4 \|X\|_p^4 +  Y^4}}\right).
\end{align*}
Thus a bound on the fourth moment of the distribution is sufficient to obtain a generalization bound from CMI.

\subsection{Application: Hinge Loss}\label{sec:gen-hinge}
Another common example of unbounded loss is hinge loss: $\ell(w,(x,y)) = \max\{0, 1 - y \cdot f_w(x) \}$, where $y \in \{+1,-1\}$ and the objective is to find $f_w$ such that $Y \cdot f_w(X) > 0$ when $(X,Y) \leftarrow \D$.

Typically $f_w(x)=\langle w , x \rangle$ is a linear function, where $\mathcal{W}=\{w\in\mathbb{R}^d:\|w\|_q \le c\}$. Thus $\ell(w,(x,y)) = \max\{0, 1 - y \cdot \langle x , w \rangle \}$ and $$|\ell(w,(x_1,y_1))-\ell(w,(x_2,y_2))| \le |y_1 \cdot \langle x_1 , w \rangle - y_2 \cdot \langle x_2 , w \rangle| = |\langle y_1x_1 - y_2x_2 , w \rangle| \le \|y_1x_1-y_2x_2\|_p \cdot \|w\|_q,$$ where $\|\cdot\|_p$ and $\|\cdot\|_q$ are dual norms (i.e., $1/p+1/q=1$). We can apply Theorem \ref{thm:cmigen-absloss} with $\Delta((x_1,y_1),(x_2,y_2)) = c\|y_1x_1-y_2x_2\|_p$. For $p=q=2$, we have $$\ex{((X_1,Y_1),(X_2,Y_2)) \leftarrow \mathcal{D}^2}{\Delta((X_1,Y_1),(X_2,Y_2))^2} = 2c^2\sum_{i=1}^d \var{(X,Y) \leftarrow \mathcal{D}}{YX_i} \le 2c^2\sum_{i=1}^d \ex{(X,Y) \leftarrow \mathcal{D}}{X_i^2},$$ which yields $$\ex{Z \gets \D^n,A}{|\ell(A(Z),Z) -\ell(A(Z),\D)|} \le O\left(\sqrt{\frac{\CMI{A}{\D}}{n} \cdot c^2\sum_{i=1}^d \var{(X,Y) \leftarrow \mathcal{D}}{YX_i}}\right).$$

\subsection{Application: Unbounded Parameter Spaces}

The applications in Sections \ref{sec:gen-squarederror} and \ref{sec:gen-hinge} assume a bounded parameter space. Some assumption of this form is necessary, as scaling up the parameters also scales up the loss. However, this assumption can be relaxed or, rather, we can incorporate the magnitude of the parameters into our generalization results. We now give one example of this. (The general idea can be combined with any of our generalization results.)

\begin{lemma}\label{lem:normalizedloss}
Let $\mathcal{D}$ be a distribution on $\mathcal{Z}$. Let $A : \mathcal{Z}^n \to \mathcal{W}$ be a randomized algorithm. Let $\ell : \mathcal{W} \times \mathcal{Z} \to \mathbb{R}$ be a (deterministic, measurable) function. 

Suppose there exist $\Delta : \mathcal{Z}^2 \to \mathbb{R}$ and $\Psi : \mathcal{W} \to (0,\infty)$ such that $$\forall w \in \mathcal{W}~~ \forall z_1,z_2 \in \mathcal{Z} ~~~~~|\ell(w,z_1)-\ell(w,z_2)| \le \Delta(z_1,z_2) \cdot \Psi(w).$$

Then, for all $\varepsilon>0$, we have $$\pr{Z \gets \D^n, A}{\left|\ell(A(Z),Z)-\ell(A(Z),\mathcal{D})\right| \ge \varepsilon \cdot \Psi(A(Z))} \le \frac{3 \cdot \CMI{A}{} + \log 3}{\varepsilon^2 n} \cdot  \ex{(Z_1,Z_2) \gets \D^2}{ \Delta(Z_1,Z_2)^2}.$$
\end{lemma}
\begin{proof}
The key idea is that we can ``normalize'' the loss function to compensate for the magnitude of the parameters and then apply our generalization results to the normalized loss function (since CMI does not depend on the loss function).
Define $\hat\ell : \mathcal{W} \times \mathcal{Z} \to \mathbb{R}$ by $$\hat\ell(w,z) = \frac{\ell(w,z)}{\Psi(w)}.$$
Then $|\hat\ell(w,z_1)-\hat\ell(w,z_2)| \le \Delta(z_1,z_2)$ for all $z_1,z_2\in\mathcal{Z}$. Thus, by Theorem \ref{thm:cmigen-linsquared},
$$
    \ex{Z \gets \D^n, A}{\left(\hat\ell(A(Z),Z)-\hat\ell(A(Z),\mathcal{D})\right)^2}  \le \frac{3 \cdot \CMI{A}{} + \log 3}{n} \cdot  \ex{(Z_1,Z_2) \gets \D^2}{ \Delta(Z_1,Z_2)^2}.
$$
By Markov's inequality, for all $\varepsilon>0$, this gives
$$
    \pr{Z \gets \D^n, A}{\left|\hat\ell(A(Z),Z)-\hat\ell(A(Z),\mathcal{D})\right| \ge \varepsilon}  \le \frac{3 \cdot \CMI{A}{} + \log 3}{\varepsilon^2 n} \cdot  \ex{(Z_1,Z_2) \gets \D^2}{ \Delta(Z_1,Z_2)^2}.
$$
Finally, we have
$$
    \pr{Z \gets \D^n, A}{\left|\hat\ell(A(Z),Z)-\hat\ell(A(Z),\mathcal{D})\right| \ge \varepsilon} = \pr{Z \gets \D^n, A}{\left|\ell(A(Z),Z)-\ell(A(Z),\mathcal{D})\right| \ge \varepsilon \cdot \Psi(A(Z))}.
$$
\end{proof}

We remark that both squared error and hinge loss satisfy the assumptions of Lemma \ref{lem:normalizedloss}.

For $\ell(w,(x,y)) = (\langle w , x \rangle - y)^2$, we have \begin{align*}
    |\ell(w,(x_1,y_1))-\ell(w,(x_2,y_2))| &\le 2\|w\|_q^2 \cdot (\|x_1\|_p^2+\|x_2\|_p^2) + 2(y_1^2+y_2^2)\\
    &\le \underbrace{2(\|w\|_q^2+1)}_{\Psi(w)} \cdot \underbrace{(\|x_1\|_p^2+\|x_2\|_p^2 + y_1^2+y_2^2)}_{\Delta((x_1,y_1),(x_2,y_2))},
\end{align*}
where $\|\cdot\|_p$ and $\|\cdot\|_q$ are dual norms.

Similarly, for $\ell(w,(x,y)) = \max\{0,1- y \cdot \langle w , x \rangle \}$, we have 
\begin{align*}
    |\ell(w,(x_1,y_1))-\ell(w,(x_2,y_2))| &\le |\langle w, y_1x_1-y_2x_2 \rangle| \\
    &\le \underbrace{\|w\|_q}_{\Psi(w)} \cdot \underbrace{\|y_1x_1-y_2x_2\|_p}_{\Delta((x_1,y_1),(x_2,y_2))}.
\end{align*}

\subsection{Application: Area Under the ROC Curve}\label{sec:auroc}

The Area Under the Receiver Operating Characteristics Curve (abbreviated AUC or AUROC) is one of the most commonly-used statistics for measuring (and comparing) the performance of classifiers, with a storied history dating back to the analysis of radar systems in the 1940s. 

In short, the ROC curve characterizes the tradeoff between false positives and false negatives. Namely, we consider a classifier that outputs a score or probability, rather than just a binary prediction. This real number is then converted to a binary prediction by thresholding -- values above the threshold are positive, while values below are negative. Increasing this threshold will result in fewer false positives, but also fewer true positives (i.e., more false negatives). The ROC curve is obtained by varying the threshold and plotting the false positive and true positive rates.

The area under the ROC curve provides a simple summary of the tradeoff captured by the full ROC curve. It is a real number between $0$ and $1$, where $1$ corresponds to a perfect classifier (i.e., there exists a threshold with no false positive and no false negatives) and $\frac12$ corresponds to a useless classifier (i.e., random guessing). The advantage of the AUROC over simply reporting classifier accuracy is that it avoids having to choose a threshold and it is always calibrated so that $\frac12$ corresponds to a useless classifier. (In contrast, an accuracy of $90\%$ is meaningless if $90\%$ of the examples are negative, but meaningful if there are equal numbers of positives and negatives.)

The AUROC can be directly defined (without reference to the ROC curve) as the probability that a random positive example receives a higher score than a random negative example. This is also known as the Wilcoxon-Mann-Whitney statistic. Formally, consider a distribution $\mathcal{D}$ on an instance space $\mathcal{Z}$ which is partitioned into positive instances $\mathcal{Z}_+ \subset \mathcal{Z}$ and negative instances $\mathcal{Z}_-=\mathcal{Z}\setminus \mathcal{Z}_+$ and a scoring function $f : \mathcal{Z} \to \mathbb{R}$. Then the AUROC is defined as $$\AUC(f,\mathcal{D}) := \begin{array}{c} \pr{Z_+,Z_- \leftarrow \mathcal{D}^2}{f(Z_+)>f(Z_-)|Z_+\in\mathcal{Z}_+,Z_-\notin\mathcal{Z}_+} ~~~~~\\~~~+ \frac12 \pr{Z_+,Z_- \leftarrow \mathcal{D}^2}{f(Z_+)=f(Z_-)|Z_+\in\mathcal{Z}_+,Z_-\notin\mathcal{Z}_+} \end{array}.$$
The second term in the above definition equates to saying that ties are broken at random (hence the factor $\frac12$); this term is often ignored, as the probability of a tie is usually small. We can equivalently define $\AUC(f,\mathcal{D})=\ex{Z_+,Z_- \leftarrow \mathcal{D}^2}{c(f,Z_+,Z_-)|Z_+\in\mathcal{Z}_+,Z_-\notin\mathcal{Z}_+}$ where\footnote{Our generalization result for AUROC (Theorem \ref{thm:auc})  holds for any function $c : \mathcal{W} \times \mathcal{Z} \times \mathcal{Z} \to [0,1]$, but this definition should be kept in mind.} $$c(f,Z_+,Z_-) = \mathbb{I}[f(Z_+)>f(Z_-)] + \frac12 \mathbb{I}[f(Z_+)=f(Z_-)].$$

To estimate the AUROC, we use the empirical AUROC: $$\AUC(f,z) = \frac{\sum_{i,j\in[n]} \mathbb{I}[z_i \in \mathcal{Z}_+]\cdot \mathbb{I}[z_j \notin \mathcal{Z}_+] \cdot c(f,z_i,z_j)}{\sum_{i,j\in[n]} \mathbb{I}[z_i \in \mathcal{Z}_+]\cdot \mathbb{I}[z_j \notin \mathcal{Z}_+]}.$$
If all of $z$ is positive or all of it is negative, then the above definition yields $\frac00$ and it is impossible to meaningfully estimate the AUROC in this case. In this case, we arbitrarily define the empirical AUROC to be $\frac12$, although any other arbitrary value could be used. Note that, with the exception of this degenerate case, the expectation of the empirical AUROC is indeed equal to the true AUROC.

We remark that it is possible to obtain AUROC bounds directly from classification accuracy bounds. In particular, for any function $f$, distribution $\D$, and threshold $\tau$, we have $$\AUC(f,\D) \ge 1-\pr{Z_+ \leftarrow \mathcal{D}}{f(Z_+)\le \tau | Z_+ \in \mathcal{Z}_+}-\pr{Z_- \leftarrow \mathcal{D}}{f(Z_-)> \tau | Z_- \notin \mathcal{Z}_+} \ge 1-\frac{\ell(f,\D)}{\min\{p,1-p\}},$$ where $p=\pr{Z \leftarrow \D}{Z\in\mathcal{Z}_+}$ and $\ell(f,z) = \mathbb{I}[z \in \mathcal{Z}_+] \cdot \mathbb{I}[f(z)\le\tau] + \mathbb{I}[z \notin \mathcal{Z}_+] \cdot \mathbb{I}[f(z)>\tau]$. In particular, applying Theorem \ref{thm:cmigen-realize1} to the above yields the following corollary.
\begin{corollary}
Let $\D$ be a distribution on $\mathcal Z$ and $\mathcal{Z}_+ \subset \mathcal{Z}$ with $0<p:=\pr{Z \leftarrow \D}{Z \in \mathcal{Z}_+} < 1$.
Let $A : \mathcal{Z}^n \to \mathbb{R}^{\mathcal{Z}}$ be a deterministic or randomized algorithm that outputs a score function $f : \mathcal{Z} \to \mathbb{R}$. Suppose the thresholded score function always has zero empirical error -- i.e., if $z \in \mathcal{Z}^n$ and $f=A(z)$, then $z_i \in \mathcal{Z}_+ \iff f(z_i)>0$ for all $i \in [n]$.  Then $$\ex{Z \leftarrow \D^n, A}{\AUC(A(Z),\D)} \ge 1 - \frac{1.5 \cdot \CMI{A}{\D}}{n \cdot \min\{p,1-p\}}.$$
\end{corollary}

We now prove that CMI implies generalization for the AUROC in all cases. This is an application of our generalization results for nonlinear loss functions, specifically Theorem \ref{thm:cmigen-nonlin-prob}. 

\begin{theorem}\label{thm:auc}
Let $\mathcal{D}$ be a distribution on $\mathcal{Z} $. Let $A :\mathcal{Z}^n \to \mathcal{W}$ be a randomized algorithm.

Let $\mathcal{Z}_+ \subseteq \mathcal{Z}$. Assume $0< p=\ex{Z \leftarrow \mathcal{D}}{Z \in \mathcal{Z}_+} <1$. Let $c : \mathcal{W} \times \mathcal{Z} \times \mathcal{Z} \to [0,1]$. Define $\AUC : \mathcal{W} \times \mathcal{Z}^n \to [0,1]$ by $$\AUC(w,z) = \left\{ \begin{array}{cl} \frac{\sum_{i,j \in [n]} \mathbb{I}[z_i \in \mathcal{Z}_+]\mathbb{I}[z_j \notin \mathcal{Z}_+]c(w,z_i,z_j)}{\sum_{i,j \in [n]} \mathbb{I}[z_i \in \mathcal{Z}_+]\mathbb{I}[z_j \notin \mathcal{Z}_+]} & \text{ if } \exists i,j\in [n] ~ z_i \in \mathcal{Z}_+ \wedge z_j \notin \mathcal{Z}_+\\ \frac12 & \text{ otherwise} \end{array} \right..$$
 Similarly, for $w \in \mathcal{W}$, define $$\AUC(w,\mathcal{D}) =  \ex{(Z,Z') \leftarrow \mathcal{D}^2}{c(w,Z,Z')| Z\in\mathcal{Z}_+,Z'\notin\mathcal{Z}_+}.$$
Then, for any $\varepsilon\in(0,1)$, 
$$\pr{\sZ, S, A}{\left|\AUC(A(\sZ_S),\sZ_S) - \AUC(A(\sZ_S),\mathcal{D}) \right| \le \varepsilon} \ge 1 - \frac{ 48 \cdot \CMI{A}{\D} + 149}{\varepsilon^2 p(1-p) n}.$$
\end{theorem}
It can be shown \cite{BirnbaumK57} that, for a fixed $w$, we have $\var{Z \leftarrow \mathcal{D}^n}{\AUC(w,Z)} \approx \frac{1}{p(1-p)n}$ in the worst case, so this bound attains roughly the correct dependence on the parameters considered.
\begin{proof}
Let $\sZ \in \mathcal{Z}^{n \times 2}$ consist of $2n$ independent samples from $\mathcal{D}$. Let $S \in \{0,1\}^n$ be uniformly random and independent from $\sZ$.
Define $g : \mathcal{W} \times \mathcal{Z}^n \to \mathbb{R}$ and $h : \mathcal{Z}^n \to \mathbb{R}$ by $$g(w,z) = \sum_{i,j \in [n]} \mathbb{I}[z_i \in \mathcal{Z}_+]\mathbb{I}[z_j \notin \mathcal{Z}_+]c(w,z_i,z_j)$$ and $$h(z) = \sum_{i \in [n]} \mathbb{I}[z_i \in \mathcal{Z}_+].$$

For $i \in [n]$, define $\Delta_i: \mathcal{W} \times \mathcal{Z}^{n \times 2} \to \mathbb{R}$ as $$\Delta_i(w, \sz) = \sup_{s \in \{0,1\}^n} |g(w,\sz_s)-g(w,\sz_{(s_1,\cdots,s_{i-1},1-s_i,s_{i+1},\cdots,s_n)})|.$$ Define $\Delta: \mathcal{Z}^{n \times 2} \to \mathbb{R}$ by $\Delta(\sz)^2 = \sup_{w\in\mathcal{W}} \sum_{i \in [n]} \Delta_i(w,\sz)^2$.

Our proof now boils down to the next four claims. Note that the probabilities are over $\sZ\gets \D^{n\times 2}, S\gets \U^n$ and the randomness of $A$, all of which we omit for ease of notation.
\begin{itemize}
    \item Claim A: $$\pr{}{\left|\AUC(A(\sZ_S),\sZ_S) - \frac{g(A(\sZ_S),\sZ_S)}{p(1-p)n^2}\right| \le \frac{\varepsilon}{4}} \ge1-\frac{16}{\varepsilon^2 p (1-p)n}.$$
    \item Claim B: $$\pr{}{\left|\frac{g(A(\sZ_S),\sZ_{\overline{S}})}{p(1-p)n^2}-\AUC(A(\sZ_S),\mathcal{D})\right|\le\frac{\varepsilon}{4}}\ge1-\frac{8}{\varepsilon^2 p^2(1-p)^2 n^4}\ex{\sZ}{\Delta(\sZ)^2} - \frac{4}{\varepsilon^2 p(1-p) n^2}.$$
    \item Claim C: 
    \begin{align*}
        &\pr{}{\left|g(A(\sZ_S),\sZ_S)-g(A(\sZ_S),\sZ_{\overline{S}})\right| \le \frac{\varepsilon}{2} p(1-p)n^2} \\&~~~~~\ge 1 - \frac{48}{\varepsilon^2 p(1-p)n}\left( \CMI{A}{\D} + 2\right) - \pr{\sZ}{\Delta(\sZ)^2>6p(1-p)n^3}.
    \end{align*}
    \item Claim D: $\ex{\sZ}{\Delta(\sZ)^2} \le 4p(1-p)n^3$, $\pr{\sZ}{\Delta(\sZ)^2 \ge 6p(1-p)n^3} \le e^{-n\cdot\min\{p,1-p\}/7}$, and $\Delta(\sz)^{2} \le n^3$ for all $\sz\in\mathcal{Z}^{n \times 2}$.
\end{itemize}
Combining them with the triangle inequality and the union bound yields
\begin{align*}
    &\pr{}{\left|\AUC(A(\sZ_S),\sZ_S) - \AUC(A(\sZ_S),\mathcal{D}) \right| \le \varepsilon} \\
    &~~~~~\ge 1 - \frac{48 \cdot  \CMI{A}{\D}+148}{\varepsilon^2 p(1-p) n} - e^{-n\cdot\min\{p,1-p\}/7}.
\end{align*}
In particular, if $v=\varepsilon^2 p (1-p) n \ge 25$, then $e^{-n\cdot\min\{p,1-p\}/7} \le e^{-v/7\varepsilon^2} \le e^{-v/7} \le \frac{1}{v} = \frac{1}{\varepsilon^2 p(1-p) n}$; if $v< 25$, then the bound is vacuous anyway. Obviously, we have not made any effort to obtain tight constants.
Now we prove the claims in turn.
\begin{proof}[Proof of Claim A]
Since $\AUC(w,z) = \frac{g(w,z)}{h(z)(n-h(z))}$ as long as $h(z) \notin \{0,n\}$, we have 
\begin{align*}
&\pr{}{\left|\AUC(A(\sZ_S),\sZ_S) - \frac{g(A(\sZ_S),\sZ_S)}{p(1-p)n^2}\right| \ge \frac{\varepsilon}{4}} \\
&\le\pr{}{h(z) \in \{0,n\} \vee \left|\frac{g(A(\sZ_S),\sZ_S)}{h(\sZ_S)(n-h(\sZ_S))} - \frac{g(A(\sZ_S),\sZ_S)}{p(1-p)n^2}\right| \ge \frac{\varepsilon}{4}} \\
&=\pr{}{h(z) \in \{0,n\} \vee \left|\frac{h(\sZ_S)(n-h(\sZ_S))}{h(\sZ_S)(n-h(\sZ_S))} - \frac{h(\sZ_S)(n-h(\sZ_S))}{p(1-p)n^2}\right| \cdot \frac{g(A(\sZ_S),\sZ_S)}{h(\sZ_S)(n-h(\sZ_S))} \ge \frac{\varepsilon}{4}} \\
&\le \pr{}{h(z) \in \{0,n\} \vee \left|\frac{h(\sZ_S)(n-h(\sZ_S))}{h(\sZ_S)(n-h(\sZ_S))} - \frac{h(\sZ_S)(n-h(\sZ_S))}{p(1-p)n^2}\right|\ge \frac{\varepsilon}{4}} \\
\intertext{\hfill(Since $0 \le \frac{g(A(\sZ_S),\sZ_S)}{h(\sZ_S)(n-h(\sZ_S))} \le 1$ assuming $h(\sZ_S) \notin \{0,n\}$)}
&= \pr{}{\left|\frac{h(\sZ_S)(n-h(\sZ_S))}{p(1-p)n^2}-1\right| \ge \frac{\varepsilon}{4}}\\
\intertext{\hfill(Since $h(\sZ_S)\in\{0,n\}$ implies $\left|\frac{h(\sZ_S)(n-h(\sZ_S))}{p(1-p)n^2}-1\right|=1\ge \frac{\varepsilon}{4}$.)}
&=\pr{}{|h(\sZ_S)(n-h(\sZ_S))-pn(n-pn)| \ge \frac{\varepsilon}{4} p(1-p)n^2}\\
&\le \pr{}{|h(\sZ_S)-pn| \ge \frac{\varepsilon}{4} p(1-p)n}\tag{($|x(n-x)-y(n-y)|\le n \cdot |x-y|$ for $x,y \in [0,n]$)}\\
&\le \frac{\var{}{h(\sZ_S)}}{\left(\frac{\varepsilon}{4} p(1-p)n\right)^2}\tag{Chebyshev's inequality}\\
&= \frac{16}{\varepsilon^2 p (1-p)n},
\end{align*}
where the final equality uses the fact that $h(\sZ_S)$ is distributed as $\mathsf{Binomial}(n,p)$.
\end{proof}

\begin{proof}[Proof of Claim B] By Markov's inequality, 
\begin{align*}
&\pr{}{\left|\frac{g(A(\sZ_S),\sZ_{\overline{S}})}{p(1-p)n^2}-\AUC(A(\sZ_S),\mathcal{D})\right|\ge\frac{\varepsilon}{4}}\\
&~~~= \pr{}{\left(g(A(\sZ_S),\sZ_{\overline{S}})-p(1-p)n^2\cdot\AUC(A(\sZ_S),\mathcal{D})\right)^2 \ge \frac{\varepsilon^2}{16}p^2 (1-p)^2 n^4}\\
&~~~\le \frac{\ex{}{\left(g(A(\sZ_S),\sZ_{\overline{S}})-p(1-p)n^2\cdot\AUC(A(\sZ_S),\mathcal{D})\right)^2}}{\frac{\varepsilon^2}{16}p^2 (1-p)^2 n^4}\\
&~~~= \frac{\ex{A(\sZ_S)}{\var{\sZ_{\overline{S}}}{g(A(\sZ_S),\sZ_{\overline{S}})}+\ex{\sZ_{\overline{S}}}{g(A(\sZ_S),\sZ_{\overline{S}})-p(1-p)n^2\cdot\AUC(A(\sZ_S),\mathcal{D})}^2}}{\frac{\varepsilon^2}{16}p^2 (1-p)^2 n^4}\\
\end{align*}
For all $w \in \mathcal{W}$, we have $$\ex{}{g(w,\sZ_{\overline{S}})} = p(1-p)n(n-1)\cdot\AUC(w,\mathcal{D})$$
and, by Lemma \ref{lem:steele}, $$\var{}{g(w,\sZ_{\overline{S}})} \le \frac12 \ex{\sZ}{\Delta(\sZ)^2}.$$
Thus
\begin{align*}
&\pr{}{\left|\frac{g(A(\sZ_S),\sZ_{\overline{S}})}{p(1-p)n^2}-\AUC(A(\sZ_S),\mathcal{D})\right|\ge\frac{\varepsilon}{4}}\\
&~~~\le \frac{\frac12 \ex{\sZ}{\Delta(\sZ)^2} + p^2(1-p)^2n^2 \cdot \ex{}{\AUC(A(\sZ_S),\mathcal{D})^2}}{\frac{\varepsilon^2}{16}p^2 (1-p)^2 n^4}\\
&~~~\le \frac{8}{\varepsilon^2 p^2(1-p)^2 n^4}\ex{\sZ}{\Delta(\sZ)^2} + \frac{16}{\varepsilon^2 n^2}\\
&~~~\le \frac{8}{\varepsilon^2 p^2(1-p)^2 n^4}\ex{\sZ}{\Delta(\sZ)^2} + \frac{4}{\varepsilon^2 p(1-p) n^2},
\end{align*}
where the final inequality uses the fact that $p(1-p)\le\frac14$.
\end{proof}
\begin{proof}[Proof of Claim C]
We apply Theorem \ref{thm:cmigen-nonlin-prob} to $g$ to obtain 
\begin{align*}
&\pr{\sZ,S,A}{\left|g(A(\sZ_S),\sZ_S)-g(A(\sZ_S),\sZ_{\overline{S}})\right| \ge \frac{\varepsilon}{2} p(1-p)n^2} \\
&~\le \inf_{u>0} \frac{2u}{(\frac{\varepsilon}{2} p(1-p)n^2)^2}\left( \CMI{A}{\D} + 2\right)+  \pr{\sZ}{\Delta(\sZ)^2>u}\\
&~\le \frac{48}{\varepsilon^2 p(1-p)n}\left( \CMI{A}{\D} + 2\right) +\pr{\sZ}{\Delta(\sZ)^2>6p(1-p)n^3},
\end{align*}
where the final inequality sets $u=6p(1-p)n^3$.
\end{proof}
\begin{proof}[Proof of Claim D]
For $s\in\{0,1\}^n$ and $k \in [n]$, denote $s_{-k}=(s_1,\cdots,s_{k-1},1-s_k,s_{k+1},\cdots,s_n)$ -- i.e., $s_{-k}$ is $s$ with the $k$-th bit flipped.

For $\sz\in\mathcal{Z}^{n \times 2}$, we have $$\Delta(\sz)^2 = \sup_{w\in\mathcal{W}} \sum_{k \in [n]} \Delta_k(w,\sz)^2,$$
where
\begin{align*}
    \Delta_k(w, \sz) &= \sup_{s \in \{0,1\}^n} |g(w,\sz_s)-g(w,\sz_{s_{-k}})|\\
    &= \sup_{s \in \{0,1\}^n} \left|\sum_{i,j \in [n]} \begin{array}{c} \mathbb{I}[(\sz_s)_i \in \mathcal{Z}_+]\mathbb{I}[(\sz_s)_j \notin \mathcal{Z}_+]c(w,(\sz_s)_i,(\sz_s)_j) \\- \mathbb{I}[(\sz_{s_{-k}})_i \in \mathcal{Z}_+]\mathbb{I}[(\sz_{s_{-k}})_j \notin \mathcal{Z}_+]c(w,(\sz_{s_{-k}})_i,(\sz_{s_{-k}})_j) \end{array}\right|\\
    &\le \sup_{s \in \{0,1\}^n} \sum_{i,j \in [n]} \left|\begin{array}{c} \mathbb{I}[(\sz_s)_i \in \mathcal{Z}_+]\mathbb{I}[(\sz_s)_j \notin \mathcal{Z}_+]c(w,(\sz_s)_i,(\sz_s)_j) \\- \mathbb{I}[(\sz_{s_{-k}})_i \in \mathcal{Z}_+]\mathbb{I}[(\sz_{s_{-k}})_j \notin \mathcal{Z}_+]c(w,(\sz_{s_{-k}})_i,(\sz_{s_{-k}})_j) \end{array}\right|\\
    &= \sup_{s \in \{0,1\}^n} \begin{array}{c} \sum_{i \in [n] \setminus\{k\}} \left|\begin{array}{c} \mathbb{I}[(\sz_s)_i \in \mathcal{Z}_+]\mathbb{I}[(\sz_s)_k \notin \mathcal{Z}_+]c(w,(\sz_s)_i,(\sz_s)_k) \\- \mathbb{I}[(\sz_{s_{-k}})_i \in \mathcal{Z}_+]\mathbb{I}[(\sz_{s_{-k}})_k \notin \mathcal{Z}_+]c(w,(\sz_{s_{-k}})_i,(\sz_{s_{-k}})_k) \end{array}\right| \\+ \sum_{j \in [n] \setminus\{k\}} \left|\begin{array}{c} \mathbb{I}[(\sz_s)_k \in \mathcal{Z}_+]\mathbb{I}[(\sz_s)_j \notin \mathcal{Z}_+]c(w,(\sz_s)_k,(\sz_s)_j) \\- \mathbb{I}[(\sz_{s_{-k}})_k \in \mathcal{Z}_+]\mathbb{I}[(\sz_{s_{-k}})_j \notin \mathcal{Z}_+]c(w,(\sz_{s_{-k}})_k,(\sz_{s_{-k}})_j) \end{array}\right| \end{array}\\
    \intertext{\hfill(If $i \ne k$ and $j \ne k$, then the term is $0$. If $i = j = k$, then the term is also $0$.)}
    &= \sup_{s \in \{0,1\}^n} \begin{array}{c} \sum_{i \in [n] \setminus\{k\}} \mathbb{I}[(\sz_s)_i \in \mathcal{Z}_+]\cdot\left|\begin{array}{c} \mathbb{I}[(\sz_s)_k \notin \mathcal{Z}_+]c(w,(\sz_s)_i,(\sz_s)_k) \\- \mathbb{I}[(\sz_{s_{-k}})_k \notin \mathcal{Z}_+]c(w,(\sz_{s_{-k}})_i,(\sz_{s_{-k}})_k) \end{array}\right| \\+ \sum_{j \in [n] \setminus\{k\}} \mathbb{I}[(\sz_s)_j \notin \mathcal{Z}_+]\cdot\left|\begin{array}{c} \mathbb{I}[(\sz_s)_k \in \mathcal{Z}_+]c(w,(\sz_s)_k,(\sz_s)_j) \\- \mathbb{I}[(\sz_{s_{-k}})_k \in \mathcal{Z}_+]c(w,(\sz_{s_{-k}})_k,(\sz_{s_{-k}})_j) \end{array}\right| \end{array}\\
    \intertext{\hfill(Since $(\sz_s)_i = (\sz_{s_{-k}})_i$ when $i\ne k$.)}
    &\le \sup_{s \in \{0,1\}^n} \begin{array}{c} \sum_{i \in [n] \setminus\{k\}} \mathbb{I}[(\sz_s)_i \in \mathcal{Z}_+]\cdot \mathbb{I}[(\sz_s)_k \notin \mathcal{Z}_+ \vee (\sz_{s_{-k}})_k \notin \mathcal{Z}_+] \\+ \sum_{j \in [n] \setminus\{k\}} \mathbb{I}[(\sz_s)_j \notin \mathcal{Z}_+]\cdot\mathbb{I}[(\sz_s)_k \in \mathcal{Z}_+ \vee (\sz_{s_{-k}})_k \in \mathcal{Z}_+] \end{array}\\
    \intertext{\hfill(Since $0\le c(\cdot,\cdot) \le 1$.)}
    &= \sup_{s \in \{0,1\}^n} \sum_{i \in [n] \setminus\{k\}} \begin{array}{c}  \mathbb{I}[(\sz_s)_i \in \mathcal{Z}_+]\cdot \mathbb{I}[\sz_{k,1} \notin \mathcal{Z}_+ \vee \sz_{k,2} \notin \mathcal{Z}_+] \\+  \mathbb{I}[(\sz_s)_i \notin \mathcal{Z}_+]\cdot\mathbb{I}[\sz_{k,1} \in \mathcal{Z}_+ \vee \sz_{k,2} \in \mathcal{Z}_+] \end{array}.
    \intertext{\hfill(Since $\{(\sz_s)_k,(\sz_{s_{-k}})_k\} = \{\sz_{k,1},\sz_{k,2}\}$.)}
\end{align*}
It is immediate that $0 \le \Delta_k(w, \sz) \le n-1$ for all $w$ and $\sz$, from which the third part of the claim follows.

Next we perform a case analysis, for which we define the following three case indicators. For $i \in [n]$ and $\sz\in\mathcal{Z}^{n\times 2}$, let $$a_i^+(\sz) = \mathbb{I}[z_{i,1}\in\mathcal{Z}_+]\mathbb{I}[z_{i,2}\in\mathcal{Z}_+], \qquad a_i^-(\sz) = \mathbb{I}[z_{i,1}\notin\mathcal{Z}_+]\mathbb{I}[z_{i,2}\notin\mathcal{Z}_+],$$ and $$a_i^\pm(\sz) = \mathbb{I}[z_{i,1}\in\mathcal{Z}_+]\mathbb{I}[z_{i,2}\notin\mathcal{Z}_+] + \mathbb{I}[z_{i,1}\notin\mathcal{Z}_+]\mathbb{I}[z_{i,2}\in\mathcal{Z}_+].$$
Also define $a^+(\sz) = \sum_{i\in[n]} a^+_i(\sz)$, $a^-(\sz) = \sum_{i\in[n]} a^-_i(\sz)$, and $a^\pm(\sz) = \sum_{i\in[n]} a^\pm_i(\sz)$ for $\sz\in\mathcal{Z}^{n \times 2}$.
Now we have
\begin{align*}
    \Delta_k(w, \sz) &\le \sup_{s \in \{0,1\}^n} \sum_{i \in [n] \setminus\{k\}} \begin{array}{c}  \mathbb{I}[(\sz_s)_i \in \mathcal{Z}_+]\cdot \mathbb{I}[\sz_{k,1} \notin \mathcal{Z}_+ \vee \sz_{k,2} \notin \mathcal{Z}_+] \\+  \mathbb{I}[(\sz_s)_i \notin \mathcal{Z}_+]\cdot\mathbb{I}[\sz_{k,1} \in \mathcal{Z}_+ \vee \sz_{k,2} \in \mathcal{Z}_+] \end{array}\\
    &\le \left\{ \begin{array}{cl} n-1 & \text{ if } a_k^\pm(\sz)=1 \\ n-1-a^+(\sz)  & \text{ if } a_k^+(\sz)=1 \\ n-1-a^-(\sz) & \text{ if } a_k^-(\sz)=1\end{array} \right\}.
\end{align*}
This yields
\begin{align*}
    \Delta(\sz)^2 &= \sup_{w\in\mathcal{W}} \sum_{k \in [n]} \Delta_k(w,\sz)^2\\
    &\le a^\pm(\sz) \cdot (n-1)^2 + a^+(\sz) \cdot (n-1-a^+(\sz))^2 + a^-(\sz) \cdot (n-1-a^-(\sz))^2\\
    &\le a^\pm(\sz) \cdot n^2 + a^+(\sz) \cdot (n-a^+(\sz))^2 + a^-(\sz) \cdot (n-a^-(\sz))^2\\
    &= n^3 + a^+(\sz)^3-2na^+(\sz)^2 + a^-(\sz)^3-2na^-(\sz)^2\tag{Since $a^\pm = n - a^+ - a^-$.}\\
    &\le n^3 - na^+(\sz)^2 - na^-(\sz)^2\tag{Since $0 \le a^+(\sz) \le n$ and $0 \le a^-(\sz) \le n$.}\\
    &=n \cdot (n^2 - a^+(\sz)^2 - a^-(\sz)^2).
\end{align*}
Since $a^+(\sZ) \gets \mathsf{Binomial}(n,p^2)$ and $a^-(\sZ) \gets \mathsf{Binomial}(n,(1-p)^2)$, we have
\begin{align*}
    \ex{}{\Delta(\sZ)^2}
    &\le n \cdot \ex{}{n^2 - a^+(\sZ)^2 - a^-(\sZ)^2}\\
    &= n \cdot \left( n^2 - (np^2)^2 - np^2(1-p^2) - (n(1-p)^2)^2 - n(1-p)^2(1-(1-p)^2) \right)\\
    &= n^2 \cdot \left( n - np^4 - p^2 + p^4 -n(1-p)^4-(1-p)^2+(1-p)^4 \right)\\
    &\le n^3 \cdot (1 - p^4 - (1-p)^4)\\
    &= n^3 \cdot 2p (1-p)(2-p(1-p))\\
    &\le 4p(1-p)n^3.
\end{align*}
This establishes the first part of the claim.

It only remains to prove the second part of the claim. We consider the case $p \le \frac12$, as the case $p\ge\frac12$ is symmetric. Let $X=n(1-p)^2-a^-(\sZ)$. Thus $\ex{}{X}=0$. Then
\begin{align*}
    \Delta(\sZ)^2 
    &\le n^3 - na^-(\sz)^2\\
    &= n^3 - n(n(1-p)^2-X)^2\\
    &= n^3 - n^3(1-p)^4 + 2n^2(1-p)^2X - nX^2\\
    &\le n^3(1-(1-p)^4) + 2n^2(1-p)^2X\\
    &= n^3 \cdot (2p (1-p)(2-p(1-p)) + p^4) + 2n^2(1-p)^2X\\
    &\le n^3 \cdot 4 p (1-p) + 2n^2(1-p)^2X,
\end{align*}
since $p^4 \le 2p^2(1-p)^2$ when $0 \le p \le \frac12$. Thus, for all $\lambda \ge 0$, we have
\begin{align*}
    \pr{}{\Delta(\sZ)^2 \ge 4p(1-p)n^3+\lambda p (1-p) n^3} &\le \pr{}{2n^2(1-p)^2 X \ge \lambda p (1-p) n^3}\\
    &= \pr{}{ X \ge \frac{\lambda p n}{2(1-p)}}\\
    &= \pr{}{Y \ge \ex{}{Y} + \frac{\lambda p n}{2(1-p)}}\\
    &\le \pr{}{Y \ge \ex{}{Y} + \frac{\lambda}{4}\ex{}{Y}}\\
    &\le \inf_{t>0} \ex{}{e^{t(Y-(1+\lambda/4)\ex{}{Y})}}\\
    &\le \inf_{t>0} e^{(e^t-1-(1+\lambda/4)t)\ex{}{Y}}\\
    &\le e^{ (\lambda/4-(1+\lambda/4)\log(1+\lambda/4))\ex{}{Y}},
\end{align*}
where $Y=X+n-n(1-p)^2=n-a^-(\sZ) \gets \mathsf{Binomial}(n,1-(1-p)^2)$ and, hence, $\ex{}{Y} = n \cdot (1-(1-p)^2) = p(2-p)n \in \left[ 1.5 pn , \frac{2pn}{1-p} \right]$.
Setting $\lambda=2$ yields the bound $$p \le \frac12 \implies \pr{}{\Delta(\sZ)^2 \ge 6p(1-p)n^3} \le e^{-\ex{}{Y}/10} \le e^{-np/7}.$$
\end{proof}
Combining Claims A, B, C, and D proves the theorem.
\end{proof}

\section{Extensions}\label{sec:extensions}

In this section we present two extensions of the basic CMI framework. These allow us to capture other effects -- adaptive composition and loss stability.

\subsection{Adaptive Composition and Universal CMI}\label{sec:ucmi}

In Section \ref{sec:composition}, we showed that CMI composes non-adaptively. That is, when running two algorithms in parallel on the same dataset, the CMI adds up. We also showed that CMI is invariant to adaptive postprocessing. That is, transforming the output of a CMI algorithm in a manner that is oblivious to the data does not increase the CMI. A desireable  property is adaptive composition, which extends both postprocessing and non-adaptive composition. Unfortunately, CMI does not satisfy adaptive composition.\footnote{Consider the following sequence of two algorithms. First, run randomized response to obtain some ``sketch'' of the dataset in a low-CMI manner. Second, run an algorithm that does something ``bad'' (as far as CMI is concerned) if the sketch is correct and something ``good'' (such as outputting nothing) if the sketch is incorrect. The second algorithm should have low CMI for a sketch that is independent of the input, but the adaptive composition of these two algorithms has high CMI.} Thus we consider a strengthening of CMI that does provide this property:

\newcommand{\uCMI}[1]{\mathsf{uCMI}\left(#1\right)}
\begin{definition}[Universal CMI (uCMI) of an Algorithm]\label{def:uCMI}
Let $A:\Z^n\rightarrow \W$ be a randomized or deterministic algorithm.
The \emph{universal conditional mutual information (uCMI) of $A$ with respect to $\D$} is $$\uCMI{A}:=\sup_{\sz\in\Z^{n \times 2}}\sup_{S} I(A(\sz_S);S),$$ 
where the second supremum is over $S$ being an arbitrary random variable (independent from $A$) on $\{0,1\}^n$. Here $\sz_S \in \Z^n$ is defined by $(\sz_S)_i = \sZ_{i,S_i+1}$ for all $i \in [n]$ -- that is, $\sz_S$ is the subset of $\sz$ indexed by $S$. 
\end{definition}

It is immediate that $\CMI{A}{} \le \uCMI{A}$ for all $A$. In particular, all of the generalization properties of CMI carry over to uCMI. The key property of uCMI is adaptive composition:

\begin{theorem}[Adaptive Composition for Universal CMI]
Let $A_1:\Z^n\rightarrow \W_1$ and $A_2: \Z^n\times \W_1 \rightarrow \W_2$ be two possibly randomized algorithms such that $\uCMI{A_1}\leq b_1$ and $\uCMI{A_2(\cdot,w_1)}\leq b_2$ for all $w_1\in\W_1$. Let $A(Z):=A_2(Z,A_1(Z))$ be the adaptive composition of $A_2$ with $A_1$. Then $\uCMI{A}\leq b_1+b_2$.
\end{theorem}
\begin{proof}
By the definition,
$$\uCMI{A}=\sup_{\sz\in\Z^{n \times 2}}\sup_{S} I(A(\sz_S);S),$$
where $S$ is an arbitrary random variable on $\{0,1\}^n$ which is independent from $A$. Let $\sz^*\in \Z^{n\times2}$ and the distribution $P$ over $\{0,1\}^n$ be the arguments that maximize the above expression. Without loss of generality, we will drop the parameter $\sz^*$ from the notation, so that $F_1(S)=A_1(\sz^*_S)$, $F_2(S,w_1)=A_2(\sz^*_S,w_1)$, and $F(S)=F_2(S, F_1(S))$. Then, it holds that for $S\gets P$,
\begin{equation*}
\uCMI{A} = I(F(S);S) =\ex{s\gets P}{\dkl{F_2(s,F_1(s))}{F_2(S,F_1(S))}}.
\end{equation*}
By the chain rule for KL divergence (Lemma~\ref{lem:chainruleKL}), we have that
\begin{equation}\label{eq:adaptiveuCMI}
    \uCMI{A} \leq \ex{s\gets P}{\dkl{F_1(s)}{F_1(S)}} + \ex{s\gets P}{\ex{w_1\gets F_1(s)}{\dkl{F_2(s,w_1)}{F_2(S',w_1)}}},
\end{equation}
where $S'\gets P_{w_1}$ which is the probability distribution defined by $P_{w_1}(s)=P(s|F_1(S)=w_1)$ for any $s\in\{0,1\}^n$. Then, the RHS of inequality~\eqref{eq:adaptiveuCMI} can be written as follows.
\begin{align*}
    \uCMI{A} &  \leq I(F_1(S);S)+\ex{w_1\gets F_1(S)}{\ex{s\gets P_{w_1}}{\dkl{F_2(s,w_1)}{F_2(S',w_1)}}} \\
    & = I(F_1(S);S) + \ex{w_1\gets F_1(S)}{I(F_2(S',w_1);S')}\\
    & \leq I(F_1(S);S) + \sup_{w_1\in\W_1} I(F_2(S',w_1);S')\\
    & = I(A_1(\sz^*_S);S) + \sup_{w_1\in\W_1} I(A_2(\sz^*_{S'},w_1);S') \\
    & \leq \sup_{\sz\in\Z^{n \times 2}}\sup_{S} I(A_1(\sz_S);S) + \sup_{w_1\in\W_1} \sup_{\sz\in\Z^{n \times 2}}\sup_{S} I(A_2(\sz_S,w_1);S)\\
    &= \uCMI{A_1} + \uCMI{A_2} \\
    & \leq b_1+b_2. \tag{by assumption}
\end{align*}
This concludes the proof of the theorem.
\end{proof}

Universal CMI is of course not unique in this regard: All of the distributional stability notions that we have discussed -- differential privacy, KL stability, Average Leave-one-out KL Stability, TV stability, and MI stability -- satisfy the adaptive composition property. It remains to form a fully unified picture of adaptive composition.

A careful examination of the relevant proofs shows that our results relating VC dimension (Theorem \ref{thm:vc-cmi}) and compression schemes (Theorem \ref{thm:comp-cmi}) carry over to uCMI. However, our results for differential privacy and other stability notions do not immediately carry over. We are however able to show some weaker results:

\begin{theorem}[Universal CMI of DP algorithms]\label{th:uCMIDP}
If $A:\Z^n\rightarrow \W$ is an $\eps$-differentially private algorithm, then
$\uCMI{A}\leq \eps n$.
\end{theorem}
\begin{proof}

Let $\sz^*\in\Z^{n\times 2}$ and $\mathcal{P}_S$ be the parameters that maximize the universal CMI of $A$. Then,
\begin{align}\label{eq:ucmifordp}
    \uCMI{A}
    & :=\sup_{\sz\in\Z^{n \times 2}}\sup_{S} I(A(\sz_S);S) \nonumber\\
    & = \ex{s\gets\mathcal{P}_S}{\dkl{A(\sz^*_s)}{A(\sz^*_S)}} \nonumber \\
    & =\ex{s\gets\mathcal{P}_S}{\ex{w\gets\mathcal{P}_{A(\sz^*_s)}}{\log\frac{\Pr[A(\sz^*_s)=w]}{\Pr[A(\sz^*_{S})=w]}}} \nonumber\\
    & \leq \ex{s\gets\mathcal{P}_S}{\log\sup_{w\in\W} \frac{\Pr[A(\sz^*_s)=w]}{\Pr[A(\sz^*_{S})=w]}} \nonumber\\
    & \leq \log \sup_{s\in\{0,1\}^n}{\sup_{w\in\W} \frac{\Pr[A(\sz^*_s)=w]}{\Pr[A(\sz^*_{S})=w]}}
\end{align}
For any $s,s'\in\{0,1\}^n$, $\sz^*_s, \sz^*_{s'}$ differ in at most $n$ points. Since $A$ is an $\eps$-DP algorithm, it follows that $\frac{\Pr[A(\sz^*_s)=w]}{\Pr[A(\sz^*_{s'})=w]}\leq e^{\eps n}$ for any $w\in \W$, 
and $s,s'\in\{0,1\}^n$. Combined with inequality~\eqref{eq:ucmifordp}, we get
\[\uCMI{A} \leq \log(e^{\eps n}) = \eps n.\]
\end{proof}
In contrast, recall that an $\eps$-DP algorithm $A$ has $\CMI{A}{\D}\leq \eps^2n/2$ for any distribution $\D$.

If $A : \Z^n \to \W$ satisfies $\frac12\varepsilon^2$-CDP, then by group privacy \cite{BunS16} we obtain the bound $\uCMI{A}\le \frac12\varepsilon^2 n^2$. Unfortunately, this is meaningless unless $\varepsilon<\sqrt{2/n}$. Group privacy for approximate differential privacy yields extremely weak bounds.

It is possible to show that group privacy yields essentially optimal bounds for uCMI. We briefly sketch a counterexample: Choose $\sz$ and $S$ so that $\sz_S$ represents a random codeword from an error correcting code. Specifically, the size of the code should be $2^{n-o(n)}$ and the distance should be $\tilde\Omega(n)$. The algorithm $A$ simply attempts to output the input codeword in a differentially private manner. We can use a GAP-MAX algorithm \cite{BunDRS18} to do this. For each codeword we measure the Hamming distance between it and the input codeword; this distance has sensitivity $1$ and we are guaranteed that one codeword has distance $0$ and all others have distance $\tilde\Omega(n)$. This allows the algorithm to identify the codeword in a differentially privacy manner with high probability. This yields high uCMI (since the code is large) and is differentially private, which shows that we cannot get strong uCMI bounds from differential privacy.

\subsection{Stability and Evaluated CMI}\label{sec:ecmi}

\subsubsection{Uniform Stability}
There is a long line of work  \cite[etc.]{RogersW78,DevroyeW79a,DevroyeW79b,KearnsR97,BousquetE02,Shalev-ShwarzSSS09,HardtRS16,FeldmanV18,FeldmanV19,DaganF19,BousquetKZ19} on proving generalization from stability -- specifically, stability with respect to the loss. There are several definitions of loss stability; here we consider a strong version, \emph{uniform stability}:

\begin{definition}[Uniform Stability \cite{BousquetE02}]
Let $A : \mathcal{Z}^n \to \mathcal{W}$ be a deterministic\footnote{For simplicity, we assume $A$ is deterministic, although this definition can be extended to randomized algorithms.} algorithm and let $\ell : \mathcal{W} \times \mathcal{Z} \to \mathbb{R}$ be a function. We say $A$ has uniform stability $\gamma$ with respect to $\ell$ if $$\forall x_1, \cdots, x_n, y, z \in \mathcal{Z} ~ \forall i \in [n] ~~~~~ | \ell(A(x_1, \cdots, x_{i-1}, y, x_{i+1}, \cdots, x_n), z) - \ell(A(x_1, \cdots, x_n),z) | \le \gamma.$$
\end{definition}

Uniform stability can be shown to hold for certain convex optimization algorithms \cite{BousquetE02,Shalev-ShwarzSSS09} and uniformly stable learners exist for all classes of bounded VC dimension \cite{DaganF19}. And, of course, uniform stability implies generalization \cite{FeldmanV18,FeldmanV19,BousquetKZ19}. In particular, the following bound is nearly tight.

\begin{theorem}[{\cite[Thm.~1.2]{FeldmanV18},\cite[Thm.~1.1]{FeldmanV19},\cite[Cor.~3.2]{BousquetKZ19}}]\label{thm:stab-gen}
Let $\mathcal{D}$ be a distribution on $\mathcal{Z}$. Let $A : \mathcal{Z}^n \to \mathcal{W}$ have uniform stability $\gamma$ with respect to $\ell : \mathcal{W} \times \mathcal{Z} \to [0,1]$. Then $$\ex{Z \leftarrow \mathcal{D}^n}{\left(\ell(A(Z),Z)-\ell(A(Z),\mathcal{D})\right)^2} \le 16 \gamma^2 + \frac{2}{n}$$ and $$\forall \delta>0 ~~~~~ \pr{Z \leftarrow \mathcal{D}^n}{\left|\ell(A(Z),Z)-\ell(A(Z),\mathcal{D})\right| \ge O\left(\gamma \cdot \log^2(n/\delta) + \sqrt{\frac{\log(1/\delta)}{n}}\right)} \le \delta.$$
\end{theorem}

Uniform stability should be contrasted with differential privacy, which is also a stability notion. The difference is that uniform stability demands that the loss not change much, whereas differential privacy demands that the distribution of the output of $A$ not change much. Differential privacy is qualitatively a stricter requirement.

In contrast with CMI, uniform stability implies high probability bounds \cite{BousquetE02,FeldmanV18,FeldmanV19}. Another key difference is that the loss function is part of the definition of uniform stability, whereas CMI is agnostic to the loss function. This difference makes it difficult to compare uniform stability with CMI.

\subsubsection{Evaluated CMI}

In order to compare uniform stability and CMI, we introduce the following relaxation of CMI. Our relaxation includes the loss function into the definition. Rather than looking at the mutual information between $S$ and $A(\sZ_S)$ we look at the mutual information between $S$ and $\ell(A(\sZ_S),\cdot)$ evaluated at each point in $\sZ$. This is a natural restriction because it is the only information relevant to computing the quantities of interest -- i.e., $\ell(A(\sZ_S),\sZ_S)$ and $\ell(A(\sZ_S),\sZ_{\overline{S}})$.

\newcommand{\eCMI}[3]{{\ifx&#2& \mathsf{eCMI} \else \mathsf{eCMI}_{#2} \fi \left(#3(#1)\right)}}

\begin{definition}[Evaluated CMI (eCMI)]\label{def:eCMI}
Let $A:\Z^n\rightarrow \W$ be a randomized or deterministic algorithm. Let $\ell : \mathcal{W} \times \mathcal{Z} \to \mathcal{R}$.\footnote{Normally $\mathcal{R}=\mathbb{R}$. However, we allow our definition to cover a function $\ell$ with a different range.}
Let $\D$ be a distribution on $\Z$ and $\sZ\in\Z^{n \times 2}$ consist of $2n$ samples drawn independently from $\D$. Let $S\in \{0,1\}^n$ be uniformly random and independent from $\sZ$. Define $\sZ_S \in \Z^n$ by $(\sZ_S)_i = \sZ_{i,S_i+1}$ for all $i \in [n]$ -- that is, $\sZ_S$ is the subset of $\sZ$ indexed by $S$. 

The \emph{conditional mutual information (eCMI) of $A$ with respect to $\D$ evaluated by $\ell$} is $$\eCMI{A}{\D}{\ell}:=I(\vec{\ell}(A(\sZ_S),\sZ);S|\sZ),$$
where $\vec{\ell} : \mathcal{W} \times \mathcal{Z}^{n \times 2} \to \mathcal{R}^{n \times 2}$ is defined by $\vec{\ell}(w,\sz)_{i,j} = \ell(w,\sz_{i,j})$ for $i \in [n]$ and $j \in [2]$.

The \emph{distribution-free  conditional mutual information (eCMI) of $A$ evaluated by $\ell$} is  $$\eCMI{A}{}{\ell}:=\sup_{\sz \in \Z^{n \times 2}} I(\vec{\ell}(A(\sz_S),\sz);S).$$
\end{definition}

It is immediate from the data processing inequality that $\eCMI{A}{\D}{\ell} \le \CMI{A}{\D}$ and $\eCMI{A}{}{\ell} \le \CMI{A}{}$ for all $A$, $\D$, and $\ell$. Thus all the upper bounds on CMI carry over to eCMI. We also find that most of the generalization properties do too.
The following provides a handful of generalization properties of Evaluated CMI.

\begin{theorem}\label{thm:lossagnostic-eval}
Let $\mathcal{D}$ be a distribution on $\mathcal{Z}$. Let $A :\mathcal{Z}^n \to \mathcal{W}$ be a randomized algorithm. Let $\ell : \mathcal{W} \times \mathcal{Z} \to \mathbb{R}$ be an arbitrary function.
\begin{itemize}
\item[(i)] If there exists $\Delta : \mathcal{Z}^2 \to \mathbb{R}$ such that $|\ell(w,z_1)-\ell(w,z_2)| \le \Delta(z_1,z_2)$ for all $z_1,z_2\in\mathcal{Z}$ and $w \in \mathcal{W}$, then 
$$\ex{Z,A}{\left|\ell(A(Z),Z)-\ell(A(Z),\mathcal{D})\right|} \le  \sqrt{\frac{2}{n} \cdot \left(\eCMI{A}{\D}{\ell} +\log 2\right)\cdot \ex{(Z_1,Z_2) \leftarrow \mathcal{D}^2}{\Delta(Z_1,Z_2)^2}}.$$

\item[(ii)] If $\ell : \mathcal{W} \times \mathcal{Z} \to [0,1]$, then $$\ex{Z,A}{\left(\ell(A(Z),Z)-\ell(A(Z),\mathcal{D})\right)^2} \le \frac{3\cdot \eCMI{A}{\D}{\ell}+\log 3}{n}.$$

\item[(iii)] If $\ell : \mathcal{W} \times \mathcal{Z} \to [0,1]$ and $\ex{Z,A}{\ell(A(Z),Z)}=0$, then $$\ex{}{\ell(A(Z),\mathcal{D})} \le  \frac{1.5}{n} \cdot \eCMI{A}{\D}{\ell}.$$
\end{itemize}
\end{theorem}

The proof of Theorem \ref{thm:lossagnostic-eval} is the same as Theorem \ref{thm:cmigen-absloss} (for part (i)), Theorem \ref{thm:cmigen-linsquared} (for part (ii)) and Theorem \ref{thm:cmigen-realize1} (for part (iii)). Note that we can only prove generalization bounds for the loss function that is included in the definition of Evaluated CMI, whereas, for ordinary CMI, we obtain generalization bounds for any loss function.

\subsubsection{Bounding Evaluated CMI from Uniform Stability}

\begin{theorem}\label{thm:stable-to-ecmi}
Let $\sigma,\gamma>0$. 
Let $A : \mathcal{Z}^n \to \mathcal{W}$ have uniform stability $\gamma$ with respect to $\ell : \mathcal{W} \times \mathcal{Z} \to \mathbb{R}$. 
Then there exists a randomized algorithm $\tilde{A} : \mathcal{Z}^n \to \tilde{\mathcal{W}}$ and a function $\tilde{\ell} : \tilde{\mathcal{W}} \times \mathcal{Z} \to \mathbb{R}$ such that the following holds. Let $\mathcal{D}$ be a continuous\footnote{We make the assumption of continuity so that the probability that two samples from $\mathcal{D}$ are equal is zero. This assumption can be relaxed, but then the collision probability must be accounted for in the analysis. We can always augment $\mathcal{D}$ with superfluous randomness to ensure this (e.g. by replacing $\mathcal{D}$ with $\mathcal{D} \times \mathcal{U}_{[0,1]}$). } distribution on $\mathcal{Z}$. 
We have $\eCMI{\tilde{A}}{\D}{\tilde\ell} \le \frac{\gamma^2 n^2}{2\sigma^2}$ and $$\pr{\sZ}{\forall s,s' \in \{0,1\}^n ~~~ \ex{\tilde{A}}{\tilde\ell(\tilde{A}(\sZ_s),\sZ_{s'})}=\ell(A(\sZ_s),\sZ_{s'}) ~\wedge~ \ex{\tilde{A}}{\left(\tilde\ell(\tilde{A}(\sZ_s),\sZ_{s'})-\ell(A(\sZ_s),\sZ_{s'})\right)^2}=\frac{\sigma^2}{n}}=1$$
\end{theorem}
\begin{proof}[Proof Sketch]
Let $\tilde{\mathcal{W}}=\mathcal{W} \times \mathbb{R}^{\mathcal{Z}}$ and $\tilde\ell((w,\xi),z)=\ell(w,z)+\xi(z)$. Define $\tilde{A}(z)=(A(z),\xi)$ where, for all $z\in\mathcal{Z}$, $\xi(z) \leftarrow \mathcal{N}(0,\sigma^2)$ and these Gaussians are independent.\footnote{It suffices for the noise to be $2n$-wise independent.}  Now
\begin{align*}
    \eCMI{\tilde{A}}{\D}{\tilde\ell} &= I(\vec{\tilde\ell}(\tilde{A}(\sZ_S),\sZ);S|\sZ)\\
    &= \ex{\sZ}{I(\vec{\tilde\ell}(\tilde{A}(\sZ_S),\sZ);S)}\\
    &= \ex{\sZ}{\inf_Q \ex{S}{\dkl{\vec{\tilde\ell}(\tilde{A}(\sZ_S),\sZ)}{Q}}}\\
    &= \ex{\sZ}{\inf_Q \ex{S}{\dkl{\mathcal{N}(\vec{\ell}(A(\sZ_S),\sZ),\sigma^2 I)}{Q}}}\\
    &\le \ex{\sZ}{\ex{S}{\dkl{\mathcal{N}(\vec{\ell}(A(\sZ_S),\sZ),\sigma^2 I)}{\mathcal{N}(\mu_{\sZ}, \sigma^2 I)}}}\\
    &= \ex{\sZ,S}{\frac{1}{2\sigma^2} \left\|\vec{\ell}(A(\sZ_S),\sZ) - \mu_{\sZ}\right\|_2^2}\\
    &= \frac{1}{2\sigma^2} \sum_{i \in [n], j\in [2]} \ex{\sZ}{\var{S}{\ell(A(\sZ_S),\sZ_{i,j})}}\\
    &= \frac{1}{2\sigma^2} \sum_{i \in [n], j\in [2]} \ex{\sZ}{\ex{S}{f_{\sZ}(S)_{i,j}^2}},
\end{align*}
where $\mu_{\sz}=\ex{S}{\vec{\ell}(A(\sz_S),\sz)}$ and $f_{\sz}(s)=\vec\ell(A(\sz_s),\sz)-\mu_{\sz}$. For all $\sz\in\mathcal{Z}^{n \times 2}$, we have $\ex{S}{f_{\sz}(S)}=0$ and $$\|f_{\sz}(s)-f_{\sz}(s')\|_\infty \le \gamma \cdot \|s-s'\|_1 $$ for all $s,s' \in \{0,1\}^n$. Thus, by Lemma \ref{lem:steele}, we have $$\ex{S}{f_{\sZ}(S)_{i,j}^2}\le \frac12 \gamma^2 n$$ for all $i \in [n]$ and $j \in [2]$. This yields the first part of the theorem -- $\eCMI{\tilde{A}}{\D}{\tilde\ell} \le \frac{\gamma^2 n^2}{2\sigma^2}$.

Since we assumed $\mathcal{D}$ to be continuous, we have $\pr{\sZ}{\forall i,i' \in [n] ~ \forall j,j' \in [2] ~~ \sZ_{i,j} \ne \sZ_{i',j'}}=1$. Conditioned on this event, $\vec{\tilde\ell}(\tilde{A}(\sZ_s),\sZ)$ consists of $2n$ independent Gaussians, each with variance $\sigma^2$ and mean $\vec\ell(A(\sZ_S),\sZ)$. This implies the second part of the result.
\end{proof}

Some remarks are in order: By combining Theorem \ref{thm:stable-to-ecmi} with Theorem \ref{thm:lossagnostic-eval}(ii), we can show that, if $A$ has uniform stability $\gamma$ with respect to $\ell$ and $\D$ is a continuous distribution, then $$\ex{Z \leftarrow \mathcal{D}^n}{\left(\ell(A(Z),Z)-\ell(A(Z),\mathcal{D})\right)^2} \le \inf_{\sigma>0} O\left(\frac{\gamma^2 n}{\sigma^2} + \frac{\sigma^2+1}{n} \right) = O\left(\gamma + \frac{1}{n} \right).$$ This should be contrasted with Theorem \ref{thm:stab-gen}, which obtains the bound $O(\gamma^2+1/n)$. Thus we conclude that eCMI (specifically, Theorem \ref{thm:stable-to-ecmi}) does not provide a tight analysis of the generalization properties of uniform stability. Further work is needed to integrate uniform stability into our CMI framework. We speculate that it may be easier (and also more enlightening) to analyze specific algorithms with CMI or eCMI than to find a generic reduction like Theorem \ref{thm:stable-to-ecmi}.

\section*{Acknowledgements}
We thank the anonymous reviewers for suggestions that helped improve this manuscript and led to a simplification of the proof of Theorem~\ref{th:CMIofTV}. Part of this work was done when LZ interned at IBM Research -- Almaden. Otherwise, LZ was supported by NSF grants CCF-1718088 and CCF-1750640.
\addcontentsline{toc}{section}{References}
\printbibliography

\end{document}